\documentclass{article}
\usepackage{microtype}
\usepackage{graphicx}
\usepackage{subcaption}
\usepackage{booktabs} 
\usepackage{paralist}
\usepackage{tikz}
\usepackage{mathtools}
\usetikzlibrary{positioning, quotes}

\PassOptionsToPackage{numbers, compress}{natbib}

\usepackage[ruled, vlined, linesnumbered]{algorithm2e}
\usepackage{hyperref}

\usepackage{apptools}
\usepackage[page, header]{appendix}
\usepackage{titletoc}
\usepackage[final]{neurips_2018}
\usepackage{comment}

\title{On Oracle-Efficient PAC RL with Rich Observations}

\usepackage{mkolar_definitions}
\usepackage{color}
\newcommand{\ntrain}{n_{\textrm{train}}}
\newcommand{\neval}{n_{\textrm{eval}}}
\newcommand{\ntest}{n_{\textrm{test}}}
\newcommand{\nexp}{n_{\textrm{exp}}}

\newcommand{\epsstat}{\epsilon_{\textrm{stat}}}
\newcommand{\epsfeas}{\epsilon_{\textrm{feas}}}
\newcommand{\epssub}{\epsilon_{\textrm{sub}}}

\newcommand{\rt}{\varnothing}

\newcommand{\mainAlg}{\textsc{Valor}\xspace} 
\newcommand{\lsvee}{\textsc{Lsvee}\xspace}

\newcommand{\Tmax}{T_{\max}}

\newcommand{\ii}{^{(i)}}

\newcommand{\order}{\ensuremath{O}}
\newcommand{\otil}{\ensuremath{\tilde{\order}}}

\newcommand{\bd}{b}

\newcommand{\Ndfs}{N_{\text{dfs}}}

\newcommand{\Ex}{\mathbf E}
\newcommand{\EEx}{\hat{\mathbf E}}
\newcommand{\prob}{\mathbf P}
\newcommand{\para}[1]{\textbf{#1}~}
\newcommand{\olive}{\textsc{Olive}\xspace}
\newcommand{\algname}{\textsc{VaLoR}\xspace}

\newcount\Comments  
\definecolor{darkgreen}{rgb}{0,0.5,0}
\definecolor{darkred}{rgb}{0.7,0,0}
\definecolor{teal}{rgb}{0.3,0.8,0.8}
\newcommand{\kibitz}[2]{\ifnum\Comments=1\textcolor{#1}{#2}\fi}

\SetCommentSty{mycommfont}

\usepackage{xcolor}
\hypersetup{
    colorlinks,
    linkcolor={red!50!black},
    citecolor={blue!50!black},
    urlcolor={blue!80!black}
}
\usepackage{nicefrac}
\usepackage{xr}

\usepackage{chngcntr}

\author{
    Christoph Dann\\
    Carnegie Mellon University\\
    Pittsburgh, Pennsylvania\\
    \texttt{cdann@cdann.net}
    \And
    Nan Jiang\thanks{The work was done while NJ was a postdoc researcher at MSR NYC.}\\
    UIUC\\
    Urbana, Illinois\\
    \texttt{nanjiang@illinois.edu}
    \And
    Akshay Krishnamurthy\\
    Microsoft Research\\
    New York, New York\\
    \texttt{akshay@cs.umass.edu}
    \And
    Alekh Agarwal\\
    Microsoft Research\\
    Redmond, Washington\\
    \texttt{alekha@microsoft.com}
    \And
    John Langford\\
    Microsoft Research\\
    New York, New York\\
    \texttt{jcl@microsoft.com}
    \And
    Robert E.~Schapire\\
    Microsoft Research\\
    New York, New York\\
    \texttt{schapire@microsoft.com}
}
\begin{document}

\maketitle

\begin{abstract}
  We study the computational tractability of 
  PAC reinforcement learning with rich observations.
  We present new provably sample-efficient algorithms for environments with
  deterministic hidden state dynamics and stochastic rich
  observations. These methods operate in an oracle model of
  computation---accessing policy and value function classes
  exclusively through standard optimization primitives---and therefore
  represent computationally efficient alternatives to prior
  algorithms that require enumeration.
  With stochastic hidden state dynamics, 
  we prove that the
  only known sample-efficient
  algorithm, \olive~\citep{jiang2017contextual}, cannot be implemented in the
  oracle model.
  We also present several examples that illustrate fundamental
  challenges of tractable PAC reinforcement learning in such general
  settings.
\end{abstract}

\section{Introduction}
\label{introduction}

We study
episodic reinforcement learning (RL)
in environments with realistically rich observations such as images or text, which we refer to broadly as \emph{contextual decision processes}.
We aim for methods that use {function approximation} in a provably
effective manner to find the best possible
policy through strategic exploration. 

While such problems are central to empirical RL research
\citep{mnih2015human}, most theoretical results on strategic
exploration focus on tabular MDPs with small state spaces
\citep{kearns2002near, brafman2003r,
  strehl2005theoretical, strehl2006pac, auer2009near,
  dann2015sample,azar2017minimax, dann2017unifying}.
Comparatively little work exists on provably effective exploration
with large observation spaces that require generalization through
function approximation.  The few algorithms that do exist either have
poor sample complexity
guarantees~\citep[e.g.,][]{kakade2003exploration,
  pazis2013pac,grande2014sample,pazis2016efficient} or require fully
deterministic environments~\citep{wen2013efficient, wen2017efficient}
and are therefore inapplicable to most real-world applications and
modern empirical RL benchmarks.  This scarcity of positive results on
efficient exploration with function approximation can likely be
attributed to the challenging nature of this problem rather than a
lack of interest by the research community.

On the statistical side, recent important progress was made by showing
that contextual decision processes (CDPs) with rich stochastic observations and
deterministic dynamics over $M$ hidden states can be learned with a
sample complexity polynomial in
$M$~\citep{krishnamurthy2016contextual}.  This was followed by an
algorithm called \olive~\citep{jiang2017contextual} that enjoys a polynomial sample
complexity guarantee for a broader range of CDPs, including ones with stochastic hidden state transitions.  While
encouraging, these efforts focused exclusively on statistical issues,
ignoring computation altogether.  Specifically, the proposed
algorithms exhaustively enumerate candidate value functions to
eliminate the ones that violate Bellman equations, 
an approach that is computationally intractable for any function class
of practical interest. 
Thus, while
showing that RL with rich observations can be statistically tractable, these results
leave open the question of computational feasibility.

In this paper, we focus on this difficult computational
challenge.  We work in an oracle model of computation, meaning that we
aim to design sample-efficient algorithms whose computation can be
reduced to common optimization primitives over function spaces, such
as linear programming and cost-sensitive classification. The
oracle-based approach has produced practically effective algorithms
for active learning~\citep{hsu2010algorithms}, contextual
bandits~\citep{agarwal2014taming}, structured
prediction~\citep{ross2014reinforcement,chang2015learning}, and
multi-class classification~\citep{allwein2000reducing}, and here, we consider oracle-based algorithms for
challenging RL settings.

We begin by studying the setting of~\citet{krishnamurthy2016contextual} with deterministic dynamics
over $M$ hidden states and stochastic rich observations.  In
Section~\ref{sec:alg}, we use cost-sensitive classification and linear
programming oracles to develop \mainAlg, the first algorithm that is
both \emph{computationally} and \emph{statistically} efficient for this
setting. While deterministic hidden-state dynamics are somewhat
restrictive, the model is considerably more general than fully
deterministic MDPs assumed by prior work~\citep{wen2013efficient, wen2017efficient}, and it accurately captures 
modern empirical benchmarks such as visual grid-worlds in
Minecraft \citep{johnson2016malmo}. As such, this method represents a
considerable advance toward provably efficient RL in practically
relevant scenarios.

Nevertheless, we ultimately seek efficient algorithms for more general
settings, such as those with stochastic hidden-state
transitions. Working toward this goal, we study the computational
aspects of the \olive\ algorithm \citep{jiang2017contextual}, which
applies to a wide range of environments.  Unfortunately, in
Section~\ref{sec:nphard}, we show that \olive \emph{cannot} be
implemented efficiently in the oracle model of computation.  As \olive
is the only known statistically efficient approach for this general
setting, our result establishes a significant barrier to computational
efficiency.  In the appendix, we also describe several other barriers,
and two other oracle-based algorithms for the deterministic-dynamics
setting that are considerably different from \mainAlg. The negative
results identify where the hardness lies while the positive results
provide a suite of new algorithmic tools. Together, these results
advance our understanding of efficient reinforcement learning with
rich observations.

\section{Related Work} \label{sec:related_work}
There is abundant work on strategic exploration in the
tabular setting \citep{kearns2002near, brafman2003r,
  strehl2005theoretical, strehl2006pac, auer2009near,
  dann2015sample,azar2017minimax, dann2017unifying}. The computation in these algorithms often involves planning in optimistic models and can be solved efficiently via dynamic programming.
To extend
the theory to the more practical settings of large state spaces,
typical approaches include (1)
distance-based state identity test under smoothness assumptions
\citep[e.g.,][]{kakade2003exploration, pazis2013pac,grande2014sample,pazis2016efficient}, or (2) working
with factored MDPs \citep[e.g.,][]{kearns1999efficient}. The former
approach is similar to the use of state abstractions
\cite{li2006towards}, 
and typically incurs exponential sample
complexity in state dimension. The latter approach does have
sample-efficient results, but the factored representation assumes
relatively disentangled state variables which cannot model rich
sensory inputs (such as images).

\citet{azizzadenesheli2016romdp} have studied regret minimization in
rich observation MDPs, a special case of contextual decision processes with a small number of hidden states and reactive policies. 
They do not utilize function approximation, 
and hence incur polynomial dependence on
the number of
unique observations in both sample and computational complexity. 
Therefore, this approach, along with related works \cite{azizzadenesheli2016reinforcement, guo2016pac},
does not scale to the rich observation settings that we focus on here.

\citet{wen2013efficient, wen2017efficient} have studied
exploration with function approximation 
in fully deterministic MDPs, which is considerably more restrictive than our
setting of deterministic hidden state dynamics with stochastic observations and rewards. 
Moreover, their analysis measures representation complexity using \emph{eluder dimension} \cite{russo2013eluder, osband2014model},
which is only known to be small for some simple function classes. In comparison, our bounds scale with more standard complexity measures and can easily extend to VC-type quantities,
which allows our theory to apply
to practical and popular function approximators including neural
networks~\citep{anthony2009neural}.

\section{Setting and Background}
\label{sec:prelim}

\begin{figure}
	\centering\includegraphics[width=.9\textwidth]{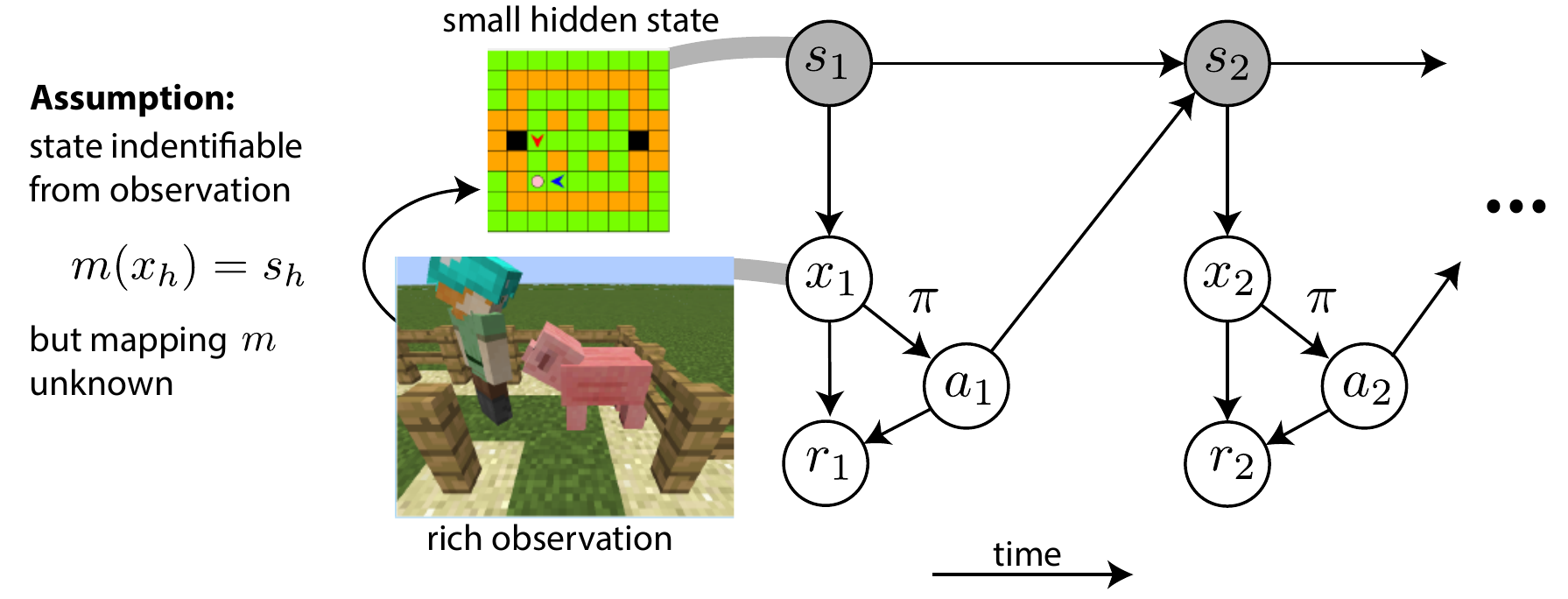}
    \caption{Graphical representation of the problem class considered by our algorithm, \mainAlg: The main assumptions that enable sample-efficient learning are (1) that the small hidden state $s_h$ is identifiable from the rich observation $x_h$ and (2) that the next state is a deterministic function of the previous state and action. State and observation examples are from \url{https://github.com/Microsoft/malmo-challenge}. }
	\label{fig:environexample}
\end{figure}

We consider reinforcement learning (RL) in a common special case of contextual decision processes~\citep{krishnamurthy2016contextual,jiang2017contextual}, sometimes referred to as rich observation MDPs~\citep{azizzadenesheli2016romdp}. 
We assume an $H$-step process where in each episode, a
random \emph{trajectory}
$s_1, x_1, a_1, r_1, s_2, x_2, \ldots, s_H, x_H, a_H, r_H$ is
generated.  For each time step (or \emph{level}) $h\in[H]$,
$s_h \in \Scal$ where $\Scal$ is a finite hidden state space,
$x_h\in\Xcal$ where $\Xcal$ is the rich observation (context) space,
$a_h \in\Acal$ where $\Acal$ is a finite action space of size $K$, and
$r_h \in \RR$.  Each hidden state $s\in\Scal$ is associated with an
emission process $O_s \in \Delta(\Xcal)$, and we use $x \sim s$ as a
shorthand for $x \sim O_s$.
We assume that each rich observation contains enough information so that $s$ can in principle be identified just from $x\sim O_s$---hence $x$ is a Markov state and the process is in fact an MDP over $\Xcal$---but the mapping $x \mapsto s$ is unavailable to the agent and $s$ is never observed. 
The hidden states $\Scal$ introduce structure into the problem, which is essential since we allow the observation space $\Xcal$ to be infinitely large.\footnote{Indeed, the lower bound in Proposition~6 in~\citet{jiang2017contextual} show that ignoring underlying structure precludes provably-efficient RL, even with function approximation.}
The issue of partial observability is not the focus of the paper. 

Let $\Gamma: \Scal\times\Acal\to\Delta(\Scal)$ define
transition dynamics over the hidden states, and let
$\Gamma_1 \in \Delta(\Scal)$ denote an initial distribution over
hidden states.
$R: \Xcal \times \Acal \to \Delta(\RR)$
is the reward function; this differs from partially observable MDPs where reward depends only on $s$, making the problem more challenging. With this notation, a trajectory is generated
as follows: $s_1 \sim \Gamma_1$, $x_1 \sim s_1$,
$r_1 \sim R(x_1, a_1)$, $s_2 \sim \Gamma(s_1, a_1)$, $x_2 \sim s_2$,
\ldots, $s_H \sim \Gamma(s_{H-1}, a_{H-1})$, $x_H \sim s_H$,
$r_H \sim R(x_H, a_H)$, with actions $a_{1:H}$ chosen by the
agent.  We emphasize that $s_{1:H}$ are unobservable to the agent.

To simplify notation, we assume that each observation and hidden state
can only appear at a particular level. This implies that $\Scal$ is
partitioned into $\{\Scal_h\}_{h=1}^H$ with size $M := \max_{h\in[H]}
|\Scal_h|$.  For regularity, assume $r_h \ge 0$ and $\sum_{h=1}^H
r_h\le 1$ almost surely.

In this setting, the learning goal is to
find a policy $\pi: \Xcal \to \Acal$ that maximizes the expected
return $V^\pi:=\Ex[\sum_{h=1}^H r_h \,|\, a_{1:H} \sim \pi]$.
Let $\pi^\star$ denote the optimal policy, which maximizes $V^\pi$, with optimal value function $g^\star$
defined as $g^\star(x) := \Ex[\sum_{h'=h}^H r_{h'} | x_h=x, a_{h:H} \sim \pi^\star]$.
As is standard, $g^\star$ satisfies the Bellman equation:
$\forall x$ at level $h$,
\[
g^\star(x) = \max_{a\in\Acal} \Ex[r_h + g^\star(x_{h+1}) | x_h = x, a_h = a],
\]
with the understanding that $g^\star(x_{H+1})\equiv 0$. A similar equation holds for the optimal Q-value function $Q^\star(x,a) := \Ex[\sum_{h'=h}^H r_{h'} | x_h = x, a_h = a, a_{h+1:H}\sim\pi^\star]$, and  $\pi^\star = \argmax_{a\in\Acal}Q^\star(x,a)$.\footnote{Note that the optimal policy and value functions depend on $x$ and not just $s$ even if $s$ was known, since reward is a function of $x$.}

Below are two special cases of the setting described above that will be important for later discussions.\\
\textbf{Tabular MDPs}: An MDP with a finite and small state space is
a special case of
this model, where $\Xcal = \Scal$ and $O_s$ is the identity map for each $s$. This setting is relevant in our discussion of oracle-efficiency of the existing \olive algorithm in Section~\ref{sec:nphard}.\\
\textbf{Deterministic dynamics over hidden states}: Our algorithm, \mainAlg, works in this special case, which
requires $\Gamma_1$ and $\Gamma(s,a)$ to be point masses. Originally proposed by \citet{krishnamurthy2016contextual}, this setting can model some challenging benchmark environments in modern reinforcement learning, including visual grid-worlds common to the deep RL literature \citep[e.g.,][]{johnson2016malmo}. In such tasks, the state records the position of each game element in a grid but the agent observes a rendered 3D view.
Figure~\ref{fig:environexample} shows a visual summary of this setting. We describe \mainAlg in detail in Section~\ref{sec:alg}.

Throughout the paper, we use $\EEx_D[\cdot]$ to denote empirical expectation
over samples from a data set $D$.

\subsection{Function Classes and Optimization Oracles}
\label{sec:oracle}

As $\Xcal$ can be rich, the agent must use function approximation to
generalize across observations.
To that end, we assume a given value function class
$\Gcal\subset (\Xcal \to [0,1])$ and policy class
$\Pi \subset (\Xcal\to \Acal)$.
Our algorithm is agnostic to the specific function classes used,
but for the guarantees to hold, they must be expressive enough to represent the optimal value function and policy,
 that is, $\pi^\star \in \Pi$
and $g^\star \in \Gcal$.
Prior works often use $\Fcal \subset (\Xcal\times\Acal\to[0,1])$ to
approximate $Q^\star$ instead, but for example \citet{jiang2017contextual} point out that
their \olive algorithm can equivalently work with
$\Gcal$ and
$\Pi$.
This $(\Gcal, \Pi)$ representation is useful in resolving the computational difficulty in
the deterministic setting, and has also been used in practice~\citep{dai2018sbeed}.

When working with large and abstract function classes
as we do here, it is natural to consider an oracle model of
computation and assume that these classes support various optimization
primitives. We adopt this \emph{oracle-based} approach here, and
specifically use the following oracles:

\para{Cost-Sensitive Classification (CSC) on Policies.}
A cost-sensitive
classification (CSC) oracle receives as inputs a parameter $\epssub$ and a sequence $\{(x\ii, c\ii)\}_{i\in[n]}$ of observations
$x\ii \in \Xcal$ and cost vectors $c\ii \in \RR^K$, where $c\ii(a)$ is the cost of predicting action $a \in \Acal$ for $x\ii$. The oracle returns a policy whose average cost is within $\epssub$ of the minimum average
cost, $\min_{\pi \in \Pi} \tfrac{1}{n} \sum_{i=1}^n c\ii(\pi(x\ii))$.
While CSC is NP-hard in the worst case, CSC can be further reduced to
binary
classification~\citep{beygelzimer2009error,langford2005sensitive} for
which many practical algorithms exist and actually form the core of empirical machine learning.
As further motivation, the CSC oracle has been used in practically
effective algorithms for contextual
bandits~\citep{langford2008epoch,agarwal2014taming}, imitation
learning~\citep{ross2014reinforcement}, and structured
prediction~\citep{chang2015learning}.

\para{Linear Programs (LP) on Value Functions.}
A linear program (LP) oracle considers an optimization problem where the objective $o: \Gcal \rightarrow \RR$ and the constraints $h_1, \dots h_m$ are
linear functionals of $\Gcal$ generated by finitely many function evaluations.
That is, $o$ and each $h_j$ have the form $\sum_{i=1}^{n}\alpha_i g(x_i)$ with coefficients $\{\alpha_{i}\}_{i\in[n]}$ and contexts $\{x_i\}_{i\in[n]}$.
Formally, for a program of the form
\begin{align*}
   \textstyle \max_{g \in \Gcal} o(g) \textrm{, ~ subject to } h_j(g) \leq c_j, ~~  \forall j \in [m],
\end{align*}
with constants $\{c_j\}_{j\in[m]}$, an LP oracle with approximation
parameters $\epssub,\epsfeas$ returns a function $\hat{g}$ that is at
most $\epssub$-suboptimal and that violates each constraint by at most
$\epsfeas$.  For intuition, if the value functions $\Gcal$ are linear with
parameter vector $\theta \in \RR^d$, i.e.,
$g(x) = \langle \theta, x\rangle$, then this reduces to a linear
program in $\RR^d$ for which a plethora of provably efficient solvers
exist. Beyond the linear case, such problems can be practically
solved using standard continuous optimization methods.
LP oracles are also employed in prior work focusing on deterministic MDPs~\citep{wen2013efficient, wen2017efficient}.

\para{Least-Squares (LS) Regression on Value Functions.}
We also consider a least-squares regression (LS) oracle that returns
the value function which minimizes a square-loss objective. Since
\mainAlg does not use this oracle, we defer details to the appendix.

We define the following notion of oracle-efficiency based on the
optimization primitives above.
\begin{definition}[Oracle-Efficient]
  An algorithm is \emph{oracle-efficient} if it can be implemented
  with polynomially many basic operations and calls to CSC, LP, and LS
  oracles.
\end{definition}
Note that our algorithmic results continue to hold if we include
additional oracles in the definition, while our hardness results
easily extend, provided that the new oracles can be efficiently
implemented in the tabular setting (i.e., they satisfy
Proposition~\ref{prop:oracles_tabular_easy}; see Section~\ref{sec:towardgeneral}).

\SetKwFunction{dfslearn}{dfslearn}
\SetKwFunction{polvalfun}{polvalfun}
\SetKwFunction{metaalg}{MetaAlg}
\SetKw{glob}{Global}
\SetKwInOut{Inputa}{Input}
\SetKwInOut{Outputa}{Output}

\section{\mainAlg: An Oracle-Efficient Algorithm}

\label{sec:alg} In this section we propose and analyze a new
algorithm, \mainAlg (Values stored Locally for RL) shown in
Algorithm~\ref{alg:metaalg} (with
\ref{alg:local_value_polval_unconstr} \&
\ref{alg:local_value_dfslearn} as subroutines).  As we will show, this
algorithm is oracle-efficient and enjoys a polynomial
sample-complexity guarantee in the deterministic
hidden-state dynamics setting described earlier, which was originally introduced
by~\citet{krishnamurthy2016contextual}.

    \scalebox{0.95}{
\begin{minipage}[t]{.53\textwidth}
\begin{algorithm}[H]

\glob: $\Dcal_1, \dots \Dcal_{H}$ initialized as $\emptyset$\;
\SetKwProg{myfun}{Function}{}{}
    \SetInd{2mm}{2mm}

    \myfun{\metaalg}{

        \dfslearn($\rt$) \label{lin:expfirst}\tcp*{Alg.\ref{alg:local_value_dfslearn}}
\For{$k=1,\ldots, MH$}{
    $\hat \pi^{(k)}, \hat V^{(k)} \gets \polvalfun()$ \label{lin:polval} \tcp*{Alg.\ref{alg:local_value_polval_unconstr}}
    $T \gets $ sample $n_{eval}$ trajectories with $\hat \pi^{(k)}$\;\label{lin:deployeval}
    $\hat V^{\hat \pi^{(k)}} \gets$ average return of $T$\;
    \lIf{$\hat V^{(k)} \leq \hat V^{\hat \pi^{(k)}} + \frac{\epsilon}{2}$}
    { \label{lin:meta_return}
        \Return $\hat \pi^{(k)}$
    }
\For{$h=1\ldots H-1$}{
    \For{all $a_{1:h}$ of $n_{expl}$ traj. $\in T$}
{\dfslearn($a_{1:h}$) \tcp*{Alg.\ref{alg:local_value_dfslearn}}
\label{lin:newlearn}}
}
}
\Return failure\;
}
\caption{Main Algorithm \mainAlg}
    \label{alg:metaalg}
\end{algorithm}

\end{minipage}
}
    \scalebox{0.95}{
\begin{minipage}[t]{.52\textwidth}
\begin{algorithm}[H]
\SetInd{2mm}{2mm}
\SetKwProg{myfun}{Function}{}{}

    \myfun{\polvalfun{}}{
        $\hat V^{\star} \gets V$ of the only dataset in $\Dcal_1$\;
        \For{$h=1:H$}{
            \tcp{CSC-oracle } 
    $\hat \pi_h \gets \argmax\limits_{\pi \in \Pi_h}\mkern-6mu \sum\limits_{(D,V, \{V_{a}\}) \in \Dcal_h} \mkern-18mu V_D(\pi; \{V_a\})$\;
    \label{lin:local_value_fit_policy}
        }
        \Return $\hat \pi_{1:H}, \hat V^{\star}$\;
    }
    \caption{Subroutine: Policy optimization with local values}
    \label{alg:local_value_polval_unconstr}
\end{algorithm}
        \vspace{0.4cm}
        \framebox{
        \begin{minipage}{0.94\textwidth}
        \textbf{Notation:} \\
            $V_D(\pi; \{V_a\}) := \EEx_{D} [K\one\{\pi(x) = a\}(r + V_a)]$
        \end{minipage}
    }
\end{minipage}
}

        \begin{algorithm}[H]

    \SetInd{2mm}{2mm}
            $\epsfeas = \epssub = \epsstat =\tilde O(\epsilon^2 / MH^3)$ \tcp*{see exact values in Table~\ref{tab:param} in the appendix}
            \vspace{.5mm}
            $\phi_{h} = (H+1-h)(6 \epsstat + 2 \epssub + \epsfeas)$ \tcp*{accuracy of learned values at level $h$}  \vspace{1mm}
\SetKwProg{myfun}{Function}{}{}
    \myfun{\dfslearn{path $p$ with length $h-1$}}{
        \For{$a \in \Acal$}{
            $D' \gets$ Sample $\ntest$ trajectories with actions $p\circ a$ \label{lin:sample_consensus}\;

            \tcp{compute optimistic / pessimistic values using LP-oracle}
            $V_{opt} \gets  \max_{g \in \Gcal_{h+1}} \EEx_{D'}[g(x_{h+1})]$ \quad (and $V_{pes}\gets \min_{g \in \Gcal_{h+1}} \EEx_{D'}[g(x_{h+1})]$)
            \label{lin:local_value_vopt}
            \hspace{2em} s.t.~ $ \forall (D, V,\_) \in \Dcal_{h+1}: 
            \vspace*{1mm}
            ~~|V - \EEx_{D}[g(x_{h+1})]| \leq \phi_{h+1} $ \label{lin:constraint2}\;
            \eIf{$|V_{opt} - V_{pes}| \leq 2\phi_{h+1} + 4 \epsstat + 2\epsfeas$ \label{lin:consensus}}{
                $V_a  \gets (V_{opt} + V_{pes})/2$
            \tcp*{consensus among remaining functions}
            \label{lin:return1}
            }{
                $V_a \gets \dfslearn(p \circ a)$ \label{lin:dfslearnchild}
            \tcp*{no consensus, descend}
            }
        }

            $\tilde D \gets $ Sample $\ntrain$ traj. with $p$ and $a_h \sim \textrm{Unif}(K)$\; \label{lin:sample_learn}
            $\tilde V \gets \max_{\pi \in \Pi_h} V_{\tilde D}(\pi; \{V_a\})$\label{lin:local_value_V}\tcp*{CSC-oracle} 

        Add $(\tilde D, \tilde V, \{V_a\}_{a \in \Acal})$ to $\Dcal_h$\;
        \Return $\tilde V$\;
        \label{lin:return2}
    }
    \caption{Subroutine: DFS Learning of local values}
    \label{alg:local_value_dfslearn}
    \end{algorithm}

Since hidden states can be deterministically reached by sequences of
actions (or \emph{paths}), from an algorithmic perspective, the process can be thought of as an
exponentially large tree where each node is associated with a hidden
state (such association is unknown to the agent). Similar to \lsvee \citep{krishnamurthy2016contextual}, \mainAlg first explores this tree (Line~\ref{lin:expfirst}) with
a form of depth first search (Algorithm~\ref{alg:local_value_dfslearn}).
To avoid visiting all of the exponentially many paths, \mainAlg performs a state identity test
(Algorithm~\ref{alg:local_value_dfslearn},
Lines~\ref{lin:sample_consensus}--\ref{lin:return1}): the data collected so far is used to (virtually) eliminate functions in $\Gcal$ (Algorithm~\ref{alg:local_value_dfslearn}, Line~\ref{lin:constraint2}), and we do not descend to a child if the remaining functions agree on the value of the child node (Algorithm~\ref{alg:local_value_dfslearn}, Line~\ref{lin:consensus}).

The state identity test prevents exploring the same hidden state twice but might also incorrectly prune unvisited states if all functions happen to agree on the value. Unfortunately, with no data from such pruned states, we are unable to
learn the optimal policy on them.
To address this issue, after \dfslearn returns, we first use the stored data and values (Line~\ref{lin:polval}) to compute a policy (see Algorithm~\ref{alg:local_value_polval_unconstr}) that is near optimal on all explored states.
Then, \mainAlg deploys the computed policy (Line~\ref{lin:deployeval}) and
only terminates if the estimated optimal value is achieved (Line~\ref{lin:meta_return}). If not, the policy has good probability of visiting those accidentally pruned states (see Appendix~\ref{app:meta_analysis}), so we invoke \dfslearn on the generated paths to complement the data sets (Line~\ref{lin:newlearn}).

In the rest of this section we describe \mainAlg in more detail, and then state its statistical and  computational guarantees.
\mainAlg follows a dynamic programming style and learns in a
bottom-up fashion. As a result, even given stationary function classes
$(\Gcal, \Pi)$ as inputs, the algorithm can return a non-stationary
policy $\hat \pi_{1:H} := (\hat \pi_1, \ldots, \hat \pi_H) \in \Pi^H$ that may use different policies at different
time steps.\footnote{This is not rare in RL; see e.g., Chapter 3.4 of \citet{ross2013interactive}.}
To avoid ambiguity,
we define $\Pi_h:=\Pi$ and $\Gcal_h := \Gcal$ for $h\in [H]$, to
emphasize the time point $h$ under consideration. For convenience, we also define $\Gcal_{H+1}$ to be the singleton $\{x\mapsto 0\}$. This notation also
allows our algorithms to handle more general non-stationary
function classes.

\para{Details of depth-first search exploration.} \mainAlg maintains many data sets collected at paths visited by \dfslearn.
Each data set $D$ is collected from some path $p$, which leads to some
hidden state $s$. (Due to determinism, we will refer to $p$ and $s$
interchangeably throughout this section.) $D$ consists of tuples
$(x,a,r)$ where $x\sim p$ (i.e., $x \sim O_s$),
$a \sim \textrm{Unif}(K)$, and $r$ is the instantaneous
reward. Associated with $D$, we also store a scalar $V$ which
approximates $V^\star(s)$, and $\{V_a\}_{a \in \Acal}$ which
approximate $\{V^\star(s\circ a)\}_{a \in \Acal}$, where $s \circ a$ denotes the state reached when taking $a$ in $s$.
The estimates $\{V_a\}_{a \in \Acal}$ of the future optimal values associated with the current path $p \in \Acal^{h-1}$ are either determined through a recursive
call (Line~\ref{lin:dfslearnchild}), or through a \emph{state-identity test} (Lines~\ref{lin:sample_consensus}--\ref{lin:return1} in \dfslearn). To check if we already know $V^\star(p\circ a)$, we
solve constrained optimization problems to compute optimistic and
pessimistic estimates, using a small amount of data from $p \circ
a$. The constraints
eliminate all $g \in \Gcal_{h+1}$ that make incorrect
predictions for $V^\star(s')$ for any previously visited $s'$ at level $h+1$.
As such, if we have learned the value of $s \circ a$ on a
different path, the optimistic and pessimistic values must agree (``consensus''), so
we need not descend.
Once we have the future values $V_a$, the value estimate $\tilde V$ (which approximates $V^\star(s)$) is computed (in Line~\ref{lin:local_value_V}) by maximizing the sum of immediate reward and future values, re-weighted using importance sampling to reflect the policy under consideration $\pi$:
\begin{align}\label{eq:shortened_cs}
V_D(\pi; \{V_a\}) := \EEx_{D} [K\one\{\pi(x) = a\}(r + V_a)].
\end{align}

\para{Details of policy optimization and exploration-on-demand.}
\polvalfun performs a sequence of policy optimization steps using all
the data sets collected so far to find a non-stationary policy that is
near-optimal at all explored states simultaneously. Note that this policy differs from that computed in (Alg.~\ref{alg:local_value_dfslearn}, Line~\ref{lin:local_value_V}) as it is common for all datasets at a level $h$. And finally using
this non-stationary policy, \metaalg estimates its suboptimality and
either terminates successfully, or issues several other calls to
\dfslearn to gather more data sets. This so-called
exploration-on-demand scheme is due
to~\citet{krishnamurthy2016contextual}, who describe the subroutine in
more detail.

\subsection{What is new compared to \lsvee?}
The overall structure of \mainAlg is similar to \lsvee~\citep{krishnamurthy2016contextual}. The main differences are
in the pruning mechanism, where we use a novel state-identity
  test, and the policy optimization step in
Algorithm~\ref{alg:local_value_polval_unconstr}.

\lsvee uses a $Q$-value function
class $\Fcal \subset (\Xcal \times \Acal \to [0,1])$ and a state
identity test based on Bellman errors on data sets $D$ consisting of
$(x,a,r,x')$ tuples:
\begin{align*}
\textstyle\EEx_{D} \left[\left(f(x, a) - r - \EEx_{x'\sim a}\max_{a'\in\Acal}f(x', a')\right)^2\right].
\end{align*}
This enables a conceptually simpler statistical analysis, but the
coupling between value function and the policy yield
challenging optimization problems that do not obviously admit
efficient solutions.

In contrast, \mainAlg uses dynamic programming to propagate optimal
value estimates from future to earlier time points. From an
optimization perspective, we fix the future value and only optimize
the current policy, which can be implemented by standard oracles, as
we will see. However, from a statistical perspective, the inaccuracy
of the future value estimates leads to bias that accumulates over
levels. By a careful design of the algorithm and through an intricate
and novel analysis, we show that this bias only accumulates linearly
(as opposed to exponentially; see e.g.,
Appendix~\ref{sec:examples_learning}), which leads to a polynomial
sample complexity guarantee.

\subsection{Computational and Sample Complexity of \mainAlg}
\mainAlg requires two types of nontrivial computations over the
function classes. We show that they can be reduced to CSC on $\Pi$ and LP on $\Gcal$ (recall Section~\ref{sec:oracle}), respectively, and hence \mainAlg is oracle-efficient.

First, Lines~\ref{lin:local_value_fit_policy} in \polvalfun
and~\ref{lin:local_value_V} in \dfslearn involve optimizing $V_D(\pi; \{V_a\})$ (Eq.~\eqref{eq:shortened_cs}) over $\Pi$,
which can be reduced to CSC as follows: We first form tuples
$(x\ii, a\ii, y\ii)$ from $D$ and $\{V_a\}$ on which $V_D(\pi; \{V_a\})$ depends, where we bind $x_h$ to $x\ii$, $a_h$ to $a\ii$,
and $r_h+V_{a_h}$ to $y\ii$. From the tuples, we  construct a CSC data set
$(x\ii, -[K\one\{a=a\ii\}y\ii]_{a \in \Acal})$. 
On this data set, the cost-sensitive error of any policy (interpreted as a classifier) is exactly $-V_D(\pi; \{V_a\})$, so minimizing error (which the oracle does) maximizes the original objective.

Second, the state identity test requires solving the following  problem over the function class $\Gcal$:
\begin{align}
& V_{opt} = \max_{g\in\Gcal} \EEx_{D'}[g(x_h)] \quad \textrm{(and $\min$ for $V_{pes}$)}\\
\textrm{s.t.} &~ V - \phi_h \le \EEx_D[g(x_h)] \le V+\phi_h, \forall (D, V) \in \Dcal_h.\nonumber
\end{align}
The objective and the constraints are linear functionals of $\Gcal$, all empirical expectations involve polynomially many samples, and the number of
constraints is $|\Dcal_h|$ which remains polynomial
throughout the execution of the algorithm, as we will show in the sample complexity analysis. Therefore, the LP oracle can directly handle this optimization problem.

We now formally state the main computational and statistical guarantees for \mainAlg.
\begin{theorem}[Oracle efficiency of \mainAlg]\label{thm:local_value_comp}
  Consider a contextual decision process with deterministic dynamics over $M$
    hidden states as described in Section~\ref{sec:prelim}. Assume
  $\pi^\star \in \Pi$ and $g^\star \in \Gcal$. Then
  for any $\epsilon, \delta \in (0, 1)$, with probability at least $1- \delta$,
    \mainAlg makes $O\left( \frac{MH^2}{\epsilon}\log \frac {MH}{\delta}\right)$  CSC oracle calls and at most $O\left( \frac{MKH^2}{\epsilon}\log \frac {MH}{\delta}\right)$ LP oracle calls with required accuracy $\epsfeas = \epssub =\tilde O(\epsilon^2 / MH^3)$. 
\end{theorem}

\begin{theorem}[PAC bound of \mainAlg]\label{thm:local_value}
    Under the same setting and assumptions as in Theorem~\ref{thm:local_value_comp},
  \mainAlg
returns a policy $\hat{\pi}$ such that
  $V^\star - V^{\hat{\pi}} \le \epsilon$ with probability at least
  $1-\delta$, after collecting at most
  $\otil\left(\frac{M^3H^8K}{\epsilon^5}\log(|\Gcal||\Pi|/\delta)\log^3(1/\delta)\right)$
trajectories.\footnote{
  $\otil(\cdot)$ suppresses logarithmic dependencies on
  $M$, $K$, $H$, $1/\epsilon$ and doubly-logarithmic dependencies on
  $1/\delta$, $|\Gcal|$, and $|\Pi|$.}
\end{theorem}
Note that this bound assumes finite value function and policy classes
for simplicity, but can be extended to infinite function classes with
bounded statistical complexity using standard tools, as in Section~5.3
of \citet{jiang2017contextual}. The resulting bound scales linearly
with the Natarajan and Pseudo-dimension of the function classes, which
are generalizations of VC-dimension.
We further expect that one can generalize the theorems above to an
approximate version of realizability as in Section~5.4
of~\citet{jiang2017contextual}.  

Compared to the guarantee
for \lsvee~\citep{krishnamurthy2016contextual},
Theorem~\ref{thm:local_value} is worse in the dependence on $M$, $H$,
and $\epsilon$.
Yet, in Appendix~\ref{sec:lagrange} we show that
a version of \mainAlg with alternative oracle assumptions enjoys a
better PAC bound than \lsvee.  Nevertheless, we emphasize that our
main goal is to understand the interplay between statistical and
computational efficiency to discover new algorithmic ideas that may
lead to practical methods, rather than improve sample complexity
bounds.

\section{Toward Oracle-Efficient PAC-RL with Stochastic Hidden State Dynamics}
\label{sec:towardgeneral}

\mainAlg demonstrates that provably sample- and oracle-efficient RL with rich
stochastic observations is possible and, as such, makes progress toward reliable and
practical RL in many applications. In this section, we discuss the natural
next step of allowing stochastic hidden-state transitions.

\subsection{\olive is not Oracle-Efficient}

\label{sec:nphard}

For this more general setting with stochastic hidden state dynamics,
\olive \citep{jiang2017contextual} is the only known algorithm with polynomial
sample complexity, but its computational properties remain underexplored.
We show here that \olive is in fact not oracle-efficient.
A brief description of the algorithm is provided below, and in the theorem
statement, we refer to a parameter $\phi$, which the algorithm uses as a
tolerance on deviations of empirical expectations.
\begin{theorem}\label{thm:olive_notoracle}
    Assuming $P \neq NP$, even with algorithm parameter $\phi=0$ and perfect evaluation of expectations, \olive is not oracle-efficient, that is, it cannot be implemented with polynomially many basic arithmetic operations and calls to CSC, LP, and LS oracles.
\end{theorem}
The assumptions of perfect evaluation of expectations and $\phi=0$ are merely
to unclutter the constructions in the proofs.  We show this result by proving
that even in tabular MDPs, \olive solves an NP-hard problem to determine its
next exploration policy, while all oracles we consider have polynomial runtime in the tabular setting.
While we only show this for CSC, LP, and LS oracles explicitly, we expect other
practically relevant oracles to also be efficient in the tabular setting, and therefore they
could not help to implement \olive efficiently.

This theorem shows that there are no known oracle-efficient PAC-RL
methods for this general setting and that simply applying clever optimization
tricks to implement \olive is not enough to achieve a practical algorithm.
Yet, this result does not preclude tractable PAC RL
altogether, and we discuss plausible directions in the subsequent section.
Below we highlight the main
arguments of the proof.

\para{Proof Sketch of Theorem~\ref{thm:olive_notoracle}. }
\olive is round-based
and follows the \emph{optimism in the face of uncertainty}
principle. At round $k$ it selects a value function and a policy to
execute $(\hat g_k,\hat \pi_k)$
that promise the highest return while satisfying all average Bellman
error constraints:
\begin{align} \label{eqn:oliveprobg}
    &\hat g_k, \hat \pi_k = \argmax_{g \in \Gcal, \pi \in \Pi} \EEx_{D_0}[g(x)]
    \\
    &
     \textrm{s.t. }~
    |\EEx_{D_i}[K\one\{a = \pi(x) \}(g(x) - r - g(x'))]| \leq \phi, ~~\forall ~ D_i \!\in\! \Dcal\!.     \nonumber
\end{align}
Here $D_0$ is a data set of initial contexts $x$, $\Dcal$ consists of
data sets of $(x,a,r,x')$ tuples collected in the previous rounds, and
$\phi$ is a statistical tolerance parameter.
If this optimistic policy $\hat \pi_{k}$ is close to optimal, \olive
returns it and terminates. Otherwise we add a constraint
to~\eqref{eqn:oliveprobg} by (i) choosing a time point $h$, (ii)
collecting trajectories with $\hat \pi_{k}$ but choosing the $h$-th
action uniformly, and (iii) storing the tuples $(x_h,a_h,r_h,x_{h+1})$
in the new data set $D_k$ which is added to the constraints for the
next round.

The following theorem shows that \olive's optimization is
NP-hard even in tabular MDPs.

\begin{theorem}\label{thm:olive_nphard_g}
  Let $\Pcal_{\olive}$ denote the family of problems of the
  form~\eqref{eqn:oliveprobg}, parameterized by
  $(\Xcal,\Acal,\textrm{Env}, t)$, which describes the optimization
  problem induced by running \olive in the MDP
  $\textrm{Env}$ (with states $\Xcal$, actions $\Acal$, and perfect
  evaluation of expectations) for $t$ rounds. \olive is given tabular function classes
  $\Gcal = (\Xcal \to [0,1])$ and $\Pi = (\Xcal \to \Acal)$  and uses
  $\phi=0$. Then $\Pcal_{\olive}$ is NP-hard.
\end{theorem}
At the same time, oracles are implementable in polynomial time:
\begin{proposition}\label{prop:oracles_tabular_easy}
    For tabular value
    functions  $\Gcal = (\Xcal \to [0,1])$ and policies $\Pi = (\Xcal \to
    \Acal)$, the CSC, LP, and LS oracles can be implemented in time polynomial in $|\Xcal|$, $K = |\Acal|$ and the input size.
\end{proposition}
Both proofs are in Appendix~\ref{sec:nphardapp}.
Proposition~\ref{prop:oracles_tabular_easy} implies that if \olive could be implemented with polynomially many CSC/LP/LS oracle calls, its total runtime would be polynomial for tabular MDPs. Assuming P $\neq$ NP, this contradicts Theorem~\ref{thm:olive_nphard_g} which states that
determining the exploration policy of \olive in tabular MDPs is NP-hard. Combining both statements therefore proves Theorem~\ref{thm:olive_notoracle}.

We now give brief intuition for Proposition~\ref{prop:oracles_tabular_easy}. To implement the CSC oracle, for each of the polynomially many observations $x \in \Xcal$, we simply add the cost vectors for that observation together
and pick the action that minimizes the total cost, that is, compute the action $\hat\pi (x)$ as
$\min_{a \in \Acal} \sum_{i\in [n] :~ x\ii = x} c\ii(a)$. Similarly, the square-loss
objective of the LS-oracle decomposes and we can compute the tabular solution one entry at a time. In both cases, the oracle runtime is $O(nK|\Xcal|)$.
Finally, using one-hot encoding, $\Gcal$ can be written as a linear function in $\RR^{|\Xcal|}$ for which the LP oracle problem reduces to an LP in $\RR^{|\Xcal|}$. The ellipsoid method \citep{khachiyan1980polynomial} solves these approximately
in polynomial time.

\subsection{Computational Barriers with Decoupled Learning Rules.}
One factor contributing to the computational intractability of \olive is that~\eqref{eqn:oliveprobg} involves optimizing over policies and values \emph{jointly}.
It is therefore promising to look for
algorithms that separate optimizations over policies and values, as in \mainAlg.  In Appendix~\ref{sec:examples}, we provide a
series of examples that illustrate some limitations of such algorithms.  First,
we show that methods that compute optimal values iteratively in the style of
fitted value iteration~\citep{gordon1995stable} need additional assumptions on $\Gcal$
and $\Pi$ besides realizability (Theorem~\ref{thm:backup}). (Storing value estimates of states explicitly
allows \mainAlg to only require realizability.)  Second, we show that with
stochastic state dynamics, average value constraints, as in
Line~\ref{lin:constraint2} of Algorithm~\ref{alg:local_value_dfslearn}, can cause the
algorithm to miss a high-value state (Proposition~\ref{prop:sq_loss_good}). Finally, we show that square-loss constraints
suffer from similar problems (Proposition~\ref{prop:sq_loss_bad}).

\subsection{Alternative Algorithms.}
An important element of \mainAlg is that it explicitly stores value
estimates of the hidden states, which we call ``local values." Local
values lead to statistical and computational efficiency under weak
realizability conditions, but this approach is unlikely to generalize
to the stochastic setting where the agent may not be able to
consistently visit a particular hidden state.
In Appendices~\ref{sec:lagrange}-\ref{sec:global_policy}, we therefore derive alternative
algorithms which do not store local values to approximate the future value $g^\star(x_{h+1})$. 
Inspired by classical RL algorithms,
these algorithms approximate $g^\star(x_{h+1})$ by either bootstrap
targets $\hat g_{h+1}(x_{h+1})$ (as in TD methods) or Monte-Carlo
estimates of the return using a near-optimal roll-out policy
$\hat \pi_{h+1:H}$ (as in PSDP~\citep{bagnell2004policy}).  Using such targets can introduce additional
errors, and stronger realizability-type assumptions on $\Pi,\Gcal$ are
necessary for polynomial sample-complexity (see
Appendix~\ref{sec:alternativealgs} and~\ref{sec:examples}).
Nevertheless, these algorithms are also oracle-efficient and while we
only establish statistical efficiency with deterministic hidden state
dynamics, we believe that they considerably expand the space of
plausible algorithms for the general setting.

\section{Conclusion}
This paper describes new RL algorithms for environments with rich
stochastic observations and deterministic hidden state dynamics. Unlike other
existing approaches, these algorithms are computationally efficient in
an oracle model, and we emphasize that the oracle-based approach has
led to practical algorithms for many other settings.
We believe this work
represents an important step toward computationally and statistically
efficient RL with rich observations.

While challenging benchmark environments in modern RL (e.g. visual
grid-worlds \citep{johnson2016malmo}) often have the assumed
deterministic hidden state dynamics, the natural goal is to develop
efficient algorithms that handle stochastic hidden-state dynamics.  We
show that the only known approach for this setting is not
implementable with standard oracles, and we also provide several
constructions demonstrating other concrete challenges of RL with
stochastic state dynamics. This provides insights into the key open
question of whether we can design an efficient algorithm for the
general setting.
We hope to resolve this question in future work.

\bibliography{fast_RL}

\begin{thebibliography}{47}
\providecommand{\natexlab}[1]{#1}
\providecommand{\url}[1]{\texttt{#1}}
\expandafter\ifx\csname urlstyle\endcsname\relax
  \providecommand{\doi}[1]{doi: #1}\else
  \providecommand{\doi}{doi: \begingroup \urlstyle{rm}\Url}\fi

\bibitem[Jiang et~al.(2017)Jiang, Krishnamurthy, Agarwal, Langford, and
  Schapire]{jiang2017contextual}
Nan Jiang, Akshay Krishnamurthy, Alekh Agarwal, John Langford, and Robert~E.
  Schapire.
\newblock Contextual decision processes with low {B}ellman rank are
  {PAC}-learnable.
\newblock In \emph{International Conference on Machine Learning}, 2017.

\bibitem[Mnih et~al.(2015)Mnih, Kavukcuoglu, Silver, Rusu, Veness, Bellemare,
  Graves, Riedmiller, Fidjeland, Ostrovski, Petersen, Beattie, Sadik,
  Antonoglou, King, Kumaran, Wierstra, Legg, and Hassabis]{mnih2015human}
Volodymyr Mnih, Koray Kavukcuoglu, David Silver, Andrei~A. Rusu, Joel Veness,
  Marc~G. Bellemare, Alex Graves, Martin Riedmiller, Andreas~K. Fidjeland,
  Georg Ostrovski, Stig Petersen, Charles Beattie, Amir Sadik, Ioannis
  Antonoglou, Helen King, Dharshan Kumaran, Daan Wierstra, Shane Legg, and
  Demis Hassabis.
\newblock Human-level control through deep reinforcement learning.
\newblock \emph{Nature}, 2015.

\bibitem[Kearns and Singh(2002)]{kearns2002near}
Michael Kearns and Satinder Singh.
\newblock Near-optimal reinforcement learning in polynomial time.
\newblock \emph{Machine Learning}, 2002.

\bibitem[Brafman and Tennenholtz(2003)]{brafman2003r}
Ronen~I. Brafman and Moshe Tennenholtz.
\newblock R-max -- a general polynomial time algorithm for near-optimal
  reinforcement learning.
\newblock \emph{Journal of Machine Learning Research}, 2003.

\bibitem[Strehl and Littman(2005)]{strehl2005theoretical}
Alexander~L. Strehl and Michael~L. Littman.
\newblock A theoretical analysis of model-based interval estimation.
\newblock In \emph{International Conference on Machine learning}, 2005.

\bibitem[Strehl et~al.(2006)Strehl, Li, Wiewiora, Langford, and
  Littman]{strehl2006pac}
Alexander~L. Strehl, Lihong Li, Eric Wiewiora, John Langford, and Michael~L.
  Littman.
\newblock {PAC} model-free reinforcement learning.
\newblock In \emph{International Conference on Machine Learning}, 2006.

\bibitem[Auer et~al.(2009)Auer, Jaksch, and Ortner]{auer2009near}
Peter Auer, Thomas Jaksch, and Ronald Ortner.
\newblock Near-optimal regret bounds for reinforcement learning.
\newblock In \emph{Advances in Neural Information Processing Systems}, 2009.

\bibitem[Dann and Brunskill(2015)]{dann2015sample}
Christoph Dann and Emma Brunskill.
\newblock Sample complexity of episodic fixed-horizon reinforcement learning.
\newblock In \emph{Advances in Neural Information Processing Systems}, 2015.

\bibitem[Azar et~al.(2017)Azar, Osband, and Munos]{azar2017minimax}
Mohammad~Gheshlaghi Azar, Ian Osband, and R{\'e}mi Munos.
\newblock Minimax regret bounds for reinforcement learning.
\newblock In \emph{International Conference on Machine Learning}, 2017.

\bibitem[Dann et~al.(2017)Dann, Lattimore, and Brunskill]{dann2017unifying}
Christoph Dann, Tor Lattimore, and Emma Brunskill.
\newblock Unifying {PAC} and regret: Uniform {PAC} bounds for episodic
  reinforcement learning.
\newblock In \emph{Advances in Neural Information Processing Systems}, 2017.

\bibitem[Kakade et~al.(2003)Kakade, Kearns, and
  Langford]{kakade2003exploration}
Sham~M. Kakade, Michael Kearns, and John Langford.
\newblock Exploration in metric state spaces.
\newblock In \emph{International Conference on Machine Learning}, 2003.

\bibitem[Pazis and Parr(2013)]{pazis2013pac}
Jason Pazis and Ronald Parr.
\newblock {PAC} optimal exploration in continuous space {M}arkov decision
  processes.
\newblock In \emph{AAAI Conference on Artificial Intelligence}, 2013.

\bibitem[Grande et~al.(2014)Grande, Walsh, and How]{grande2014sample}
Robert Grande, Thomas Walsh, and Jonathan How.
\newblock Sample efficient reinforcement learning with gaussian processes.
\newblock In \emph{International Conference on Machine Learning}, 2014.

\bibitem[Pazis and Parr(2016)]{pazis2016efficient}
Jason Pazis and Ronald Parr.
\newblock Efficient {PAC}-optimal exploration in concurrent, continuous state
  {MDP}s with delayed updates.
\newblock In \emph{AAAI Conference on Artificial Intelligence}, 2016.

\bibitem[Wen and Van~Roy(2013)]{wen2013efficient}
Zheng Wen and Benjamin Van~Roy.
\newblock Efficient exploration and value function generalization in
  deterministic systems.
\newblock In \emph{Advances in Neural Information Processing Systems}, 2013.

\bibitem[Wen and Van~Roy(2017)]{wen2017efficient}
Zheng Wen and Benjamin Van~Roy.
\newblock Efficient reinforcement learning in deterministic systems with value
  function generalization.
\newblock \emph{Mathematics of Operations Research}, 2017.

\bibitem[Krishnamurthy et~al.(2016)Krishnamurthy, Agarwal, and
  Langford]{krishnamurthy2016contextual}
Akshay Krishnamurthy, Alekh Agarwal, and John Langford.
\newblock {PAC} reinforcement learning with rich observations.
\newblock In \emph{Advances in Neural Information Processing Systems}, 2016.

\bibitem[Hsu(2010)]{hsu2010algorithms}
Daniel~Joseph Hsu.
\newblock \emph{Algorithms for active learning}.
\newblock PhD thesis, UC San Diego, 2010.

\bibitem[Agarwal et~al.(2014)Agarwal, Hsu, Kale, Langford, Li, and
  Schapire]{agarwal2014taming}
Alekh Agarwal, Daniel Hsu, Satyen Kale, John Langford, Lihong Li, and Robert~E.
  Schapire.
\newblock Taming the monster: A fast and simple algorithm for contextual
  bandits.
\newblock In \emph{International Conference on Machine Learning}, 2014.

\bibitem[Ross and Bagnell(2014)]{ross2014reinforcement}
Stephane Ross and J~Andrew Bagnell.
\newblock Reinforcement and imitation learning via interactive no-regret
  learning.
\newblock \emph{arXiv:1406.5979}, 2014.

\bibitem[Chang et~al.(2015)Chang, Krishnamurthy, Agarwal, Daume~III, and
  Langford]{chang2015learning}
Kai-Wei Chang, Akshay Krishnamurthy, Alekh Agarwal, Hal Daume~III, and John
  Langford.
\newblock Learning to search better than your teacher.
\newblock In \emph{International Conference on Machine Learning}, 2015.

\bibitem[Allwein et~al.(2000)Allwein, Schapire, and
  Singer]{allwein2000reducing}
Erin~L. Allwein, Robert~E. Schapire, and Yoram Singer.
\newblock Reducing multiclass to binary: A unifying approach for margin
  classifiers.
\newblock \emph{Journal of Machine Learning Research}, 2000.

\bibitem[Johnson et~al.(2016)Johnson, Hofmann, Hutton, and
  Bignell]{johnson2016malmo}
Matthew Johnson, Katja Hofmann, Tim Hutton, and David Bignell.
\newblock {The Malmo Platform for artificial intelligence experimentation}.
\newblock In \emph{International Joint Conference on Artificial Intelligence},
  2016.

\bibitem[Kearns and Koller(1999)]{kearns1999efficient}
Michael Kearns and Daphne Koller.
\newblock Efficient reinforcement learning in factored {MDP}s.
\newblock In \emph{International Joint Conference on Artificial Intelligence},
  1999.

\bibitem[Li et~al.(2006)Li, Walsh, and Littman]{li2006towards}
Lihong Li, Thomas~J. Walsh, and Michael~L. Littman.
\newblock {Towards a unified theory of state abstraction for MDPs}.
\newblock In \emph{International Symposium on Artificial Intelligence and
  Mathematics}, 2006.

\bibitem[Azizzadenesheli et~al.(2016{\natexlab{a}})Azizzadenesheli, Lazaric,
  and Anandkumar]{azizzadenesheli2016romdp}
Kamyar Azizzadenesheli, Alessandro Lazaric, and Animashree Anandkumar.
\newblock Reinforcement learning in rich-observation {MDP}s using spectral
  methods.
\newblock \emph{arXiv:1611.03907}, 2016{\natexlab{a}}.

\bibitem[Azizzadenesheli et~al.(2016{\natexlab{b}})Azizzadenesheli, Lazaric,
  and Anandkumar]{azizzadenesheli2016reinforcement}
Kamyar Azizzadenesheli, Alessandro Lazaric, and Animashree Anandkumar.
\newblock Reinforcement learning of {POMDP}s using spectral methods.
\newblock In \emph{Conference on Learning Theory}, 2016{\natexlab{b}}.

\bibitem[Guo et~al.(2016)Guo, Doroudi, and Brunskill]{guo2016pac}
Zhaohan~Daniel Guo, Shayan Doroudi, and Emma Brunskill.
\newblock A {PAC} {RL} algorithm for episodic {POMDP}s.
\newblock In \emph{Artificial Intelligence and Statistics}, 2016.

\bibitem[Russo and Van~Roy(2013)]{russo2013eluder}
Dan Russo and Benjamin Van~Roy.
\newblock Eluder dimension and the sample complexity of optimistic exploration.
\newblock In \emph{Advances in Neural Information Processing Systems}, 2013.

\bibitem[Osband and Van~Roy(2014)]{osband2014model}
Ian Osband and Benjamin Van~Roy.
\newblock Model-based reinforcement learning and the {E}luder dimension.
\newblock In \emph{Advances in Neural Information Processing Systems}, 2014.

\bibitem[Anthony and Bartlett(2009)]{anthony2009neural}
Martin Anthony and Peter~L Bartlett.
\newblock \emph{Neural network learning: Theoretical foundations}.
\newblock Cambridge University Press, 2009.

\bibitem[Dai et~al.(2018)Dai, Shaw, Li, Xiao, He, Liu, Chen, and
  Song]{dai2018sbeed}
Bo~Dai, Albert Shaw, Lihong Li, Lin Xiao, Niao He, Zhen Liu, Jianshu Chen, and
  Le~Song.
\newblock Sbeed: Convergent reinforcement learning with nonlinear function
  approximation.
\newblock In \emph{International Conference on Machine Learning}, pages
  1133--1142, 2018.

\bibitem[Beygelzimer et~al.(2009)Beygelzimer, Langford, and
  Ravikumar]{beygelzimer2009error}
Alina Beygelzimer, John Langford, and Pradeep Ravikumar.
\newblock Error-correcting tournaments.
\newblock In \emph{International Conference on Algorithmic Learning Theory},
  2009.

\bibitem[Langford and Beygelzimer(2005)]{langford2005sensitive}
John Langford and Alina Beygelzimer.
\newblock Sensitive error correcting output codes.
\newblock In \emph{International Conference on Computational Learning Theory},
  2005.

\bibitem[Langford and Zhang(2008)]{langford2008epoch}
John Langford and Tong Zhang.
\newblock The epoch-greedy algorithm for multi-armed bandits with side
  information.
\newblock In \emph{Advances in Neural Information Processing Systems}, 2008.

\bibitem[Ross(2013)]{ross2013interactive}
Stephane Ross.
\newblock \emph{Interactive learning for sequential decisions and predictions}.
\newblock PhD thesis, Carnegie Mellon University, 2013.

\bibitem[Khachiyan(1980)]{khachiyan1980polynomial}
Leonid~G Khachiyan.
\newblock Polynomial algorithms in linear programming.
\newblock \emph{USSR Computational Mathematics and Mathematical Physics}, 1980.

\bibitem[Gordon(1995)]{gordon1995stable}
Geoffrey~J Gordon.
\newblock Stable function approximation in dynamic programming.
\newblock In \emph{International Conference on Machine Learning}, 1995.

\bibitem[Bagnell et~al.(2004)Bagnell, Kakade, Schneider, and
  Ng]{bagnell2004policy}
J~Andrew Bagnell, Sham~M Kakade, Jeff~G Schneider, and Andrew~Y Ng.
\newblock Policy search by dynamic programming.
\newblock In \emph{Advances in Neural Information Processing Systems}, 2004.

\bibitem[Arora et~al.(2012)Arora, Hazan, and Kale]{arora2012multiplicative}
Sanjeev Arora, Elad Hazan, and Satyen Kale.
\newblock The multiplicative weights update method: a meta-algorithm and
  applications.
\newblock \emph{Theory of Computing}, 2012.

\bibitem[Munos and Szepesv{\'a}ri(2008)]{munos2008finite}
R{\'e}mi Munos and Csaba Szepesv{\'a}ri.
\newblock Finite-time bounds for fitted value iteration.
\newblock \emph{Journal of Machine Learning Research}, 2008.

\bibitem[Antos et~al.(2008)Antos, Szepesv{\'a}ri, and Munos]{antos2008learning}
Andr{\'a}s Antos, Csaba Szepesv{\'a}ri, and R{\'e}mi Munos.
\newblock Learning near-optimal policies with {B}ellman-residual minimization
  based fitted policy iteration and a single sample path.
\newblock \emph{Machine Learning}, 2008.

\bibitem[Gretton et~al.(2012)Gretton, Borgwardt, Rasch, Sch{\"o}lkopf, and
  Smola]{gretton2012kernel}
Arthur Gretton, Karsten~M. Borgwardt, Malte~J. Rasch, Bernhard Sch{\"o}lkopf,
  and Alexander~J. Smola.
\newblock A kernel two-sample test.
\newblock \emph{Journal of Machine Learning Research}, 2012.

\bibitem[Sch{\"o}lkopf and Smola(2002)]{scholkopf2002learning}
Bernhard Sch{\"o}lkopf and Alexander~J. Smola.
\newblock \emph{Learning with kernels: support vector machines, regularization,
  optimization, and beyond}.
\newblock MIT Press, 2002.

\bibitem[Gr{\"o}tschel et~al.(1981)Gr{\"o}tschel, Lov{\'a}sz, and
  Schrijver]{grotschel1981ellipsoid}
Martin Gr{\"o}tschel, L{\'a}szl{\'o} Lov{\'a}sz, and Alexander Schrijver.
\newblock The ellipsoid method and its consequences in combinatorial
  optimization.
\newblock \emph{Combinatorica}, 1981.

\bibitem[Ernst et~al.(2005)Ernst, Geurts, and Wehenkel]{ernst2005tree}
Damien Ernst, Pierre Geurts, and Louis Wehenkel.
\newblock Tree-based batch mode reinforcement learning.
\newblock \emph{Journal of Machine Learning Research}, 2005.

\bibitem[Farahmand et~al.(2010)Farahmand, Szepesv{\'a}ri, and
  Munos]{farahmand2010error}
Amir-Massoud Farahmand, Csaba Szepesv{\'a}ri, and R{\'e}mi Munos.
\newblock Error propagation for approximate policy and value iteration.
\newblock In \emph{Advances in Neural Information Processing Systems}, 2010.

\end{thebibliography}
\bibliographystyle{unsrtnat}

\clearpage

\appendix
\onecolumn
\appendixpage
\startcontents[sections]
\tableofcontents
\newpage
\section{Additional Notation and Definitions}
\label{sec:moredefinitions}
In the next few sections we analyze the new algorithms for the deterministic setting. We will adopt the following conventions:
\begin{itemize}
\item In the deterministic setting (which we focus on here), a path $p$ always deterministically leads to some state $s$, so we use them interchangeably, e.g., $V^\star(p) \equiv V^\star(s)$, $x \sim p \Leftrightarrow x \sim s$.
\item It will be convenient to define $V^\pi(s) := \Ex[\sum_{h'=h}^H r_{h'} \,|\, s_{h} = s, a_{h:H} \sim \pi]$ for $s$ at level $h$, which is the analogy of $V^\star(s)$ for $\pi$. Recall that $V^\pi \equiv V^\pi(\rt)$ and $V^\star \equiv V^\star(\rt)$. Also define $Q^\star(s, \pi) := \Ex_{x\sim s}[Q^\star(x, \pi(x))]$.
    \item We use $\EEx_D[\cdot]$ to denote empirical expectation over
  samples drawn from data set $D$, and we use $\Ex_p[\cdot]$ to denote
  population averages where data is drawn from path $p$. Often for
  this latter expectation, we will draw $(x,a,r,x')$ where $x \sim p,
  a \sim \textrm{Unif}(\Acal)$ and $r,x'$ are sampled according to the
  appropriate conditional distributions. In the notation $\Ex_p$ we
  default to the uniform action distribution unless otherwise
  specified.

\end{itemize}

\subsection{Additional Oracles}

\paragraph{Least-Squares (LS) Oracle}
The least-squares oracle takes as inputs a parameter $\epssub$ and a sequence
$\{(x\ii, v\ii)\}_{i\in[n]}$ of observations $x\ii \in \Xcal$ and values $v\ii
\in \RR$. It outputs a value function $\hat g \in \Gcal$ whose squared
error is $\epssub$ close to the least-squares fit
\begin{align}
    \min_{g \in \Gcal} \frac{1}{n}\sum_{i=1}^n (v\ii - g(x\ii))^2.
\end{align}

\paragraph{Multi Data Set Classification Oracle}
The multi data set classification oracle receives as inputs a parameter $\epsfeas$, $m$ scalars that are upper bounds on the allowed cost $\{U_j\}_{j \in [m]} \in \RR^{m}$, and $m$ cost-sensitive classification data sets $D_1, \dots D_m$, each of which
consists of a sequence
 of observations
$\{x_j\ii\}_{i \in [n]} \in \Xcal^n$ and a sequence of cost vectors $\{c\ii_j\}_{i \in [n]} \in \RR^{K \times n}$, where $c\ii_j (a)$ is the cost of predicting action $a \in \Acal$ for $x\ii_j$. The oracle returns a policy that
achieves on each data set $D_j$ at most an average cost of $U_j + \epsfeas$, if a policy exists in $\Pi$ that achieves costs at most $U_j$ on each dataset. Formally, the oracle returns a policy in
\begin{align}
    \left\{ \pi \in \Pi ~ \bigg| ~ \forall j \in [m]: ~ \frac 1 n \sum_{i =1}^{n} c\ii_j(\pi(x\ii_j)) \leq U_j + \epsfeas \right\}.
\end{align}
This oracle generalizes the CSC oracle by requiring the same policy to achieve low cost on multiple CSC data sets simultaneously.
Nonetheless, it can be implemented with a CSC oracle as follows: We associate a Lagrange parameter with each constraint, and optimize the Lagrange parameters
using multiplicative weights. In each iteration, we use the
multiplicative weights to combine the $m$ constraints into a single one, and then solve the resulting cost-sensitive problem with the CSC
oracle. The slack in the constraint as witnessed by the resulting
policy is used as the loss to update the multiplicative weights
parameters. See~\cite{arora2012multiplicative} for more details.

\subsection{Assumptions on the Function Classes}
While \mainAlg only requires realizability of the policy and the value function
classes, our other algorithms require stronger assumptions which we introduce
below.

\begin{assum}[Policy realizability] \label{asm:pol_realize}
$\pi^\star \in \Pi$.
\end{assum}

\begin{assum}[Value realizability]
\label{asm:val_realize}
$g^\star \in \Gcal$.
\end{assum}

\begin{assum}[Policy-value completeness] \label{asm:polval_complete}
At each level $h$, $\forall g' \in \Gcal_{h+1}$, there exists $\pi_{g'}^\star \in \Pi_h$ such that $\forall x\in\Xcal$,
\begin{align*}
\pi_{g'}^\star(x) = \argmax_{a\in\Acal} \Ex [r + g'(x_{h+1}) | x_h = x, a_h = a].
\end{align*}
In addition, $\forall  g'\in\Gcal_{h+1}$, $\exists g_{\star, g'}\in\Gcal_h$ s.t.~$\forall x \in \Xcal$,
\begin{align*}
g_{\star, g'}(x) = \Ex [r + g'(x') | x_h = x, a_h = \pi_{g'}^\star(x)].
\end{align*}
\end{assum}

\begin{assum}[Policy completeness] \label{asm:pol_complete}
  For every $h$, and every non-stationary policy $\pi_{h+1:H}$, there
  exists a policy $\pi \in \Pi_h$ such that, for all $x \in \Xcal_h$, we have
  \begin{align*}
    \pi(x) = \textstyle \argmax_{a} \Ex[\sum_{h'=h}^H r_{h'} | x, a, a_{h+1:H}\sim \pi_{h+1:H}].
  \end{align*}
\end{assum}

\begin{fact}[Relationship between the assumptions]~\\
Assum.\ref{asm:polval_complete} $\Rightarrow$ Assum.\ref{asm:pol_complete} $\Rightarrow$ Assum.\ref{asm:pol_realize}. \quad Assum.\ref{asm:polval_complete} $\Rightarrow$
Assum.\ref{asm:val_realize}.
\end{fact}

In words, these assumptions ask that for any possible approximation of
the future value that we might use, 
the induced square loss or cost-sensitive problems are
realizable using $\Gcal, \Pi$, which is a much stronger notion of
realizability than Assumptions~\ref{asm:pol_realize}
and~\ref{asm:val_realize}. Such assumptions are closely related to the
conditions needed to analyze Fitted Value/Policy Iteration methods
\citep[see e.g.,][]{munos2008finite, antos2008learning}, and are
further justified by Theorem~\ref{thm:backup} in Appendix~\ref{sec:examples}.

\section{Analysis of \algname} \label{app:valor}
\begin{table}
    \begin{center}
    \begin{tabular}{|c|}\hline
        \(\displaystyle \epsstat = \epssub = \epsfeas = \frac{\epsilon}{2^6 7 H^2 \Tmax}\)\\[5mm]
        \(\displaystyle \phi_h = (H-h+1)(6\epsstat + 2\epssub + \epsfeas)\)\\[5mm]
        \(\displaystyle\Tmax = MH\nexp + M\)\\[5mm]
        \(\displaystyle\ntest = \frac{\log\left(12KH\Tmax|\Gcal|/\delta\right)}{2\epsstat^2},\)\\[5mm]
        \(\displaystyle\ntrain= \frac{16K\log(12 H\Tmax |\Gcal||\Pi|/\delta)}{\epsstat^2},\)\\[5mm]
        \(\displaystyle\nexp = \frac{8\log(4MH/\delta)}{\epsilon},\)\\[5mm]
        \(\displaystyle\neval = \frac{32\log(8MH/\delta)}{\epsilon^2}\)\\[5mm]
        \hline
\end{tabular}
    \end{center}
    \caption{Exact values of parameters of \mainAlg run with inputs $\epsilon, \delta \in (0, 1)$ and $M, K \in \NN$.}
    \label{tab:param}
\end{table}

\newcommand{\Slearned}{\Scal^{\textrm{learned}}}
\begin{definition}
  A state $s \in \Scal_h$ is called \emph{learned} if there is a
  data set in $\Dcal_h$ that is sampled from a path leading to that
  state. The set of all learned states at level $h$ is $\Slearned_h$
  and 
  $\Slearned := \bigcup_{h \in [H]} \Slearned_h$.
 \end{definition}

\subsection{Concentration Results}
We now define an event $\Ecal$ that holds with high probability and will be the main concentration argument in the proof. This event uses a parameter $\epsstat$ whose value we will set later.

\begin{definition}[Deviation Bounds]
    \label{def:devbounds}
    Let $\Ecal$ denote the event that for all $h \in [H]$ the total number of calls to $\dfslearn(p)$ at level $h$ is at most $\Tmax = MH \nexp + M$ during the execution of \metaalg and that for all these calls to $\dfslearn(p)$ the following deviation bounds hold for all
    $g \in \Gcal_h$ and $\pi \in \Pi_h$ (where $D'_a$ is a data set of
    $\ntest$ observations sampled from $p \circ a$ in
    Line~\ref{lin:sample_consensus}, and $\tilde D$ is the data set of
    $\ntrain$ samples from Line~\ref{lin:sample_learn} with stored
    values $\{V_a\}_{a \in \Acal}$):
    \begin{align}
      &   \left|\hat{\Ex}_{D'_a}[g(x_{h+1})] - \Ex_{p \circ a}[g(x_{h+1})]\right| \leq \epsstat, \quad \forall a \in \Acal \label{eq:devb1}\\
      &   \left|\hat{\Ex}_{\tilde D}[g(x_h)] - \Ex_{p}[g(x_h)]\right| \leq  \epsstat\\
      &  \left|\hat{\Ex}_{\tilde D}[K \one \{\pi(x_h) = a_h\}(r_h + V_a)] - \Ex_{p}[K\one\{\pi(x_h) = a_h\}(r_h + V_a)]\right| \leq  \epsstat.\label{eq:devb2}
    \end{align}
\end{definition}

In the next Lemma, we bound $\prob[\Ecal]$, which is the main
concentration argument in the proof. The bound involves a new quantity
$T_{\max}$ which is the maximum number of calls to \dfslearn. We will
control this quantity later.
\begin{lemma}
    \label{lem:probE}
    Set
\begin{align*}
        \ntest \geq  \frac{1}{2\epsstat^2} \ln \left(\frac{ 12 K H\Tmax|\Gcal|}{\delta} \right), \qquad
        \ntrain \geq   \frac{16K}{\epsstat^2} \ln \left(\frac{ 12 H\Tmax |\Gcal||\Pi|}{\delta} \right).
\end{align*}
Then $\prob[\Ecal] \ge 1 - \delta/2$.
\end{lemma}
\begin{proof}
    Let us denote the total number of calls to \dfslearn before the algorithm stops by $\Ndfs$ (which is random) and first focus on the $j$-th call to \dfslearn. 
    Let $\Bcal_j$ be the sigma-field of all samples collected before the $j$th call to \dfslearn (if it exists, or otherwise the last call to \dfslearn) and all intrinsic randomness of the algorithm.
    The current path is denoted by $p_j$ at level $h_j$ and data sets $D'_{a}$, $\tilde D$ collected are denoted by $D'_{j,a}$ and $\tilde D_{j}$ respectively.
    Consider a fix $a \in \Acal$ and $g \in \Gcal$ and define 
    \begin{align}
        Y_{i,j} = \begin{cases}
            0 & \textrm{if } j > N_{dfs}\\
            g(x_{h+1}^{(i,j)}) - \Ex_{p_j \circ a}[ g(x_{h+1})] & \textrm{otherwise}
        \end{cases}
    \end{align}
    which is well-defined even if $j > N_{dfs}$ and where $x_{h+1}^{(i,j)}$ is the $i$-th sample of $x_{h+1}$ in $D'_{j,a}$. 
Since $|Y_{i,j}| \leq 1$ and since contexts $x_{h+1}$ are sampled i.i.d. from $p_j \circ a$ conditioned on $p_j$ which is measurable in $\Bcal_j$, we get by Hoeffding's lemma that $\Ex[\exp(\lambda Y_{i,j}) | Y_{1:i-1, j}, \Bcal_j] = \Ex[\exp(\lambda Y_{i,j}) | \Bcal_j] \leq \exp(\lambda^2 / 2)$ for $\lambda \in \RR$. 
    As a result, we have $
    \Ex[ \exp(\lambda \sum_{i=1}^{\ntest} Y_{i,j})] 
    = \Ex[\Ex[ \exp(\lambda \sum_{i=1}^{\ntest} Y_{i,j}) | \Bcal_j]] 
    \leq \exp(\ntest \lambda^2 /2)$ and by Chernoff's bound the following concentration result holds
\begin{align*}
     \left| \hat{\Ex}_{D'_{j,a}}[g(x_{h+1})] - \Ex_{p_j \circ a}[g(x_{h+1})]\right|
\leq \sqrt{\frac{\log(2K|\Gcal|/\delta')}{2\ntest}}
\end{align*}
    with probability at least $1 - \frac{\delta'}{K|\Gcal|}$ for a fixed $a$ and $g$ and $j$ as long as $j \leq N_{dfs}$.
    With a union bound over $\Acal$ and $\Gcal$, the following statement holds: Given a fix $j \in \NN$, with probability at least $1-\delta'$, if $j \leq N_{dfs}$ then for all $g \in \Gcal_{h+1}$ and $a \in \Acal$    
    \begin{align*}
    \left| \hat{\Ex}_{D'_{j,a}}[g(x_{h+1})] - \Ex_{p_j\circ a}[g(x_{h+1})]\right|
\leq \sqrt{\frac{\log(2K|\Gcal|/\delta')}{2\ntest}}.
\end{align*}
Choosing 
    $ \ntest \geq  \frac{1}{2\epsstat^2} \ln \left(\frac{ 12 K H\Tmax|\Gcal|}{\delta} \right)$ and  $\delta' =
\frac{\delta}{6HT_{\max}}$ allows us to bound the LHS by $\epsstat$.
In exactly the same way since the data set $\tilde D_j$
consists of $\ntrain$ samples that, given $\Bcal_j$, are sampled i.i.d. from $p_j$,  we have for all $g \in \Gcal_h$
\begin{align*}
    \left| \hat{\Ex}_{\tilde D_j} [g(x_h)] - \Ex_{p_j}[g(x_h)]\right| \leq \sqrt{\frac{\log(2|\Gcal|/\delta')}{2\ntrain}},
\end{align*}
    with probability $1-\delta'$ as long as $j \leq N_{dfs}$. As above, our choice of
$\ntrain$ ensures that this deviation is bound by $\epsstat$.

Finally, for the third inequality we must use Bernstein's
inequality. For the random variable $K\one\{\pi(x_h) = a_h\}
(r_h+V_{a_h})$, since $a_h$ is chosen uniformly at random, it is not
hard to see that both the variance and the range are at most $2K$ (see for example Lemma~14 by \citet{jiang2017contextual}). As
such, Bernstein's inequality with a union bound over $\pi \in \Pi$ gives that with probability $1-\delta'$,
\begin{align*}
    \left| (\hat{\Ex}_{\tilde D_j}- \Ex_{p_j})[K \one\{\pi(x_h)=a_h\} (r_h + V_{a_h})]\right| \leq \sqrt{\frac{4 K\log(2 |\Pi|/\delta')}{\ntrain}} + \frac{4K}{3\ntrain}\log(2 |\Pi|/\delta') \leq \epsstat,
\end{align*}
    since $\{V_a\}$ and $p_j$ can essentially be considered fixed at the time when $\tilde D_j$ is collected (a more formal treatment is analogous to the proof of the first two inequalities).
Using a union bound, the deviation bounds \eqref{eq:devb1}--\eqref{eq:devb2} hold for a single call to \dfslearn with probability $1 - 3 \delta'$.

    Consider now the event $\Ecal'$ that these bounds hold for the first $\Tmax$ calls at each level $h$. Applying a union bound let us bound
$\prob(\Ecal') \geq 1 - 3 H \Tmax \delta' = 1 - \frac{\delta}{2}$.
It remains to show that $\Ecal' \subseteq \Ecal$.

First note that in event $\Ecal'$ in the first $\Tmax$ calls to \dfslearn, the algorithm does not call itself recursively if $p \circ a$ leads to a learned state. To see this assume $p \circ a$ leads to a state $s \in \Slearned$. Let $D'_a$ be
  the data set collected in Line~\ref{lin:sample_consensus} for this
  action $a$. Since the subsequent state $s \in \Slearned$, then there
  is a data set $(D, V, \{V_b\}) \in \Dcal_{h+1}$ sampled from this
  state (we will only use the first two items in the tuple). This
  means that $D'_a$ and $D$ are two data sets sampled from the same
  distribution, and as such, we have
    \begin{align*}
        V_{opt} - V_{pes}
        & = \EEx_{D_a'}[g_{opt}(x_{h+1}) - g_{pes}(x_{h+1})]
        \leq \Ex_s[g_{opt}(x_{h+1}) - g_{pes}(x_{h+1})] + 2\epsstat\\
        & \leq  \EEx_D[g_{opt}(x_{h+1}) - g_{pes}(x_{h+1})] + 4 \epsstat\\
        & \leq V + \phi_{h+1} + \epsfeas - V + \phi_{h+1} + \epsfeas + 4 \epsstat
        = 2\phi_{h+1} + 4 \epsstat + 2 \epsfeas.
    \end{align*}
    The last line holds because the constraints for $g_{opt}$ and $g_{pes}$ include the one based on $(D, V)$ (Line~\ref{lin:local_value_vopt}), so the expectation of  $g_{opt}$ and $g_{pes}$ on $D$ can only differ by the amount of the allowed slackness $2\phi_{h+1}$ and the violations of feasibility $2\epsfeas$.
    Therefore the condition in the if clause is satisfied and the algorithm
    does not call itself recursively. We here assumed that the constrained optimization problem has an approximately feasible solution but if
    that is not the case, the if condition is trivially satisfied.

    Since the number of learned states per level is bounded by $M$, this means
    that within the first $\Tmax$ calls to \dfslearn, the algorithm can make recursive calls
    to the level below at most $M$ times.  Further note that for any fixed
    level $h$ the total number of non-recursive calls to \dfslearn is bounded
    by $MH\nexp$ since \metaalg has at most $MH$ iterations
    and in each \dfslearn is called $\nexp$ times at each level (but the first). Therefore, in
    event $\Ecal'$, the total number of calls to \dfslearn at any level $h$ is
    bounded by $MH\nexp + M \leq \Tmax$ and the
    statement follows. 
\end{proof}

\subsection{Bound on Oracle Calls}

\begin{proof}[Proof of Theorem~\ref{thm:local_value_comp}]
Consider event $\Ecal$ from Definition~\ref{def:devbounds} which by Lemma~\ref{lem:probE} has probability at least $1 - \delta/2$.
\mainAlg requires two types of nontrivial computations over the
function classes. We show that they can be reduced to CSC on $\Pi$ and LP on $\Gcal$ (recall Sec.~\ref{sec:oracle}), respectively, and hence \mainAlg is oracle-efficient.

First, Line~\ref{lin:local_value_V} in \dfslearn involves optimizing $V_D(\pi; \{V_a\})$ (Eq.~\eqref{eq:shortened_cs}) over $\Pi$,
which can be reduced to CSC as follows: We first form tuples
$(x\ii, a\ii, y\ii)$ from $D$ and $\{V_a\}$ on which $V_D(\pi; \{V_a\})$ depends, where we bind $x_h$ to $x\ii$, $a_h$ to $a\ii$,
and $r_h+V_{a_h}$ to $y\ii$. From the tuples, we  construct a CSC data set
$(x\ii, -[K\one\{a=a\ii\}y\ii]_{a \in \Acal})$, where the second
 argument is a $K$-dimensional vector with one non-zero.
    On this data set, the cost-sensitive risk of any policy (interpreted as a classifier) is exactly $-V_D(\pi; \{V_a\})$, so minimizing risk (which the oracle does) maximizes the original objective.\footnote{Note that the inputs to the oracle have polynomial length: $D$ consists of polynomially many $(x, a, r, x')$ tuples, each of which should be assumed to have polynomial description length, and $\{V_a\}$ similarly.}

Second, the optimization in Line~\ref{lin:local_value_fit_policy} in \polvalfun can be reduced to CSC with the very same argument, except that we now accumulate all CSC inputs for each data set in $\Dcal_h$. Since $|\Dcal_h| \leq \Tmax$ is polynomial, the total input size is still polynomial.

    Third, the state identity test in Line~\ref{lin:local_value_vopt} in \dfslearn requires solving the following  problem over the function class $\Gcal$:
\begin{align}
& V_{opt} = \max_{g\in\Gcal} \EEx_{D'}[g(x_h)] \quad \textrm{(and $\min$ for $V_{pes}$)}\\
\textrm{s.t.} &~ V - \phi_h \le \EEx_D[g(x_h)] \le V+\phi_h, \forall (D, V) \in \Dcal_h.
\end{align}
The objective and the constraints are linear functionals of $\Gcal$, all empirical expectations involve polynomially many samples, and the number of
constraints is $|\Dcal_h| \leq \Tmax$ which remains polynomial
throughout the execution of the algorithm. Therefore, the LP oracle can directly handle this optimization problem.

Altogether, we showed that all non-trivial computations can be reduced to oracle calls with inputs with polynomial description length.
    It remains to show that the number of calls is bounded. Since there are at most $\Tmax$ calls to \dfslearn at each level $h \in [H]$, the total number of calls to the LP oracle is $\Tmax H K$. Similarly, the number of CSC oracle calls from \dfslearn is at most $\Tmax H$. In addition, there at at most $MH$ calls to the CSC oracle in \polvalfun.
    The statement follows with realizing that $\Tmax = MH \nexp + M  = O\left(\frac{MH}{\epsilon} \ln \left( \frac{MH}{\delta}\right) \right)$.
\end{proof}

\subsection{Depth First Search and Estimated Values}
In this section, we show that in the high-probability event $\Ecal$
(Definition~\ref{def:devbounds}), \dfslearn produces good estimates of optimal values on learned states.
The next lemma first quantifies the error in the value estimate at
level $h$ in terms of the estimation error of the values of the next
time step $\{V_a\}_a$.
\begin{lemma}[Error propagation when learning a state]
\label{lem:errorprop}
Consider a call to \dfslearn with input path
$p$ of depth $h$. Assume that all values $\{V_a\}_{a \in \Acal}$ in
Algorithm~\ref{alg:local_value_dfslearn} satisfy
$|V_a - V^\star(p\circ a)| \leq \beta$ for some $\beta > 0$. Then
in event $\Ecal$, $\tilde V$ returned in Line~\ref{lin:return2}
satisfies $|\tilde V - V^\star(p)| \leq \epsstat + \beta + \epssub$.
\end{lemma}
\begin{proof}
  The proof follows a standard analysis of empirical risk minimization
  (here we are maximizing). Let $\tilde \pi$ denote the empirical risk
  maximizer in Line~\ref{lin:local_value_V} and let $\pi^\star$ denote
  the globally optimal policy (which is in our class due to
  realizability). Then
\begin{align*}
    \tilde V & \leq \EEx_{\tilde D} [ K \one\{ \tilde \pi(x_h) = a_h\}(r_h + V_{a_h})]
    \leq \Ex_p [ K \one\{ \tilde \pi(x_h) = a_h\}(r_h + V_{a_h})] + \epsstat\\
    & \leq \Ex_p [ K\one\{ \tilde \pi(x_h) = a_h\}(r_h + g^\star(x_{h+1}))]
          +  \beta + \epsstat\\
    & \leq \Ex_p [ K \one\{ \pi^\star(x_h) = a_h\}(r_h + g^\star(x_{h+1}))]
    + \beta + \epsstat = V^\star(s) + \beta + \epsstat.
\end{align*}
The first inequality is the deviation bound, which holds in event
$\Ecal$. The second inequality is based on the precondition on
$\{V_a\}_{a \in \Acal}$, linearity of expectation, and the
realizability property of $g^\star_{h+1}$. The third inequality uses
that $\pi^\star$ is the global and point-wise maximizer of the
long-term expected reward, which is precisely $r_h+g^\star$.

Similarly, we can lower bound $\tilde V$ by
\begin{align*}
    \tilde V
    & = \EEx_{\tilde D} [ K \one\{ \tilde \pi(x_h) = a_h\}(r_h + V_{a_h})] - \epssub
    \geq  \EEx_{\tilde D} [ K \one\{ \pi^\star(x_h) = a_h\}(r_h + V_{a_h})] - \epssub\\
    & \geq \Ex_p [ K \one\{ \pi^\star(x_h) = a_h\}(r_h + V_{a_h})] - \epsstat - \epssub\\
    & \geq \Ex_p [ K \one\{ \pi^\star(x_h) = a_h\}(r_h + g^\star(x_{h+1}))]
         - \epsstat - \beta - \epssub
    = V^\star(s) - \epsstat - \beta - \epssub.
\end{align*}
Here we first use $\tilde V$ is optimal up to $\epssub$ and then  that $\tilde \pi$ is the empirical maximizer. Subsequently, we
leveraged the deviation bounds of event $\Ecal$ and finally used the
assumption about the estimation accuracy from the level below. This
proves the claim.
\end{proof}

The goal of the proof is to apply the above lemma inductively so that
we can learn all of the values to reasonable accuracy. Before doing
so, we need to quantify the estimation error when $V_a$ is set in
Line~\ref{lin:return1} of the algorithm without a recursive call.
\begin{lemma}[Error when not recursing]
    \label{lem:errornotlearn}
    Consider a call to \dfslearn with input
    path $p$ of depth $h$. If $g^\star$ is feasible for Line~\ref{lin:constraint2} of \dfslearn and
    $V_a$ is set in Line~\ref{lin:return1} of
    Algorithm~\ref{alg:local_value_dfslearn}, then in event $\Ecal$,
    the value $V_a = \frac{V_{opt} + V_{pes}}{2}$ satisfies
    $ |V_a - V^\star(p \circ a)| \leq \phi_{h+1} + 3 \epsstat + \epsfeas + \epssub$.
\end{lemma}
\begin{proof}
  Recall that $D'_a$ is the data set sampled in
  Line~\ref{lin:sample_consensus} for the particular action $a$ in
  consideration. Since $g^\star_{h+1}$ is feasible for both $V_{opt}$
  and $V_{pes}$, we have
    \begin{align*}
        V_{pes} - \epssub =& \EEx_{D_a'}[g_{pes}(x_{h+1})] - \epssub \leq
        \EEx_{D_a'}[g^\star(x_{h+1})]
        \leq
        \EEx_{D_a'}[g_{opt}(x_{h+1})] + \epssub = V_{opt} + \epssub.
    \end{align*}
    Without loss of generality, we can assume that $V_{pes} \leq V_{opt}$, otherwise we can just exchange them.
    This implies that
    $0 \leq V_{opt} - V_a = V_a - V_{pes} = \frac{V_{opt} - V_{pes}}{2}  \leq \phi_{h+1} + 2 \epsstat + \epsfeas$. Therefore,
    \begin{align*}
      \EEx_{D'}[g^\star(x_{h+1})] - V_a & \leq V_{opt} - V_a + \epssub = \frac{V_{opt} - V_{pes}}{2} + \epssub \leq \phi_{h+1} + 2\epsstat + \epsfeas + \epssub.\\
      V_a - \EEx_{D'}[g^\star(x_{h+1})] & \leq V_a - V_{pes}  + \epssub = \frac{V_{opt} - V_{pes}}{2} + \epssub \leq \phi_{h+1} + 2\epsstat + \epsfeas + \epssub.
    \end{align*}
	By the triangle inequality
    \begin{align*}
        |V_a - V^\star(p \circ a)|
        \leq& | \EEx_{D_a'}[g^\star(x_{h+1})] - V_a|  + | \EEx_{D_a'}[g^\star(x_{h+1})] - V^\star(p \circ a)| \\
        \leq & \phi_{h+1} + 3 \epsstat + \epsfeas + \epssub.
    \end{align*}
    The last inequality is the concentration statement, which holds in
    event $\Ecal$.
\end{proof}

We now are able to apply Lemma~\ref{lem:errorprop} inductively in combination with Lemma~\ref{lem:errornotlearn} to obtain the main result of \dfslearn in this section.
\begin{proposition}[Accuracy of learned values]
    \label{prop:Vaccuracy}
    Assume the realizability condition $g^\star \in \Gcal_h$.
    Set $\phi_h = (H+1 - h)(6\epsstat+2\epssub + \epsfeas)$ for all $h \in
    [H]$. Then under event $\Ecal$, for any level $h \in [H]$ and any
    state $s \in \Scal_h$ all triplets $(D, V, \{V_a\}) \in \Dcal_h$
    associated with state $s$ (formally with paths $p$ that lead to
    $s$) satisfy
    \begin{align*}
        |V - V^\star(s)| \leq \phi_h - 2 \epsstat, \qquad
        |V_a - V^\star(s \circ a)| \leq \phi_{h+1} + 3 \epsstat + \epsfeas +  \epssub.
    \end{align*}
    Moreover, under event $\Ecal$, we have $g^\star$ is feasible for Line~\ref{lin:constraint2} of \dfslearn for all $h$, at all times.
\end{proposition}
\begin{proof}
  We prove this statement by induction over $h$.  For $h=H+1$ the
  statement holds trivially since $\Gcal_{H+1} = \{ g^\star_{h+1} \}$
  the constant $0$ function is the only function in $\Gcal_{H+1}$ and
  therefore the algorithm always returns on Line~\ref{lin:return1} and
  never calls level $H+1$ recursively.

  Consider now some data set
  $(\tilde{D},\tilde{V}, \{V_a\}) \in \Dcal_h$ at level $h$ associated
  with state $s \in \Scal_h$. This data set was obtained by calling
  \dfslearn at some path $p$ (pointing to state $s$). Since when we
  added this data set, we have not yet exhausted the budget of $\Tmax$
  calls to \dfslearn (by the preconditions of the lemma), we have that
  the once we reach Line~\ref{lin:sample_learn} the inductive
  hypothesis applies for all data sets at level $h+1$ (which may have
  been added by recursive calls of this execution). Each of the $V_a$
  values can be set in one of two ways.
  \begin{enumerate}
  \item The algorithm did not make a recursive call. Since by the inductive assumption
    $g^\star$ is feasible for Line~\ref{lin:constraint2} of \dfslearn , we can apply
    Lemma~\ref{lem:errornotlearn} and get that
    \begin{align*}
      |V_a - V^\star(s \circ a)| \leq \phi_{h+1} + 3 \epsstat + \epsfeas + \epssub.
    \end{align*}
  \item The algorithm made a recursive call. Since the value returned was added as
    a data set at level $h+1$, it satisfies the inductive assumption
    \begin{align*}
      |V_a - V^\star(s \circ a)| \leq \phi_{h+1} - 2 \epsstat.
    \end{align*}
  \end{enumerate}
  This demonstrates the second inequality in the inductive step. For
  the first, applying Lemma~\ref{lem:errorprop} with
  $\beta = \phi_{h+1} + 3 \epsstat + \epsfeas + \epssub$, we get that
  $|\tilde V - V^\star(s)| \leq \phi_{h+1} + 4 \epsstat + \epsfeas + 2\epssub = \phi_h - 2
  \epsstat$, by definition of $\phi_h$.  Finally, this also implies that
  $|\tilde V - \EEx_{\tilde D}[g^\star_h(x_h)]| \leq |\tilde V -
  V^\star(s)| + |\EEx_{\tilde D}[g^\star_h(x_h)] - V^\star(s)| \leq
  \phi_h$ which means that $g^\star$ is still feasible.
\end{proof}

\subsection{Policy Performance}
In this section, we bound the quality of the policy returned by
\polvalfun in the good event $\Ecal$ by using the fact that \dfslearn
produces accurate estimates of the optimal values (previous section).
Before we state the main result of this section in
Proposition~\ref{prop:localvalue_polguarantee}, we prove the following
helpful lemma. This Lemma is essentially Lemma 4.3 in~\citet{ross2014reinforcement}.
\begin{lemma}
    \label{lem:subopteq}
    The suboptimality of a policy $\pi$ can be written as
\begin{align*}
    V^\star - V^{\pi}= \Ex\left[\sum_{h=1}^H (V^\star(s_h) - Q^\star(s_h, \pi_h)) ~|~ a_h \sim \pi_h \right].
\end{align*}
\end{lemma}
\begin{proof}
  The difference of values of a policy $\pi$ compared to the optimal
  policy in a certain state $s \in \Scal_h$ can be expressed as
  \begin{align*}
    V^\star(s) - V^{\pi}(s)
    & = V^\star(s) - \Ex_s[K\one\{\pi_h(x_h) = a_h\}(r_h + V^\pi(x_{h+1})]\\
    & = V^\star(s) - \Ex_s[K\one\{\pi_h(x_h) = a_h\}(r_h + V^\star(x_{h+1}) - V^\star(x_{h+1}) + V^\pi(x_{h+1}))]\\
    & = V^\star(s) - Q^\star(s,\pi_h) + \Ex_s[K\one\{\pi_h(x_h) = a_h\} (V^\star(x_{h+1}) - V^\pi(x_{h+1}))]\\
    & = V^\star(s) - Q^\star(s,\pi_h) + \Ex_s[V^\star(x_{h+1}) - V^\pi(x_{h+1}) \mid a_h \sim \pi_h].
  \end{align*}
  Therefore, by applying this equality recursively, the suboptimality of $\pi$ can be written as
  \begin{align}
    V^\star(s) - V^{\pi}
    = \Ex\left[\sum_{h=1}^H (V^\star(s_h) - Q^\star(s_h, \pi_h)) ~|~ a_h \sim \hat \pi_h \right]. \tag*\qedhere
  \end{align}
\end{proof}

Now we may bound the policy suboptimality.
\begin{proposition}
    \label{prop:localvalue_polguarantee}
    Assume $g^\star_h \in \Gcal_h$ and the we are in event
    $\Ecal$. Recall the definition $\phi_h = (H+1 - h)(6\epsstat+2\epssub + \epsfeas)$ for all
    $h \in [H]$.  Then the policy $\hat{\pi} = \hat \pi_{1:H}$
    returned by \polvalfun satisfies
    \begin{align*}
        V^{\hat{\pi}} \ge V^\star - p_{ul}^{\hat \pi} - 2 H^2 \Tmax (7\epsstat + 3\epssub + 2\epsfeas)
    \end{align*}
    where
    $p_{ul}^{\hat \pi}= \prob( \exists h \in [H] ~:~ s_h \notin
    \Slearned ~|~ a_{1:H} \sim \hat\pi)$ is the probability of hitting
    an unlearned state when following $\hat \pi$.
\end{proposition}
\begin{proof}
  To bound the suboptimality of the learned policy, we bound the
  difference of how much following $\hat \pi_h$ for one time step can
  hurt per state using Proposition~\ref{prop:Vaccuracy}. For a state
  $s \in \Slearned$ at level $h$, we have
  \begin{align*}
      & V^\star(s) - Q^\star(s, \hat \pi_h)\\
    & = \Ex_s[K(\one\{\pi^\star_h(x_h) = a_h\}- \one\{\hat{\pi}_h(x_h) = a_h\})(r_h + g^\star_{h+1}(x_{h+1})]\\
    & \leq \sum_{s \in \Slearned_h} \Ex_s[K(\one\{\pi^\star_h(x_h) = a_h\} - \one\{\hat{\pi}_h(x_h) = a_h\})(r_h + g^\star_{h+1}(x_{h+1})]\\
    & \leq \sum_{(s, \_, \{V_a\}) \in \Dcal_h} \Big(\Ex_s[K(\one\{\pi^\star_h(x_h) = a_h\} - \one\{\hat{\pi}_h(x_h) = a_h\})(r_h + V_{a_h})] \\
      & \qquad + 2 \phi_{h+1} + 6 \epsstat + 2 \epsfeas + 2 \epssub \Big)\\
    & \leq \sum_{(\tilde D, \_, \{V_a\}) \in \Dcal_h} \Big(\Ex_{\tilde D}[K(\one\{\pi^\star_h(x_h) = a_h\} - \one\{\hat{\pi}_h(x_h) = a_h\})(r_h + V_{a_h})] \\
      & \qquad + 2 \phi_{h+1} + 8 \epsstat + 2 \epsfeas + 2 \epssub \Big)\\
    & \leq 2 | \Dcal_h| (\phi_{h+1} + 4 \epsstat + \epsfeas + 2\epssub).
  \end{align*}
  Here the first identity is based on expanding definitions. For the
  first inequality, we use that $s \in \Slearned$ and also that
  $\pi^\star$ simultaneously maximizes the long term reward from all
  states, so the terms we added in are all non-negative. In the second
  inequality, we introduce the notation $(s, \_, \{V_a\}) \in \Dcal_h$ to
  denote a data set in $\Dcal_h$ associated with state $s$ with
  successor values $\{V_a\}$. For this inequality we use
  Proposition~\ref{prop:Vaccuracy} to control the deviation of the
  successor values. The third inequality uses the deviation bound that
  holds in event $\Ecal$.

    Since per \dfslearn call, only one data set can be added to $\Dcal_{h}$, the magnitude $|\Dcal_{h}| \leq \Tmax$ is bounded by the total number of calls to \dfslearn at each level. Using
  Lemma~\ref{lem:subopteq}, the suboptimality of $\hat \pi$ is therefore at most
\begin{align*}
    V^\star - V^{\hat \pi} & \leq p_{ul}^{\hat \pi} + (1 -p_{ul}^{\hat \pi}) \sum_{h=1}^H 2| \Dcal_h| (\phi_{h+1} + 4 \epsstat + 2\epssub + \epsfeas)\\
    & \leq  p_{ul}^{\hat \pi} + 2 H T_{\max}(4 \epsstat + 2\epssub + \epsfeas) + 2 T_{\max} \sum_{h=1}^H \phi_{h+1}\\
    & \leq p_{ul}^{\hat \pi} + 2 H T_{\max}(4 \epsstat + 2\epssub + \epsfeas)  +  2 T_{\max} (6 \epsstat + 2 \epssub + \epsfeas) \sum_{h=1}^H (H-h)\\
    & \leq p_{ul}^{\hat \pi} + 2 H T_{\max} (4 \epsstat + 2\epssub + \epsfeas) +  H^2 T_{\max} (6 \epsstat + 2 \epssub + \epsfeas)\\
    & \leq p_{ul}^{\hat \pi} +  14H^2T_{\max} \epsstat + 6H^2T_{\max} \epssub + 3H^2T_{\max} \epsfeas.
\end{align*}
This argument is similar to the proof of Lemma 8
in~\citet{krishnamurthy2016contextual}. Note that we introduce the
dependency on $T_{\max}$ since we perform joint policy optimization, which will
degrade the sample complexity.
\end{proof}

\subsection{Meta-Algorithm Analysis} \label{app:meta_analysis}
Now that we have the main guarantees for \dfslearn and \polvalfun, we
may turn to the analysis of \metaalg.
\begin{lemma}
    \label{lem:metaalg}
    Consider running \metaalg with \dfslearn and \polvalfun (Algorithm~\ref{alg:metaalg} + \ref{alg:local_value_polval_unconstr} + \ref{alg:local_value_dfslearn}) with parameters
    \begin{align*}
      \nexp \ge \frac{8}{\epsilon}\ln\left(\frac{4MH}{\delta}\right), \quad n_{eval} \geq \frac{32}{\epsilon^2}\ln\left( \frac{8MH}{\delta}\right),
        \quad \epsstat = \epssub = \epsfeas = \frac{\epsilon}{2^6 7 H^2 \Tmax}
    \end{align*}
    Then with probability at least $1-\delta$,
    \metaalg returns a policy that is at least $\epsilon$-optimal
    after at most $MK$ iterations.
\end{lemma}
\begin{proof}
  First apply Lemma~\ref{lem:probE} so that the good event $\Ecal$
  holds, except with probability $\delta/2$.

  In the event $\Ecal$, since before the first execution of
  \polvalfun, we called $\dfslearn(\rt)$, by
  Proposition~\ref{prop:Vaccuracy}, we know that
  $|\hat{V}^\star - V^\star| \leq \phi_1 - 2\epsstat$ where
  $\hat{V}^\star$ is the value stored in the only dataset associated
  with the root. This value does not change for the remainder of the
  algorithm, and the choice of $\epsstat, \phi$ ensure that
  \begin{align*}
    |\hat{V}^\star - V^\star| \leq \phi_1 - 2\epsstat = 6H\epsstat + 2H\epssub + H\epsfeas - 2\epsstat \leq \epsilon/8.
  \end{align*}
  This is true for all executions of \polvalfun (formally all
  $\hat{V}^{(k)}$ values).
  Next, since we perform at most $MH$ iterations of the loop in
  \metaalg, we consider at most $MH$ policies. Via a standard
  application of Hoeffding's inequality, with probability
  $1-\delta/4$, we have that for all $k \in [MH]$
  \begin{align*}
    |\hat{V}^{\hat{\pi}_k} - V^{\hat{\pi}_k}| \leq \sqrt{\frac{\log(8MH/\delta)}{2\neval}}.
  \end{align*}
  The choice of $\neval$ ensure that this is at most $\epsilon/8$.
  With these two bounds, if \metaalg terminates, the termination condition implies that
  \begin{align*}
    V^\star - V^{\hat \pi^{(k)}} \leq
    \hat V^{(k)} - \hat V^{\hat \pi^{(k)}} + \frac{\epsilon}{4} \leq
    \frac{3}{4}\epsilon \leq \epsilon
  \end{align*}
  and hence the returned policy is $\epsilon$-optimal.

  On the other hand, if the algorithm does not terminate in iteration
  $k$, we have that
  $\hat V^{(k)} - \hat V^{\hat \pi^{(k)}} > \frac{\epsilon}{2}$ and
  therefore
  \begin{align*}
    V^\star - V^{\hat \pi^{(k)}} \geq
    \hat V^{(k)} - \hat V^{\hat \pi^{(k)}} - \frac{\epsilon}{4} \geq
    \frac{\epsilon}{4}.
  \end{align*}
  We now use this fact with
  Proposition~\ref{prop:localvalue_polguarantee} to argue that the
  policy $\hat{\pi}^{(k)}$ must visit an unlearned state with sufficient
  probability. Under the conditions here, applying
  Proposition~\ref{prop:localvalue_polguarantee}, we get that
  \begin{align*}
      \frac{\epsilon}{4} \leq V^\star - V^{\hat \pi^{(k)}} \leq p_{ul}^{\hat{\pi}^{(k)}} + 2 \Tmax H^2 ( 7\epsstat + 3\epssub + 2\epsfeas).
  \end{align*}
  With the choice of $\epsstat$, rearranging this inequality reveals that
  $p_{ul}^{\hat{\pi}^{(k)}} \ge \epsilon/8 > 0$.  Hence, if the
  algorithm does not terminate there must be at least one unlearned
  state, i.e., $\Scal \setminus \Slearned \neq \emptyset$.

  For the last step of the proof, we argue that since
  $p_{ul}^{\hat{\pi}^{(k)}}$ is large, the probability of reaching an
  unlearned state is high, and therefore the additional calls to
  \dfslearn in Line~\ref{lin:newlearn} with high probability will
  visit a new state, which we will then learn. Specifically, we will
  prove that on every non-terminal iteration of \metaalg, we learn at
  least one previously unlearned state. With this fact, since there
  are at most $MH$ states, the algorithm must terminate and return a
  near-optimal policy after at most $MH$ iterations.

  In a non-terminal iteration $k$, the probability that we do not hit an
  unlearned state in Line~\ref{lin:newlearn} is
  \begin{align*}
    (1 - p_{ul}^{\hat{\pi}^{(k)}})^{\nexp} \leq(1 - \epsilon/8)^{\nexp} \leq \exp(-\epsilon\nexp/8).
  \end{align*}
  This follows from independence of the $\nexp$ trajectories sampled
  from $\hat{\pi}^{(k)}$.
  $\nexp \ge\frac{8}{\epsilon}\ln\left(\frac{4MH}{\delta}\right)$
  ensures that the probability of not hitting unlearned states in any of the $MH$ iterations is at most $\delta/4$.

  In total, except with probability $\delta/2 + \delta/4 + \delta/4$
  (for the three events we considered above), on every iteration, either
  the algorithm finds a near optimal policy and returns it, or it
  visits a previously unlearned state, which subsequently becomes
  learned. Since there are at most $MH$ states, this
  proves that with probability at least $1-\delta$, the algorithm
  returns a policy that is at most $\epsilon$-suboptimal.
\end{proof}

\subsection{Proof of Sample Complexity: Theorem~\ref{thm:local_value}}

We now have all parts to complete the proof of
Theorem~\ref{thm:local_value}. For the calculation, we instantiate all the parameters as
\begin{align*}
    \epsstat &= \epssub = \epsfeas = \frac{\epsilon}{2^6 7 H^2 \Tmax},\\
    \phi_h &= (H-h+1)(6\epsstat + 2\epssub + \epsfeas), \quad \Tmax = MH\nexp + M,\\
  \ntest &= \frac{\log\left(12KH\Tmax|\Gcal|/\delta\right)}{2\epsstat^2}, \quad \ntrain= \frac{16K\log(12 H\Tmax |\Gcal||\Pi|/\delta)}{\epsstat^2},\\
  \nexp &= \frac{8\log(4MH/\delta)}{\epsilon}, \quad \neval = \frac{32\log(8MH/\delta)}{\epsilon^2}.
\end{align*}
These settings suffice to apply all of the above lemmas and therefore
with these settings the algorithm outputs a policy that is at most
$\epsilon$-suboptimal, except with probability $\delta$. For the
sample complexity, since $\Tmax$ is an upper bound on the number of
data sets we collect (because $\Tmax$ is an upper bound on the number of execution of
\dfslearn at any level), and we also $\neval$ trajectories for each of the $MH$
iterations of \metaalg, the total sample complexity is
\begin{align*}
& H\Tmax \ntrain + KH\Tmax \ntest + MH\neval\\
& = \order\left(\frac{\Tmax^3KH^5}{\epsilon^2}\log\left(\frac{MKH}{\epsilon\delta}|\Gcal||\Pi|\log(MH/\delta)\right) + \frac{MH}{\epsilon^2}\log(MH/\delta)\right)\\
& = \order\left( \frac{M^3KH^8}{\epsilon^5}\log^3(MH/\delta)\log\left(\frac{MKH}{\epsilon\delta}|\Gcal||\Pi|\log(MH/\delta)\right)\right).
\end{align*}
This proves the theorem. \qed

\subsection{Extension: \mainAlg with Constrained Policy Optimization} \label{sec:lagrange}
We note that Theorem~\ref{thm:local_value}  suffers relatively high sample complexity compared to the original \lsvee. The issue is that \mainAlg pools all the data sets together for policy optimization (Algorithm~\ref{alg:local_value_polval_unconstr}). This implicitly weights all data sets uniformly, and allows some undesired trade-off: the policy that maximizes the objective could sacrifice significant amount of value on one data set (for some hidden state) to gain slightly more value on many others, only to find out later that the sacrificed state is visited very often during execution. This is the well-known distribution mismatch issue of reinforcement learning. 

To address this issue and attain better sample complexity results, Algorithm~\ref{alg:local_value_polval_constr} shows an alternative to
the policy optimization component of \mainAlg in
Algorithm~\ref{alg:local_value_polval_unconstr}.  Instead of using an
unconstrained optimization problem, it finds the policy through a
feasibility problem, and hence avoid the undesired trade-off mentioned above. 
The computation can be implemented by the multi data set
classification oracle defined in Section~\ref{sec:moredefinitions}. 
\begin{algorithm}[ht]
\SetKwProg{myfun}{Function}{}{}
    \myfun{\polvalfun{}}{
        $\hat V^{\star} \gets V$ associated with only dataset in $\Dcal_1$\;
        \For{$h=1:H$}{
            Pick $\hat \pi_h$ such that the following constraints are violated at most $\epsfeas$ for all $(D, V, \{V_{a}\}_a) \in \Dcal_h~:~$
            $\EEx_{D}[K \one\{\pi(x_h) = a_h\} (r_h + V_{a_h})] \geq V - 2 \phi_h + 4 \epsstat + \epssub$
            \label{lin:local_value_fit_policy_alt}\;
        }
        \Return $\hat \pi_{1:H}, \hat V^{\star}$\;
    }
    \caption{Constrained policy optimization with local values}
    \label{alg:local_value_polval_constr}
\end{algorithm}

Below, we prove a stronger version of
Proposition~\ref{prop:localvalue_polguarantee} (which is for Algorithm~\ref{alg:local_value_polval_unconstr}) for this approach based on
feasibility. First, we show that $\pi^\star$ is always a feasible
choice in Line~\ref{lin:local_value_fit_policy_alt} in event $\Ecal$.
\begin{lemma}
    \label{lem:gstarfeas}
    Assume $g^\star \in \Gcal_h$, $\pi^\star \in \Pi_h$ and $\phi_h = (H+1 - h)(6 \epsstat + 2 \epssub + \epsfeas)$ for all
    $h \in [H]$. Then $\pi^\star$ is a valid choice in
    Line~\ref{lin:local_value_fit_policy_alt} of \polvalfun in
    Algorithm~\ref{alg:local_value_polval_constr} in event $\Ecal$.
\end{lemma}
\begin{proof}
    Consider a single data set $(D, V, \{V_{a}\}_a) \in \Dcal_h$ that is associated with state $s \in \Scal_h$.
    Using Proposition~\ref{prop:Vaccuracy}, we can bound the deviation of the optimal policy for each constraint as
    \begin{align*}
        & V - \EEx_{D}[K \one\{\pi^\star(x_h) = a_h\} (r_h + V_{a_h})] \\
& \leq 
V^\star(s) + \phi_h - 2\epsstat - \EEx_{D}[K \one\{\pi^\star(x_h) = a_h\} (r_h + V_{a_h})] \\
& \leq 
V^\star(s) + \phi_h - 2\epsstat - \Ex_{s}[K \one\{\pi^\star(x_h) = a_h\} (r_h + V_{a_h})] + \epsstat\\ 
& \leq 
V^\star(s) + \phi_h + 2\epsstat - \Ex_{s}[K \one\{\pi^\star(x_h) = a_h\} (r_h + V^\star(s \circ a_h))] + \phi_{h+1} + \epssub + \epsfeas\\ 
& = 
 \phi_h + 2\epsstat + \phi_{h+1} + \epssub +\epsfeas
 =
2 \phi_h - 4 \epsstat - \epssub.
    \end{align*}
    Here we first used that $V$ is close to the optimal value $V^\star(s)$, the
    deviation bounds next and finally leveraged that $V_a$ is a good estimate.
    Since that inequality holds for all constraints, $\pi^\star$ is feasible.
\end{proof}
We now show that Algorithm~\ref{alg:local_value_polval_constr}
produces policies with a better guarantees than its unconstrained
counterpart. The difference is that we eliminate the $\Tmax$ term in
the error bound.
\begin{proposition}[Improvement over Proposition~\ref{prop:localvalue_polguarantee}]
    \label{prop:localvalue_polguarantee_constr}
    Assume $g^\star \in \Gcal_h$ and that we are in event
    $\Ecal$. Recall the definition $\phi_h = (H+1 - h)(6 \epsstat + 2 \epssub + \epsfeas)$ for all
    $h \in [H]$.  Then the policy $\hat{\pi} = \hat \pi_{1:H}$
    returned by \polvalfun in Algorithm~\ref{alg:local_value_polval_constr} satisfies
    \begin{align*}
        V^{\hat{\pi}} \ge V^\star - p_{ul}^{\hat \pi} - 32H^2(\epsstat+\epsfeas + \epssub)
    \end{align*}
    where
    $p_{ul}^{\hat \pi}= \prob( \exists h \in [H] ~:~ s_h \notin
    \Slearned ~|~ a_{1:H} \sim \hat\pi)$ is the probability of hitting
    an unlearned state when following $\hat \pi$.
\end{proposition}
\begin{proof}
  We bound the difference of how much following $\hat \pi_h$ for one
  time step can hurt per state using
  Proposition~\ref{prop:Vaccuracy}. First note that by
  Lemma~\ref{lem:gstarfeas}, the optimization problem always has a
  feasible solution in event $\Ecal$, so $\hat \pi_h$ is well
  defined. For a state $s \in \Slearned_h$, we have
  \begin{align*}
    & V^\star(s) - Q^\star(s, \hat \pi_h)\\
    & = \Ex_s[K(\one\{\pi^\star_h(x_h) = a_h\} - \one\{\hat \pi_h(x_h) = a_h\}  ) (r_h + g^\star_{h+1}(x_{h+1})]\\
    & \leq \Ex_s[K(\one\{\pi^\star_h(x_h) = a_h\} - \one\{\hat \pi_h(x_h) = a_h\}  ) (r_h + V_{a_h})] + 2 \phi_{h+1} + 6 \epsstat + 2 \epssub + 2 \epsfeas\\
    & \leq \EEx_{D}[K(\one\{\pi^\star_h(x_h) = a_h\} - \one\{\hat \pi_h(x_h) = a_h\}  ) (r_h + V_{a_h})] + 2 \phi_{h+1} + 8 \epsstat + 2 \epssub + 2 \epsfeas\\
    & \leq V + \epssub - V + 2\phi_{h} -4 \epsstat - \epssub + \epsfeas + 2 \phi_{h+1} + 8 \epsstat + 2 \epssub + 2 \epsfeas\\
      &=  4 \phi_{h+1} + 16 \epsstat + 5 \epsfeas + 6 \epssub
    =  4 \phi_{h} - 8 \epsstat - 2 \epssub + \epsfeas.
\end{align*}
Here $(D, V, \{V_a\})$ is one of the data sets in $\Dcal_h$ that is
associated with $s$, which has optimal policy value $V$ by
construction. We first applied definitions and then used that $V_a$
are good value estimates. Subsequently we applied the deviation bounds
and finally leveraged the definition of $V$ and the approximate feasibility of
$\hat{\pi}_h$.  Using Lemma~\ref{lem:subopteq}, the suboptimality of
$\hat \pi$ is therefore at most
\begin{align*}
    V^\star - V^{\hat{\pi}} & \leq p_{ul}^{\hat \pi} + (1 -p_{ul}^{\hat \pi}) \sum_{h=1}^H
    (4 \phi_{h+1} + 16 \epsstat + 5 \epsfeas + 6 \epssub)\\
    &\leq p_{ul}^{\hat \pi} + 16H\epsstat + 6H\epssub + 5H \epsfeas + 4 \sum_{h=1}^H \phi_{h+1}\\
    & \leq  p_{ul}^{\hat \pi} +16H\epsstat + 6H\epssub + 5H \epsfeas 
    + 4(6\epsstat + 2 \epssub + \epsfeas) \sum_{h=1}^H (H-h)\\
    & \leq  p_{ul}^{\hat \pi} +16H\epsstat + 6H\epssub + 5H \epsfeas 
    + 2H^2(6\epsstat + 2 \epssub + \epsfeas) \\
    & \leq p_{ul}^{\hat \pi} + 32H^2(\epsstat+\epsfeas + \epssub). \qedhere
\end{align*}
\end{proof}

Using this improved policy guarantee, we obtain a tighter analysis of
\metaalg that does not have a dependency on $\Tmax$ in $\epsstat$.
\begin{lemma}
    \label{lem:metaalg_constr}
    Consider running \metaalg with \dfslearn and \polvalfun (Algorithm~\ref{alg:metaalg} + \ref{alg:local_value_polval_constr} + \ref{alg:local_value_dfslearn}) with parameters
    \begin{align*}
        \nexp \ge \frac{8}{\epsilon}\ln\left(\frac{4MH}{\delta}\right), \quad n_{eval} \geq \frac{32}{\epsilon^2}\ln\left( \frac{8MH}{\delta}\right), \quad \epsstat = \epsfeas = \epssub = \frac{\epsilon}{2^{10} H^2}.
    \end{align*}
    Then with probability at least $1-\delta$,
    \metaalg returns a policy that is at least $\epsilon$-optimal
    after at most $MK$ iterations.
\end{lemma}
\begin{proof}
  The proof is identical to the proof of Lemma~\ref{lem:metaalg}
  except using Proposition~\ref{prop:localvalue_polguarantee_constr}
  in place of Proposition~\ref{prop:localvalue_polguarantee}, and
  using Lemma~\ref{lem:gstarfeas} to guarantee that the optimization
  problem in Line~\ref{lin:local_value_fit_policy_alt} is always
  feasible, in event $\Ecal$.
\end{proof}
Finally, we are ready to assemble all statements to the following sample-complexity bound:
\begin{theorem}\label{thm:local_value_constr}
Consider a Markovian CDP with deterministic dynamics over $M$ hidden states, as described in Section~\ref{sec:prelim}. When $\pi^\star \in \Pi$ and $g^\star \in \Gcal$ (Assumptions~\ref{asm:pol_realize} and \ref{asm:val_realize} hold), for any $\epsilon, \delta \in (0, 1)$, the local value algorithm with constrained policy optimization (Algorithm~\ref{alg:metaalg} + \ref{alg:local_value_polval_constr} + \ref{alg:local_value_dfslearn}) returns a policy $\pi$ such that $V^\star - V^{\pi} \le \epsilon$ with probability at least $1-\delta$, after collecting at most 
$\otil\left(\frac{MKH^6}{\epsilon^3}\log(|\Gcal||\Pi|/\delta)\log(1/\delta)\right)$
trajectories.
\end{theorem}
\begin{proof}
We now have all parts to complete the proof of
Theorem~\ref{thm:local_value}. For the calculation, we instantiate all the parameters as
\begin{align*}
    & \epsstat = \epsfeas = \epssub= \frac{\epsilon}{2^{10} H^2}, 
    \quad \phi_h = (H-h+1)(6\epsstat + 2 \epssub + \epsfeas),\\
   &  \Tmax =MH\nexp + M,\\
 & \ntest = \frac{\log\left(12KH\Tmax|\Gcal|/\delta\right)}{2\epsstat^2}, \quad \ntrain= \frac{16K\log(12 H\Tmax |\Gcal||\Pi|/\delta)}{\epsstat^2},\\
 & \nexp = \frac{8\log(4MH/\delta)}{\epsilon}, \quad \neval = \frac{32\log(8MH/\delta)}{\epsilon^2}.
\end{align*}
These settings suffice to apply all of the above lemmas for these algorithms and therefore
with these settings the algorithm outputs a policy that is at most
$\epsilon$-suboptimal, except with probability $\delta$. For the
sample complexity, since $\Tmax$ is an upper bound on the number of
data sets we collect (because $\Tmax$ is an upper bound on the number of execution of
\dfslearn at any level), and we also $\neval$ trajectories for each of the $MH$
iterations of \metaalg, the total sample complexity is
\begin{align*}
& H\Tmax \ntrain + KH\Tmax \ntest + MH\neval\\
& = \order\left(\frac{\Tmax KH^5}{\epsilon^2}\log\left(\frac{MKH}{\epsilon\delta}|\Gcal||\Pi|\log(MH/\delta)\right) + \frac{MH}{\epsilon^2}\log(MH/\delta)\right)\\
& = \order\left( \frac{M KH^6}{\epsilon^3}\log(MH/\delta)\log\left(\frac{MKH}{\epsilon\delta}|\Gcal||\Pi|\log(MH/\delta)\right)\right). \tag*{\qedhere}
\end{align*}
\end{proof}

\section{Alternative Algorithms}
\label{sec:alternativealgs}

\begin{theorem}[Informal statement]
  Under Assumption~\ref{asm:polval_complete} or
  Assumptions~\ref{asm:val_realize}+\ref{asm:pol_complete}, there
  exist oracle-efficient algorithms with polynomial sample complexity
    in CDPs (contextual decision processes) with deterministic dynamics over small hidden states. These algorithms do not store or
  use local values.
\end{theorem}

\subsection{Algorithm with Two-Sample State-Identity Test}

\label{sec:mmd}
See Algorithm~\ref{alg:metaalg} + \ref{alg:twosam}. The algorithm uses
a novel state identity test which compares two distributions using a
two-sample test~\citep{gretton2012kernel} in Line~\ref{lin:mmd} (recall that $\Gcal_h = \Gcal$ for $h \in [H]$ and $\Gcal_{H+1} = \{ x \mapsto 0 \}$). Such
an identity test mechanism is very different from the one used in the \mainAlg algorithm, and the two mechanisms have very different
behavior. For example, if $\Gcal= \{g^\star\}$, the local value
algorithm will claim every state $s$ as ``not new'' because it knows
the optimal value $V^\star(s)$, whereas the two-sample test may still
declare a state $s$ to be new if
$\Ex_{s}[g^\star(x)] \ne \Ex_{s'}[g^\star(x)]$ for any previously
visited $s'$. On the other hand, the two-sample test algorithm may not
have learned $V^\star(s)$ at all when it claims that a state $s$ is
not new. Given the novelty of the mechanism, we believe analyzing the
two-sample test algorithm and understanding its computational and
statistical properties enriches our toolkit for dealing with the
challenges addressed in this paper.

\newcommand{\Dpol}{\Dcal^{\textrm{learned}}}
\newcommand{\Dval}{\Dcal^{\textrm{val}}}
\newcommand{\Spol}{\Scal^{\textrm{learned}}}
\newcommand{\Sval}{\Scal^{\textrm{val}}}
\newcommand{\Shit}{\Scal^{\textrm{check}}}

\begin{algorithm}

\SetKwProg{myfun}{Function}{}{}
    \myfun{\polvalfun{}}{
        $\hat g_{H+1} \gets 0$ \;
        \For{$h=H:1$}{
            $\hat \pi_h \gets \argmax_{\pi \in \Pi_h} \sum_{D \in \Dpol_h} \EEx_{D}[K \one\{\pi(x_h) = a_h\} (r_h + \hat g_{h+1}(x_{h+1}))]$ \label{lin:mmd_pol_opt}\;
            $\hat g_h \gets \argmin_{g \in \Gcal_h} \sum_{D \in \Dval_h} \EEx_{D}[K \one\{\hat \pi_h(x_h) = a_h\} (g(x_h) - r_h - \hat g_{h+1}(x_{h+1}))^2]$ \label{lin:mmd_val_reg}\;
        }
        $\hat V^{\star} \gets \EEx_{D}[\hat g_1(x_1)]$ where $D$ is the only distribution in $\Dval_1$\;
        \Return $\hat \pi_{1:H}, \hat V^{\star}$\;
    }

    \myfun{\dfslearn{$a_{1:h-1}$}}{
        $\tilde D \gets$ sample $x_h \sim a_{1:h-1}$, $a_h \sim \textrm{Unif}(\Acal)$, $r_h$, $x_{h+1}$ \label{lin:sample_learn_mmd}\;
        $d_{\textrm{MMD}} \gets \min_{D \in \Dval_h} \sup_{g \in \Gcal_h} \left| \EEx_D[g(x_h)] - \EEx_{\tilde D}[g(x_h)]\right|$ \label{lin:mmd}\;
        \If{$d_{\textrm{MMD}} \le 2 \tau$ \textbf{and} \textsc{is\_recursive\_call}}{\Return \label{lin:check_return}}
        \If{$d_{\textrm{MMD}} > 2\tau$}{Add $\tilde D$ to $\Dval_h$}
        Add $\tilde D$ to $\Dpol_h$\;
        \For{$a \in \Acal$}{
            $\dfslearn(a_{1:h-1} \circ a)$ \label{lin:mmd_descend}\; 
        }
    }
    \caption{Algorithm with Two-Sample State-Identity Test}
    \label{alg:twosam}
\end{algorithm}

\subsubsection{Computational considerations}
The two-sample test algorithm requires three types nontrivial
computation. Line~\ref{lin:mmd_pol_opt} requires importance weighted
policy optimization, which is simply a call to the CSC oracles. Line~\ref{lin:mmd_val_reg} performs squared-loss
regression on $\Gcal_h$, which is a call to a LS oracle.

The slightly unusual computation occurs on Line~\ref{lin:mmd}: we
compute the (empirical) Maximum Mean Discrepancy (MMD) between $D$ and
$\tilde D$ against the function class $\Gcal_h$, and take the minimum
over $D \in \Dval$. First, since $|\Dval_h|$ remains small over the
execution of the algorithm, the minimization over $D \in \Dval_h$ can
be done by enumeration. Then, for a fixed $D$, computing the MMD is a
linear optimization problem over $\Gcal_h$. In the special case where
$\Gcal_h$ is the unit ball in a Reproducing Kernel Hilbert Space
(RKHS)~\citep{scholkopf2002learning}, MMD can be computed in closed
form by $O(n^2)$ kernel evaluations, where $n$ is the number of data
points involved~\citep{gretton2012kernel}.

To unclutter the sample-complexity analysis, we assume that perfect oracles, i.e., $\epsfeas = \epssub = 0$.
\subsubsection{Sample complexity}
\begin{theorem}
    \label{thm:mmd_samplecomplex}
Consider the same Markovian CDP setting as in Theorem~\ref{thm:local_value} but we explicitly require here that the process is an MDP over $\Xcal$. Under Assumption~\ref{asm:polval_complete}, for any $\epsilon, \delta \in (0, 1)$, the two-sample state-identity test algorithm (Algorithm~\ref{alg:metaalg}+\ref{alg:twosam}) returns a policy $\pi$ such that $V^\star - V^{\pi} \le \epsilon$ with probability at least $1-\delta$, after collecting at most
$\otil\left(\frac{M^2K^2H^6}{\epsilon^4}\log(|\Gcal||\Pi|/\delta)\log^2(1/\delta)\right)$
trajectories.
\end{theorem}

\newcommand{\gh}{\hat g_h}
\newcommand{\pih}{\hat \pi_h}
\newcommand{\ghn}{\hat g_{h+1}}
\newcommand\numberthis{\addtocounter{equation}{1}\tag{\theequation}}

For this algorithm, we use the following notion of learned state:
\begin{definition}[Learned states]
    Denote the sequence of states whose data sets are added to $\Dpol_h$ as $\Spol_h$. States that are in $\Spol = \bigcup_{h \in [H]} \Spol_h$ are called learned. The sequence of states whose data sets are added to $\Dval_h$ are denoted by $\Sval_h$. Let $\Shit_h$ denote the set of all states that have been reached by any previous \dfslearn call at level $h$.
\end{definition}

\begin{fact}
We have $\Sval_h \subseteq \Spol_h \subseteq \Shit_h$. Furthermore, $\forall s \in \Spol_h$ and $a \in \Acal$, $s\circ a \in \Shit_{h+1}$.
\end{fact}

Define the following short-hand notations for the objective functions used in Algorithm~\ref{alg:twosam}:
\begin{align*}
    V_D(\pi; g') :=&  \EEx_{D}[ K\one\{\pi(x) = a \}(r + g'(x'))]. \\
    V_{\Dpol_h}(\pi; g') :=& \sum_{D \in \Dpol_{h}} V_D(\pi; g'). \\
    L_D(g; \pi, g') :=& \EEx_{D} [K\one\{\pi(x) = a \} (g(x)-r - g'(x'))^2]. \\
    L_{\Dval_h}(g; \pi, g') :=& \sum_{D \in \Dval_{h}} L_D(g; \pi, g').
\end{align*}
Also define $V_s$, $V_{\Spol_h}$, $L_s$, $L_{\Sval_h}$ as the population version of $V_D$, $V_{\Dpol_h}$, $L_D$, $L_{\Dval_h}$, respectively.

\paragraph{Concentration Results.}
For our analysis we rely on the following concentration bounds that
define the good event $\Ecal$. This definition involves
parameters $\tau, \tau_L, \tau_V$ whose values we will set later.
\begin{definition}
    \label{def:devbounds_mmd}
    Let $\Ecal$ denote the event that for all $h \in [H]$ the total number of calls to $\dfslearn(p)$ at level $h$ is at most $\Tmax = M(K+1)(1 + H\nexp)$ during the execution of \metaalg and that for all these calls to $\dfslearn(p)$ the following deviation bounds hold for all
    $g \in \Gcal_h, g' \in \Gcal_{h+1}$ and $\pi \in \Pi_h$ (where $\tilde D$ is the data set of
    $\ntrain$ samples from Line~\ref{lin:sample_learn_mmd} and $s$ is the state reached by $p$):
    \begin{align}
        |\EEx_{\tilde D}[g(x)] - \Ex_{s}[g(x)] | \leq & \tau \label{eq:devb1mmd} \\
        |V_{\tilde D}(\pi; g') - V_{s}(\pi; g')| \le & \tau_V\\
        |L_{\tilde D}(g; \pi, g') - L_{s}(g; \pi, g') | \le & \tau_L. \label{eq:devb2mmd}
    \end{align}
\end{definition}
We now show that this event has high probability.

\begin{lemma}
    \label{lem:probEmmd}
    Set $\ntrain$ so that
\begin{align*}
        \ntrain \geq \max\Big\{ & \frac{1}{2\tau^2} \ln \left(\frac{ 12 H\Tmax |\Gcal|}{\delta} \right), \\
      &  \frac{16K}{\tau_V^2} \ln \left(\frac{ 12 H\Tmax |\Gcal||\Pi|}{\delta} \right), \\
      &  \frac{32K}{\tau_L^2} \ln \left(\frac{ 12 H\Tmax |\Gcal|^2|\Pi|}{\delta} \right) \Big\}.
\end{align*}
Then $\prob[\Ecal] \ge 1 - \delta/2$ where $\Ecal$ is defined in
    Definition~\ref{def:devbounds_mmd}. In addition, in event $\Ecal$, during
    all calls the sequences are bounded as $|\Sval_h| \leq M$ and $|\Spol_h|
    \leq \Tmax$.
\end{lemma}
\begin{proof}
Let us first focus on one call to \dfslearn, say at path $p$ at level $h$.
First, observe that the data set $\tilde D$ is
a set of $\ntrain$ transitions sampled i.i.d. from the state $s$ that is reached by $p$. By Hoeffding's
inequality and a union bound, with probability $1-\delta'$, for all $g
\in \Gcal_{h}$
\begin{align*}
\left| \EEx_{\tilde D}[g(x_h)] - \Ex_{s}[g(x_{h})]\right|
\leq \sqrt{\frac{\log(2|\Gcal|/\delta')}{2\ntrain}}.
\end{align*}
With $\delta' = \frac{\delta}{6HT_{\max}}$ the choice for $\ntrain$ let us bound the LHS by $\tau$.

For the random variable $K\one\{\pi(x_h) = a_h\} (r_h+ g'(x_{h+1}))$,
since $a_h$ is chosen uniformly at random, it is not hard to see that
both the variance and the range are at most $2K$ (see for example
Lemma~14 by \citet{jiang2017contextual}).  Applying Bernstein's
inequality and a union bound, for all $\pi \in \Pi_h$ and
$g \in \Gcal_{h+1}$, we have
\begin{align*}
|V_{\tilde D}(\pi; g') - V_{s}(\pi; g')| \leq \sqrt{\frac{4 K\log(2 |\Gcal||\Pi|/\delta')}{\ntrain}} + \frac{4K}{3\ntrain}\log(2 |\Gcal||\Pi|/\delta')
\end{align*}
with probability $1-\delta'$. As above, with $\delta' =
\frac{\delta}{6HT_{\max}}$ our choice of
$\ntrain$ ensures that this deviation is bound by $\tau_V$.

Similarly, we apply Bernstein's inequality to the random variable
$K\one\{\pi(x_h) = a_h \} (g(x_h)-r_h - g'(x_{h+1}))^2$ which has
range and variance at most $4K$.  Combined with a union bound over all
$g \in \Gcal_h, g' \in \Gcal_{h+1}, \pi \in \Pi_h$ we have that with
probability $1-\delta'$,
\begin{align*}
|L_{\tilde D}(g; \pi, g') - L_{s}(g; \pi, g')| \leq \sqrt{\frac{8 K\log(2 |\Gcal|^2|\Pi|/\delta')}{\ntrain}} + \frac{2K}{\ntrain}\log(2 |\Gcal|^2|\Pi|/\delta') \leq \tau_L.
\end{align*}
This last inequality is based on the choice for $\ntrain$ and
$\delta' = \frac{\delta}{6H\Tmax}$. For details on this concentration
bound see for example Lemma~14 by \citet{jiang2017contextual}.  Using
a union bound, the deviation bounds
\eqref{eq:devb1mmd}--\eqref{eq:devb2mmd} hold for a single call to
\dfslearn with probability $1 - 3 \delta'$.

Consider now the event $\Ecal'$ that these bounds hold for the first
$\Tmax$ calls at each level $h$. Applying a union bound let us bound
$\prob(\Ecal') \geq 1 - 3 H \Tmax \delta' = 1 - \frac{\delta}{2}$.  It
remains to show that $\Ecal' \subseteq \Ecal$.

First note that in event $\Ecal'$ in the first $\Tmax$ calls to
\dfslearn at level $h$, the algorithm does not call itself recursively during
a recursive call 
if $p$ leads to a state $s \in \Sval_h$. To see this assume $p$ leads
to a state $s \in \Sval_h$ and let $D \in \Dval_h$ be a data set
sampled from this state. This means that $\tilde D$ and $D$ are
sampled from the same distribution, and as such, we have for every
$g \in \Gcal_h$
    \begin{align}
       |\EEx_{\tilde D}[g(x)] - \EEx_{D}[g(x)]|
       \leq
        |\EEx_{\tilde D}[g(x)] - \Ex_{s}[g(x)]| + |\Ex_{s}[g(x)]- \EEx_{D}[g(x)]|\leq 2 \tau.\label{eq:above1222}
   \end{align}
   Therefore $d_{MMD} \leq 2\tau$, the condition in the first clause is
   satisfied, and the algorithm does not recurse. If this condition is
   not satisfied, the algorithm adds $\tilde D$ to
   $\Dval_h$. Therefore, the initial call to \dfslearn at the root can
   result in at most $MK$ recursive calls per level, since the
   identity tests must return true on identical states.

   Further, for any fixed level, we issue at most $MH\nexp$ additional
   calls to \dfslearn, since \metaalg has at most $MH$ iterations and
   in each one, \dfslearn is called $\nexp$ times per level. Any new
   state that we visit in this process was already counted by the $MK$
   calls per level in the initial execution of \dfslearn. On the other
   hand, these calls always descend to the children, so the number of calls to old states is at most $M(1+K)H\nexp$ per level. In total the number
   of calls to \dfslearn per level is at most
   $M(1+K)H\nexp + MK \leq \Tmax$, and $\prob(\Ecal) \leq \delta/2$
   follows.

   Further, the bound $|\Spol_h| \leq \Tmax$ follows from the fact
   that per call only one state can be added to $\Spol_h$ and there
   are at most $\Tmax$ calls.  The bound $|\Sval_h| \leq M$ follows
   from the fact that in $\Ecal$ no state can be added twice to
   $\Sval_h$ since as soon as it is in $\Sval_h$ once,
   $d_{MMD} \leq 2 \tau$ holds (see Eq.\eqref{eq:above1222}) and the
   current data set is not added to $\Dval_h$.
\end{proof}

\paragraph{Depth-first search and learning optimal values.}
We now prove that \polvalfun and \dfslearn produce good value function estimates.
\begin{proposition}
    \label{prop:mmdlearnedvalue}
    In event $\Ecal$, consider an execution of \polvalfun and let
    $\{\hat{g}_h,\hat{\pi}_h\}_{h \in [H]}$ denote the learned value
    functions and policies. Then every state $s$ in $\Shit_h$
    satisfies
\begin{align} \label{eq:induct_g}
\left|\Ex_{s} [\gh(x_h)] - \Ex_{s}[g^\star(x_h)] \right|
\le (H+1-h)(2M\tau_V + \sqrt{4M^2\tau_V + 2\Tmax \tau_L} + 8\tau),
\end{align}
and every learned state $s \in \Spol_h$ satisfies
\begin{align} \label{eq:immediate_loss}
V^\star(s) - Q^\star(s, \pih) 
\le 2 M \tau_V + 2 (H-h)(2M\tau_V + \sqrt{4M^2\tau_V + 2\Tmax\tau_L} + 8\tau).
\end{align}
\end{proposition}
\begin{proof}
  We prove both inequalities simultaneously by induction over $h$. For
  convenience, we use the following short hand notations:
  $\epsilon_V = M \tau_V$ and $\epsilon_L = \Tmax \tau_L$.  Using this
  notation, in event $\Ecal$,
  $|V_{\Dval_h}(\pi; g') - V_{\Sval_h}(\pi; g')| \le \epsilon_V$ and
  $|L_{\Dpol_h}(g; \pi, g') - L_{\Spol_h}(g; \pi, g') | \le
  \epsilon_L$ hold for all $g, g'$ and $\pi$.

  \paragraph{Base case:} Both statement holds trivially for $h=H+1$
  since the LHS is $0$ and the RHS is non-negative. In particular
  there are no actions, so Eq.~\eqref{eq:immediate_loss} is trivial.

\paragraph{Inductive case:} Assume that Eq.~\eqref{eq:induct_g} holds on level $h+1$. For any learned $s \in \Spol_h$, we first show that $\hat \pi_h$ achieves high value compared to $\pi_{\ghn}^\star$ (recall its definition from Assumption~\ref{asm:polval_complete}) under $V_s(\cdot; \ghn)$:
\begin{align*}
V_s(\pi_{\ghn}^\star; \ghn) - V_s(\pih; \ghn)
\le &~ \sum_{s \in \Spol_h} V_s(\pi_{\ghn}^\star; \ghn) - V_s(\pih; \ghn) \tag{$\pi_{\ghn}^\star$ is optimal w.r.t.~$\ghn$ in all $s$}  \\
= &~ V_{\Spol_h}(\pi_{\ghn}^\star; \ghn) - V_{\Spol_h}(\pih; \ghn) \\
\le &~ V_{\Dpol_h}(\pi_{\ghn}^\star; \ghn) - V_{\Dpol_h}(\pih; \ghn) + 2\epsilon_V \leq 2\epsilon_V. \numberthis \label{eq:mmd_relative_loss}
\end{align*}
Eq.~\eqref{eq:immediate_loss} follows as a corollary:
\begin{align*}
&~ V^\star(s) - Q^\star(s, \pih) \\
= &~ V_s(\pi^\star; g^\star)  - V_s(\pih; g^\star)  \tag{definition of $V_s$}\\
\le &~ V_s(\pi^\star; \ghn) - V_s(\pih; \ghn) 
    + 2 \sup_{s' \textrm{ being child of } s}|\Ex_{x' \sim s'}[\ghn(x_{h+1}) - g^\star(x_{h+1})]|  \\
\le &~ V_s(\pi_{\ghn}^\star; \ghn) -  V_s(\pih; \ghn) 
    +  2 (H-h)(2\epsilon_V + \sqrt{4 M\epsilon_V + 2\epsilon_L} + 8\tau)
    \tag{$s \in \Sval_h \Rightarrow s' \in \Shit_{h+1}$ and induction}\\
\le &~ 2 \epsilon_V + 2 (H-h)(2 \epsilon_V + \sqrt{4M \epsilon_V + 2\epsilon_L} + 8\tau). \tag{using Eq.\eqref{eq:mmd_relative_loss}}
\end{align*}
This proves Eq.~\eqref{eq:immediate_loss} at level $h$. The rest of the proof proves Eq.\eqref{eq:induct_g}. First we introduce and recall the definitions:
\begin{align*}
&~  g_{\pih,\ghn}(x) = \Ex[r + \ghn(x_{h+1}) \mid x_h = x, a_h = \pih(x)],  \\
&~ g_{\star,\ghn}(x) = \Ex[r + \ghn(x_{h+1}) \mid x_h = x, a_h = \pi^\star_{\ghn}(x)].
\end{align*}
Note that $g_{\pih,\ghn} \notin \Gcal_h$ in general, but it is the
Bayes optimal predictor for the squared losses
$L_s(\cdot; \pih, \ghn)$ for all $s$ simultaneously.  On the other
hand, Assumption~\ref{asm:polval_complete} guarantees that
$g_{\star,\ghn} \in \Gcal_h$, for any $\ghn$.

The LHS of Eq.\eqref{eq:induct_g} can be bounded as
\begin{align} \label{eq:value_learned_decompose}
\left|\Ex_{s} [g^\star(x_h)] - \Ex_{s}[\gh(x_h)] \right|
\le \left|\Ex_{s} [g^\star(x_h)] - \Ex_{s}[g_{\pih, \ghn}(x_h)] \right| 
+ \left|\Ex_{s}[g_{\pih, \ghn}(x_h)] - \Ex_{s}[\gh(x_h)] \right|.
\end{align}
To bound the first term in Eq.\eqref{eq:value_learned_decompose},
\begin{align*}
\left|\Ex_{s} [g^\star(x_h)] - \Ex_{s}[g_{\pih, \ghn}(x_h)] \right|
\le &~ \left|\Ex_{s} [g^\star(x_h)] - \Ex_{s} [g_{\star, \ghn}(x_h)]\right| 
+ \left|\Ex_{s} [g_{\star, \ghn}(x_h)] - \Ex_{s}[g_{\pih, \ghn}(x_h)] \right| \\
= &~ \left|\Ex_{s} \left[g^\star(x_h)- g_{\star, \ghn}(x_h)\right]\right| + \left| V_s(\pi_{\ghn}^\star; \ghn) - V_s(\pih; \ghn) \right| \\
\le &~ \left|\Ex_{s} \left[g^\star(x_h)- g_{\star, \ghn}(x_h)\right]\right| + 2\epsilon_V. \tag{using Eq.\eqref{eq:mmd_relative_loss}}
\end{align*}
    Now consider each individual context $x_h$ emitted in $s \in \Scal_h$: 
\begin{align*}
&~ g^\star(x_h)- g_{\star, \ghn}(x_h) \\
    = &~ \Ex_{r_h\sim R(x_h, \pi^\star(x_h))}[r_h] + \Ex_{s \circ \pi^\star(x_h)}[g^\star(x_{h+1})] 
    - \max_{a\in\Acal} \left(\Ex_{r_h \sim R(x_h, a)}[r_h] + \Ex_{s\circ a}[\ghn(x_{h+1})]\right) \\
\le &~ \Ex_{r_h \sim R(x_h, \pi^\star(x_h))}[r_h] + \Ex_{s \circ \pi^\star(x_h)}[\ghn(x_h)] \\&~
- \max_{a\in\Acal} \left(\Ex_{r_h \sim R(x_h, a)}[r] + \Ex_{s \circ a}[\ghn(x_h)]\right) 
    + |\Ex_{s \circ \pi^\star(x_h)}[\ghn(x_{h+1}) - g^\star(x_{h+1})]| \\
\le &~ |\Ex_{s \circ \pi^\star(x_h)}[\ghn(x_{h+1}) - g^\star(x_{h+1})]| \\
\le &~ (H-h)(2\epsilon_V + \sqrt{4M\epsilon_V + 2\epsilon_L} + 8\tau).
\end{align*}
    The second inequality is true since the second term optimizes over $a\in\Acal$ and the first term is the special case of $a = \pi^\star(x_h)$. The last inequality follows from the fact that if $s \in \Spol_h \Rightarrow s \circ a \in \Shit_{h+1}$ and we can therefore apply the induction hypothesis.
We can use the same argument to lower bound the above quantity. This gives
\begin{align*}
\left|\Ex_{s} [g^\star(x_h)] - \Ex_{s}[g_{\pih, \ghn}(x_h)] \right|
\le &~ (H-h)(2\epsilon_V + \sqrt{4M\epsilon_V + 2\epsilon_L} + 8\tau) + 2\epsilon_V.
\end{align*}
Next, we work with the second term in Equation~\eqref{eq:value_learned_decompose}:
\begin{allowdisplaybreaks}
\begin{align*}
&~ \left|\Ex_{s} [\gh(x_h)] - \Ex_{s}[g_{\pih, \ghn}(x_h)] \right| \\
\le &~ \sqrt{\Ex_{s} [\left(\gh(x_h) - g_{\pih, \ghn}(x_h)\right)^2]} \tag{Jensen's inequality} \\
= &~ \sqrt{L_s(\gh; \pih, \ghn) - L_s(g_{\pih, \ghn}; \pih, \ghn)} \tag{$g_{\pih, \ghn}$ is Bayes-optimal for $L_s(\cdot; \pih, \ghn)$} \\
\le &~ \sqrt{L_{\Sval_h}(\gh; \pih, \ghn) - L_{\Sval_h}(g_{\pih, \ghn}; \pih, \ghn)} \tag{\ldots for all $s$} \\
\le &~ \sqrt{L_{\Dval_h}(\gh; \pih, \ghn) - L_{\Sval_h}(g_{\pih, \ghn}; \pih, \ghn) + \epsilon_L} \\
\le &~ \sqrt{L_{\Dval_h}(g_{\star, \ghn}; \pih, \ghn) - L_{\Sval_h}(g_{\pih, \ghn}; \pih, \ghn) + \epsilon_L} \tag{$\hat g_h$ minimizes the first term over $\Gcal_h$, and $g_{\star, \ghn} \in \Gcal_h$ from Assumption~\ref{asm:polval_complete}} \\
\le &~ \sqrt{L_{\Sval_h}(g_{\star, \ghn}; \pih, \ghn) - L_{\Sval_h}(g_{\pih, \ghn}; \pih, \ghn) + 2 \epsilon_L} \\
= &~ \sqrt{\sum_{s \in \Sval_h} \Ex_{x\sim s}[(g_{\star, \ghn}(x) - g_{\pih, \ghn}(x))^2] + 2 \epsilon_L} \\
\le &~ \sqrt{\sum_{s \in \Sval_h} \Ex_{x\sim s}[2|g_{\star, \ghn}(x) - g_{\pih, \ghn}(x)|] + 2\epsilon_L} \tag{range of variables} \\
= &~ \sqrt{\sum_{s \in \Sval_h} 2\Ex_{x\sim s}[g_{\star, \ghn}(x) - g_{\pih, \ghn}(x)] + 2\epsilon_L} \tag{$g_{\star, \ghn}(x) \ge g_{\pih, \ghn}(x)$ ~ $\forall x$} \\
= &~ \sqrt{\sum_{s \in \Sval_h} 2(V_s(\pi_{\ghn}^\star; \ghn) - V_s(\pih; \ghn)) + 2\epsilon_L} \\
\le &~ \sqrt{4 M \epsilon_V + 2\epsilon_L}. \tag{$|\Sval_h|\le M$ and Eq.\eqref{eq:mmd_relative_loss}} 
\end{align*}
\end{allowdisplaybreaks}
Put together, we get the desired result for states $s \in \Sval_h$:
\begin{align}\label{eq:value_learned_guarantee}
\left|\Ex_{s} [g^\star(x_h)] - \Ex_{s}[\gh(x_h)] \right|
\le (H-h)(2\epsilon_V + \sqrt{4M\epsilon_V + 2\epsilon_L} + 8\tau) + 2\epsilon_V + \sqrt{4M\epsilon_V + 2\epsilon_L}.
\end{align}
It remains to deal with states $s \in \Shit_h \setminus \Sval_h$. According to the algorithm, this only happens when the MMD test suggests that the data set $\tilde D$ drawn from $s$ looks very similar to a previous data set $D \in \Dval_h$, which corresponds to some $s' \in \Sval_h$. So,
\begin{align*}
    &~ |\Ex_{s}[\gh(x_h)] - \Ex_{s}[g^\star(x_h)]| \\
\le &~ |\Ex_{s'}[\gh(x_h)] - \Ex_{s'}[g^\star(x_h)]| 
+ |\Ex_{s'}[\gh(x_h)] - \Ex_{s}[\gh(x_h)]| 
+ |\Ex_{s}[g^\star(x_h)] - \Ex_{s'}[g^\star(x_h)]| \\
\le &~ (H-h)(2\epsilon_V + \sqrt{4M\epsilon_V + 2\epsilon_L} + 8\tau) + 2\epsilon_V + \sqrt{4M\epsilon_V + 2\epsilon_L} \tag{$s' \in\Sval_h$ \& Eq.\eqref{eq:value_learned_guarantee}} \\
    &~  + |\EEx_{D}[\gh(x_h)] - \EEx_{\tilde D}[\gh(x_h)]| + |\EEx_{D}[g^\star(x_h)] - \Ex_{\tilde D}[g^\star(x_h)]|
+ 4\tau \\
\le &~ (H-h)(2\epsilon_V + \sqrt{4M\epsilon_V + 2\epsilon_L} + 8\tau) + 2\epsilon_V + \sqrt{4M\epsilon_V +2\epsilon_L} 
+ 2\tau + 2\tau + 4\tau \tag{MMD test fires} \\
= &~ (H+1-h)(2\epsilon_V + \sqrt{4M\epsilon_V + 2\epsilon_L} + 8\tau). \tag*{\qedhere}
\end{align*}
\end{proof}

\paragraph{Quality of Learned Policies and Meta-Algorithm Analysis.}
After quantifying the estimation error of the value function returned by \polvalfun, it remains to translate that into a bound on the suboptimality of the returned policy:
\begin{proposition}
    \label{prop:mmd_polguarantee}

    Assume we are in event $\Ecal$. Then the policy $\hat{\pi} = \hat \pi_{1:H}$
    returned by \polvalfun in Algorithm~\ref{alg:twosam} satisfies
    \begin{align*}
        V^{\hat{\pi}} \ge V^\star - p_{ul}^{\hat \pi} - 2 HM\tau_V -  H^2(2M\tau_V + \sqrt{4M^2\tau_V + 2\Tmax \tau_L} + 8\tau)
    \end{align*}
    where
    $p_{ul}^{\hat \pi}= \prob( \exists h \in [H] ~:~ s_h \notin
    \Spol ~|~ a_{1:H} \sim \hat\pi)$ is the probability of hitting
    an unlearned state when following $\hat \pi$.
\end{proposition}
\begin{proof}
    Proposition~\ref{prop:mmdlearnedvalue} states that for every learned state $s \in \Spol_h$
\begin{align}
V^\star(s) - Q^\star(s, \pih)
\le 2 M\tau_V + 2 (H-h)(2M\tau_V + \sqrt{4M^2\tau_V + 2\Tmax \tau_L} + 8\tau).
\end{align}
Using Lemma~\ref{lem:subopteq}, we can show that $\hat \pi$ yields expected return that is optimal up to
\begin{align*}
    V^\star - V^{\hat \pi} \leq &
    p_{ul}^{\hat \pi} + (1 -p_{ul}^{\hat \pi}) \sum_{h=1}^H
    (2 M\tau_V + 2 (H-h)(2M\tau_V + \sqrt{4M^2\tau_V + 2\Tmax \tau_L} + 8\tau))
    \\
    \leq &
    p_{ul}^{\hat \pi} + 2 HM\tau_V +  2(2M\tau_V + \sqrt{4M^2\tau_V + 2\Tmax \tau_L} + 8\tau)\sum_{h=1}^H
    (H-h)\\
    \leq &
    p_{ul}^{\hat \pi} + 2 HM\tau_V +  H^2(2M\tau_V + \sqrt{4M^2\tau_V + 2\Tmax \tau_L} + 8\tau)
     . \qedhere
\end{align*}
\end{proof}

\begin{lemma}
    \label{lem:metaalg_mmd}
    Consider running \metaalg with \dfslearn and \polvalfun (Algorithm~\ref{alg:metaalg} + \ref{alg:twosam}) with parameters
    \begin{align*}
        \nexp \geq & \frac{8}{\epsilon}\ln\left(\frac{4MH}{\delta}\right),
        \qquad n_{eval} \geq \frac{32}{\epsilon^2}\ln\left( \frac{8MH}{\delta}\right),\\
        \tau =& \frac{\epsilon}{2^6 3 H^2},\quad
        \tau_V = \frac{\epsilon^2}{2^8 3^4 M^2 H^4},\quad
        \tau_L = \frac{\epsilon^2}{2^7 3^2 H^4 \Tmax}
    \end{align*}
    Then with probability at least $1-\delta$,
    \metaalg returns a policy that is at least $\epsilon$-optimal
    after at most $MK$ iterations.
\end{lemma}
\begin{proof}
The proof is completely analogous to the proof of Lemma~\ref{lem:metaalg} except with using Proposition~\ref{prop:mmd_polguarantee} instead of Proposition~\ref{prop:localvalue_polguarantee}.
    We set the parameters $\tau$, $\tau_L$ and $\tau_V$ so that the policy guarantee of Proposition~\ref{prop:mmd_polguarantee}
is
    $V^{\hat \pi} \geq V^\star - p_{ul}^{\hat \pi} - \epsilon/8$. More specifically, we bound the guaranteed gap as
    \begin{align*}
        & 2 HM\tau_V +  H^2(2M\tau_V + \sqrt{4M^2\tau_V + 2\Tmax \tau_L} + 8\tau)
    \\ \leq
        & 2 MH\tau_V +  2MH^2\tau_V + 2MH^2 \sqrt{\tau_V} + H^2\sqrt{2\Tmax \tau_L} + 8H^2\tau
    \\ \leq
        & 6 MH^2\sqrt{\tau_V} + H^2\sqrt{2\Tmax \tau_L} + 8H^2\tau
    \end{align*}
    and then set $\tau$, $\tau_L$ and $\tau_V$ so that each terms evaluates to $\epsilon/24$.
\end{proof}

\paragraph{Proof of Theorem~\ref{thm:mmd_samplecomplex}.}
We now have all parts to complete the proof of
Theorem~\ref{thm:mmd_samplecomplex}.

\begin{proof}
    For the calculation, we instantiate all the parameters as
\begin{align*}
        \nexp = & \frac{8}{\epsilon}\ln\left(\frac{4MH}{\delta}\right),
        \quad \neval = \frac{32}{\epsilon^2}\ln\left( \frac{8MH}{\delta}\right),\quad
        \ntrain =  16K \left(\frac{2}{\tau_L^2} + \frac{1}{\tau_V^2} \right) \ln \left(\frac{ 12 H\Tmax |\Gcal|^2|\Pi|}{\delta} \right).
        \\
        \tau =& \frac{\epsilon}{2^6 3 H^2},\quad
        \tau_V = \frac{\epsilon^2}{2^8 3^4 M^2 H^4},\quad
        \tau_L = \frac{\epsilon^2}{2^7 3^2 H^4 \Tmax}\quad
    \Tmax = M(K+1)(1 + H\nexp).
\end{align*}
These settings suffice to apply all of the above lemmas for these algorithms and therefore
with these settings the algorithm outputs a policy that is at most
$\epsilon$-suboptimal, except with probability $\delta$. For the
sample complexity, since $\Tmax$ is an upper bound on the number of
datasets we collect (because $\Tmax$ is an upper bound on the number of execution of
\dfslearn at any level), and we also $\neval$ trajectories for each of the $MH$
iterations of \metaalg, the total sample complexity is
\begin{align*}
 H\Tmax \ntrain  + MH\neval = \otil\left(\frac{M^5 H^{12} K^4}{\epsilon^7}\log(|\Gcal||\Pi|/\delta) \log^3(1/\delta)\right). \tag*{\qedhere}
\end{align*}
\end{proof}

\subsection{Global Policy Algorithm}
\label{sec:global_policy}

\newcommand{\taupol}{\tau_{pol}}
\newcommand{\tauval}{\tau_{val}}
\newcommand{\pruned}{\textsc{Pruned}}
\newcommand{\learned}{\textsc{Learned}}
\newcommand{\eqs}{=_{s}}

See Algorithm~\ref{alg:global_policy}. 
As the other algorithms, this method learns states using depth-first search.
The state identity test is similar to that of \mainAlg at a high level: for any new path $p$, we derive an upper bound and a lower bound on $V^\star(p)$, and prune the path if the gap is small. Unlike in \mainAlg where both bounds are derived using the value function class $\Gcal$, here only the upper bound is from a value function (see Line~\ref{lin:stateidcond}), and the lower bound comes from \emph{Monte-Carlo roll-out} with a near-optimal policy, which avoids the need for on-demand exploration.  

More specifically, 
the global policy algorithm does not store data sets but maintains a global
policy, a set of learned  paths, and a set of pruned paths, all of which are updated over time. 
We always guarantee that $\hat \pi_{h:H}$ is near-optimal for any learned state at level $h$, and leverage this property to conduct state-identity test: if a new path $p$ leads to the same state as a learned path $q$, then Eq.\eqref{prog:global_upper} yields a tight upper bound on $V^\star(p)$, which can be achieved by $\hat \pi_{h:H}$ up to some small error and we check by Monte-Carlo roll-outs. 
If the test succeeds,
the path $p$ is added to the set $\pruned(h)$. Otherwise, all successor states are learned (or pruned) in a recursive manner, after which the state
itself becomes learned (i.e., $p$ added to $\learned(h)$).  Then, the policy at level $h$ is updated to be near-optimal for the newly learned state in addition to the previous ones (Line~\ref{lin:glb_pol_fit}). 
Once we change the global policy, however, all the pruned states need to be re-checked (Line~\ref{lin:recheck_pruned}), as their optimal values are only guaranteed to be realized by the previous global policy and not necessarily by the new policy. 

\begin{algorithm*}[htbp]
\SetKwInOut{Inputa}{Input}
\SetKwInOut{Outputa}{Output}
\SetKwProg{myfun}{Function}{}{}

\SetKwFunction{dfslearn}{dfslearn}
\SetKwFunction{metaalg}{main}
\SetKwFunction{testlearned}{TestLearned}

    \myfun{\metaalg}{
      Global $\textsc{learned}(h)$, $h\in [H]$\;
      Global $\textsc{pruned}(h)$, $h \in [H]$\;
      Global $\{\hat{\pi}_h\}_{h \in [H]}$\;
      \dfslearn($\circ$)\;
      \Return{$\{\hat{\pi}_h\}_{h \in [H]}$}\;
    }

    \myfun{\testlearned($p,h$)}{
      Collect dataset $D = \{(x_h,\bar{r})\}$ of size $\ntest$ where $x_h\sim p$, $a_{h:H} \sim \hat{\pi}_{h:H}$, $\bar{r} = \sum_{h'=h}^H r_{h'}$
      \label{lin:sample_d_global}\;
      \For{$q \in \textsc{learned}(h)$}{
        Collect dataset $D'_q = \{(x_h,\bar{r})\}$ of size $\ntest$ where $x_h\sim q$, $a_{h:H} \sim \hat{\pi}_{h:H}$, $\bar{r} = \sum_{h'=h}^H r_{h'}$
       \label{lin:sample_dprime_global} 
        \;
        Solve
        \begin{align} \label{prog:global_upper}
          V_{opt} &= \max_{g \in \Gcal} \Ex_{D} [g(x_h)] 
            \textrm{ s.t. } 
            \Ex_{D'_q}[g(x_h) - \bar{r}] \le \phi_h+2\tauval
        \end{align}
        \If{$V_{opt} \le \Ex_{D'_q}[\bar{r}] + \phi_h +4\tauval$ 
        and 
    $\Ex_{D}[\bar{r}] \ge \Ex_{D'_q}[\bar{r}] - 2\tauval$
    \label{lin:stateidcond}
    }{
          \Return{true}\;
        }
      }
      \Return{false}\;
    }

\myfun{\dfslearn($p$)}{
  Let $h = |p|-1$ the current level\;
  \If{\texttt{not\_called\_from\_Line\_\ref{lin:learn_pruned}} and \testlearned($p,h$)}{
    Add $p$ to $\textsc{pruned}(h)$\;
    \Return
    \label{lin:retlearned}\;
  }
  \For{$a \in \Acal$}{
    \dfslearn($p \circ a$) \label{lin:glb_pol_descend}\;
  }
  Add $p$ to $\textsc{learned}(h)$\;
  \For{$q \in \textsc{learned}(h)$}{
    Collect dataset $D_q$ = $\{(x_h,a_h,\bar{r})\}$ of size $\ntrain$ where $x_h \sim q$, $a_h \sim \textrm{Unif}$, $a_{h+1:H} \sim \hat{\pi}_{h+1:H}$, $\bar{r} = \sum_{h'=h}^H r_{h'}$
    \label{lin:sample_learnedq}
    \;
    $\hat V(q) \gets \max_{\pi \in \Pi} \Ex_{D_q} [K\one\{a_h = \pi(x_h)\} \bar{r}]$\label{lin:glb_localval}
    \;
  }
  Update $\hat{\pi}_h$ to be any policy satisfying \label{lin:glb_pol_fit}
  \begin{align*}
    \forall q \in \textsc{learned}(h) \quad \hat \Ex_{D_q} [K \one\{a_h = \pi(x_h)\} \bar{r}] \ge \hat V(q) - 2 \taupol
  \end{align*}
    \For{$q \in \textsc{pruned}(h)$}{
      \If{$\testlearned(q,h) = false$}{
      \label{lin:recheck_pruned}
        remove $q$ from $\pruned(h)$\;
        \dfslearn($q$) \label{lin:learn_pruned}\;
      }
    }
  \Return
  \label{lin:retpathlearned}\;
}
\caption{Global Policy}
\label{alg:global_policy}
\end{algorithm*}

\subsubsection{Computational efficiency}
The algorithm contains three non-trivial computational components. In Eq.\eqref{prog:global_upper}, a linear program is solved to determine the optimal value estimate of the current path given the value of one learned state (LP oracle).
In Line~\ref{lin:glb_localval}, computing the value of each learned path can be reduced to multi-class cost-sensitive classification as in the other two algorithms (CSC oracle).
Finally, fitting the global policy in Line~\eqref{lin:glb_pol_fit} requires the same problem as the policy fitting procedure discussed in Section~\ref{sec:lagrange} (multi data set classification oracle).

As with the previous algorithm, we assume no error in the oracles ($\epsfeas = \epssub = 0$) in the following to simplify the analysis.
\subsubsection{Sample complexity}
\begin{theorem}\label{thm:global_policy}
Consider a Markovian contextual decision process with deterministic dynamics over $M$ hidden states, as described in Section~\ref{sec:prelim}. When Assumption~\ref{asm:pol_complete} and \ref{asm:val_realize} hold, for any $\epsilon, \delta \in (0, 1)$, the global policy algorithm (Algorithm~\ref{alg:global_policy}) returns a policy $\pi$ such that $V^\star - V^{\pi} \le \epsilon$ with probability at least $1-\delta$, after collecting at most 
$\otil\left(\frac{M^3H^3K}{\epsilon^2}\log\left(|\Pi||\Gcal|/\delta\right)\right)$
trajectories.
\end{theorem}

In the following, we prove this statement but first introduce helpful notation:

\begin{definition}[Deviation Bounds]
    \label{def:devbounds_global}
We say the deviation bound holds for a data set of
    $\ntrain$ observations sampled from $q$ in
            Line~\ref{lin:sample_learnedq} during a call to \dfslearn if for all  
     $\pi \in \Pi_h$ 
    \begin{align*}
        |\EEx_{D_q} [K\one\{a = \pi(x)\} \bar{r}] - \Ex_{q, \hat \pi_{h+1}:H}[K\one\{a_h = \pi(x_h)\}\bar r]| \leq \taupol,
    \end{align*}
    where we use $\Ex_{q, \hat \pi_{h+1:H}}[\cdot]$ as shorthand for $\Ex[ \cdot | s_h = s, a_h \sim \textrm{Uniform}(K), a_{h+1:H} \sim \hat \pi_{h+1:H}]$ with $s$ being the state reached by $p$ and $\hat \pi_{h+1:H}$ being the current policy when the data set was collected.
We say the deviation  
bound holds for a data set of
    $\ntest$ observations sampled in
            Line~\ref{lin:sample_d_global} during a call to \testlearned if for all
            $g \in \Gcal_h$:
    \begin{align*}
        |\EEx_{D} [g(x_h)] - \Ex_{p}[g(x_h)]| & \leq \tauval, \qquad 
        |\EEx_{D} [\bar r] - V^{\hat \pi_{h+1:H}}(p)| \leq \tauval.
    \end{align*}
We say the deviation  
bound holds for a data set of
    $\ntest$ observations sampled in
            Line~\ref{lin:sample_dprime_global} during a call to \testlearned if for all 
            $g \in \Gcal_h$:
    \begin{align*}
        |\EEx_{D'_q} [g(x_h)] - \Ex_{q}[g(x_h)]| & \leq \tauval, \qquad 
        |\EEx_{D'_q} [\bar r] - V^{\hat \pi_{h+1:H}}(q)| \leq \tauval.\\
    \end{align*}
\end{definition}

\paragraph{Learning Values using Depth First Search.}
We first show that if the current policy is close to optimal for all learned states, then the policy is also good on all states for which \testlearned returns true.

\begin{lemma}[Policy on Tested States] \label{lem:glb_pol_testlearned}
  Consider a call of \testlearned at path $p$ and level $h$ and assume
    the deviation bounds of Definition~\ref{def:devbounds_global} hold for all data sets collected during this and all prior calls. Assume further that $\hat{\pi}_{h:H}$ satisfies
  $V^{\hat{\pi}_{h:H}}(q) \ge V^\star(q) - \phi_h$ for all
    $q \in \learned(h)$. Then $g^\star$ is always feasible for the program in Equation~\eqref{prog:global_upper}  and if \testlearned returns true, then
   the current policy $\hat{\pi}_{h:H}$ is near optimal for $p$, that is
      $V^{\hat \pi_{h:H}}(p) \geq V^\star(p) - \phi_h - 8 \tauval$.
\end{lemma}
\begin{proof}
The optimal value function $g^\star$ is always feasible since
\begin{align*}
    \EEx_{D'}[g^\star(x) - \bar r]   \leq
    V^\star(q) - V^{\hat \pi_{h:H}}(q) + 2\tauval \le \phi_h + 2\tauval.
\end{align*}
Here, we first used the deviation bounds and then the assumption about the performance of the current policy on learned states.
Therefore, $V_{opt} \geq \EEx_D [g^\star(x)] \geq V^\star(p) - \tauval$ cannot underestimate the optimal value of $p$ by much.
Consider finally the performance of the current policy on $p$ if \testlearned returns true:
    \begin{align*}
V^{\hat \pi_{h:H}}(p)  \geq &~ \EEx_{D} [\bar r] - \tauval \geq \EEx_{D'}[\bar r] - 3 \tauval\\
    \geq &~ V_{opt} - 3 \tauval - 4 \tauval - \phi_h
     \geq V^\star(p) - 8 \tauval - \phi_h.
\end{align*}
Here, the first inequality follows from the deviation bounds, the second from the second condition of the if-clause in \testlearned,
    the third from the first condition of the if-clause and finally the fact that $V_{opt}$ is an accurate estimate of the optimal value of $p$.
\end{proof}
Thus, the \testlearned routine can identify paths where the current policy is close to optimal if this policy's performance on all learned states is good. Next, we
prove that the policy has near-optimal performance on all the learned states. 

\begin{lemma}[Global policy fitting] \label{lem:glb_pol_vallearn}
Consider a call of
    \dfslearn($p$) at level $h$ and assume the deviation bounds hold for all data sets collected during this and all prior calls.
    Then the program in Line~\ref{lin:glb_pol_fit} is always
feasible and after executing that line, we have $\forall q \in
\textsc{learned}(h)$,
\begin{align}\label{eq:glb_pol_relative_loss}
Q^{\hat{\pi}_{h+1:H}}(q, \hat{\pi}_h) \ge Q^{\hat{\pi}_{h+1:H}}(q, 
\star) - 3 \taupol, 
\end{align}
where $\star$ is a shorthand for $\pi_{\hat{\pi}_{h+1:H}}^\star$, the
policy defined in Assumption \ref{asm:pol_complete} w.r.t.~the current
policy $\hat \pi_{h+1:H}$.  This implies that if all children nodes $q'$ of
$q$ satisfy $V^{\hat{\pi}_{h+1:H}}(q') \ge V^\star(q') - \beta$ for some
$\beta$, then $V^{\hat{\pi}_{h:H}}(q) \ge V^\star(q) - \beta - 3\taupol$.
\end{lemma}

\begin{proof}
We prove feasibility by showing that $\pi_{\hat{\pi}_{h+1:H}}^\star$ is always feasible. For each $q\in \learned(h)$, let $\hat \pi^q_h$ denote the policy that achieves the maximum in computing $\hat V(q)$. Then 
\begin{align*}
\EEx_{D_q} [K \one\{a_h = \pi_{\hat{\pi}_{h+1:H}}^\star(x_h)\} \bar{r}]
\ge Q^{\hat{\pi}_{h+1:H}}(q, \star) - \taupol 
\ge Q^{\hat{\pi}_{h+1:H}}(q, \hat{\pi}^q_h) - \taupol 
\ge \hat V(q) - 2\taupol.
\end{align*}
The first and last inequality are due to the deviation bounds and the second inequality follows from definition of $\pi_{\hat{\pi}_{h+1:H}}^\star$. 
This proves the feasibility. 
Now, using this inequality along  with $\hat V(q) = \max_{\pi \in \Pi} \Ex_{D_q} [K\one\{a_h = \pi(x_h)\} \bar{r}]$,
    we can relate $\hat V(q)$ and $Q^{\hat{\pi}_{h+1:H}}(q, \star)$:
    \begin{align*}
\hat V(q) \ge \EEx_{D_q} [K \one\{a_h = \pi_{\hat{\pi}_{h+1:H}}^\star(x_h)\} \bar{r}] \ge Q^{\hat{\pi}_{h+1:H}}(q, \star) - \taupol.
\end{align*}
Finally, since 
    $\hat \pi_h$ is feasible in Line~\ref{lin:glb_pol_fit},
\begin{align*}
    V^{\hat \pi_{h:H}}(q) = Q^{\hat{\pi}_{h+1:H}}(q, \hat{\pi}_h) \ge \hat V(q) - 2\taupol \ge Q^{\hat{\pi}_{h+1:H}}(q, \star) - 3 \taupol.
\end{align*}
To prove the implication, consider the case where for $a \in \Acal$, all paths $q' = q \circ a$ satisfy $V^{\hat{\pi}_{h+1:H}}(q') \ge V^\star(q') - \beta$. Then
\begin{align*}
V^\star(q) - V^{\hat{\pi}_{h:H}}(q)
    \le &~ V^\star(q) - Q^{\hat{\pi}_{h+1:H}}(q, \star) + 3 \taupol 
\le V^\star(q) - Q^{\hat{\pi}_{h+1:H}}(q, \pi^\star) + 3 \taupol \\
= &~ \Ex_{q' \sim q \circ \pi^\star}[V^\star(q') - V^{\hat{\pi}_{h+1:H}}(q')] + 3 \taupol 
\le \beta + 3\taupol,
\end{align*}
    where we first used the inequality from above and then the fact that $\pi_{\hat{\pi}_{h+1:H}}^\star$ is optimal given the fixed policy $\hat \pi_{h+1:H}$.
    The equality holds since both $V^\star(q) - Q^{\hat{\pi}_{h+1:H}}(q, \pi^\star)$ both are with respect to $a_h \sim \pi^\star_h$ and finally we apply the assumption.
\end{proof}

We are now ready to apply both lemmas above recursively to control the
performance of the current policy on all learned and pruned paths:
\begin{lemma}\label{lem:global_polqual}
    Set $\phi_h = (H-h+1)(8 \tauval + 3 \taupol)$ and consider a call to $\dfslearn(p)$ at level $h$. Assume the deviation bounds hold for all data sets collected until this call terminates. Then for all $p \in \learned(h)$, the current policy satisfies 
    \begin{align*}
        V^{\hat\pi_{h:H}}(p) \ge V^\star(p) - \phi_h
    \end{align*}
    at all times except between adding a new path and updating the policy.
    Further, for all $p \in \pruned(h)$ the currently policy satisfies 
    \begin{align*}
        V^{\hat\pi_{h:H}}(p) \ge V^\star(p) - \phi_h - 8 \tauval
    \end{align*}
    whenever \dfslearn returns from level $h$ to $h-1$.
\end{lemma}
\begin{proof}
    We prove the claim inductively. For $h = H+1$ the statement is trivially
    true since there are no actions left to take and therefore the value of
    all policies is identical $0$ by definition.  
    
    Assume now the statement holds for $h+1$.  We first study the
    learned states. To that end, consider a call to $\dfslearn(p)$ at
    level $h$ that does not terminate in Line~\ref{lin:retlearned} and
    performs a policy update.  Since \dfslearn is called recursively
    for all $p \circ a$ with $a \in \Acal$ before $p$ is added to
    $\learned(h)$ and every path that \dfslearn is called with either
    makes that path learned or pruned, all successor states of $p$ are
    in $\pruned(h)$ or $\learned(h)$ when $p$ is added.  Since the
    statement holds for $h+1$, for all successor paths $p'$ we have
    $V^{\hat\pi_{h+1:H}}(p') \ge V^\star(p') - \phi_{h+1} - 8
    \tauval$.  We can apply Lemma~\ref{lem:glb_pol_vallearn} and
    obtain that after changing $\hat \pi_h$, it holds that for all $q
    \in \learned(h)$ $V^{\hat \pi_{h:H}}(q) \geq V^\star(q) -
    \phi_{h+1} - 8 \tauval - 3 \taupol = V^\star(q) - \phi_h$.  Since
    that is the only place where the policy changes or a state is
    added to $\learned(h)$, this proves the first part of the
    statement for level $h$.

    For the second part, we can apply Lemma~\ref{lem:glb_pol_testlearned} which
    claims that for all paths $q$ for which $\testlearned(q,h)$ returns true, it
    holds that $V^{\hat \pi_{h:H}}(q) \geq V^\star(q) - \phi_h - 8 \tauval$. It
    remains to show that whenever \dfslearn returns to a higher level, for all paths $q \in
    \pruned(h)$, $\testlearned(q,h)$ evaluates to true. 
    This condition can only be violated when we add a new state to $\pruned(h)$ or change the policy $\hat \pi_{h:H}$.
    
    For the later case, we explicitly check the condition in
    Lines~\ref{lin:recheck_pruned}-\ref{lin:learn_pruned} after we change the
    policy before returning.  Therefore \dfslearn can only return after
    Line~\ref{lin:learn_pruned} without further recursive calls to \dfslearn if
    \testlearned evaluated to true for all $q \in \pruned(h)$. The statement is
    therefore true if the algorithm returns after Line~\ref{lin:learn_pruned}.
    Further, a path can only be added to $\pruned(h)$ after we explicitly checked that
    \testlearned evaluates true for it before we return in
    Line~\ref{lin:retlearned}. Hence, the second part of the statement also
    holds for $h$ which completes the proof.
\end{proof}

\begin{lemma}[Termination]\label{lem:global_termination}
    Assume the deviation bounds hold for all Data sets collected during the first $\Tmax = 3 M^2HK$ calls of \dfslearn and \testlearned. The algorithm terminates during these calls and at all times for all $h \in [H]$ it holds $|\learned(h)|\leq M$. Moreover, the number of paths that have ever been added to $\pruned(h)$ (that is, counting those removed in Line~\ref{lin:recheck_pruned}) is at most $KM$.
\end{lemma}
\begin{proof}
    Consider a call to $\testlearned(p,h)$ where $p$ leads to the same state as a $q \in \learned(h)$. Assume the deviation bounds hold for all data sets collected during this call and before, and we can show that \testlearned must evaluate to true:
    Using Lemma~\ref{lem:global_polqual}
    we get that on all learned paths $p$ it holds that $V^{\hat\pi_{h:H}}(p) \ge V^\star(p) - \phi_h$. Therefore, $g^\star$ is feasible in \eqref{prog:global_upper}
        since $\EEx_{D'}[g^\star(x) - \bar r]
        \leq
        V^\star(q) - V^{\hat \pi_{h:H}}(q) + 2\tauval \le \phi_h + 2\tauval$.
    This allows us to relate $V_{opt}$ to the optimal value as
    \begin{align*}
        V_{opt} \geq \EEx_D [g^\star(x)] \geq V^\star(p) - \tauval.
    \end{align*}
    It further holds that
    \begin{align*}
\EEx_D [\bar r] \geq V^{\hat \pi_{h:H}}(p) - \tauval = V^{\hat \pi_{h:H}}(q)
    - \tauval \geq \EEx_{D'} [\bar r] - 2 \tauval. 
\end{align*}
    and so the second condition in the if-clause holds. For the first condition, let $\hat g$ be the function that achieves the maximum in the computation of $V_{opt}$. Then
\begin{align*}
    V_{opt} = &~ \EEx_{D}[\hat g(x_h)] \leq \Ex_{s}[\hat g(x_h)] + \tauval \le \EEx_{D'_q}[\hat g(x_h)] + 2\tauval \\
    \le &~ \EEx_{D'_q}[\bar r] + \phi_h + 2 \tauval + 2\tauval = \EEx_{D'_q}[\bar r] + \phi_h + 4\tauval.
\end{align*}
Then the first condition is also true and \testlearned returns true.
    Therefore, \testlearned evaluates to true for all paths that reach the same
    state as a learned path. As a consequence, if \dfslearn is called with such a path it
    returns in Line~\ref{lin:retlearned}. Furthermore, as long as all deviation
    bounds hold, the number of learned paths per level is bounded by
    $|\learned(h)| \leq M$.  
    
    We next show that the number of paths that have ever appeared in $\pruned(h)$ is at most $KM$. 
This is true since there are at most $KM$
    recursive calls to \dfslearn at level $h$ from level $h-1$ and only during
    those calls a path can be added to $\pruned(h)$ that has not been in
    $\pruned(h)$ before.

    Assume the deviation bounds hold for all data sets collected during the first $\Tmax$  calls of \dfslearn.
    There can be at most $MH$ calls of \dfslearn in which a path is learned. 
    Since the recursive call in Line~\ref{lin:learn_pruned} always learns a new state at the next level, the only way to grow $\pruned(h)$ is via the recursive call on Line~\ref{lin:glb_pol_descend}, which occurs at most $MKH$ times. 
    Therefore
    the algorithm terminates after at most $MH+MHK$ calls to \dfslearn. Each of these  calls can make at most $1$ call to \testlearned unless it learns a new state and calls \testlearned up to $|\pruned(h)| + 1 \le MK + 1$ times. Therefore, the total number of calls to \testlearned is bounded by $MH(MK+1) + MHK$. The lemma follows by noticing that both numbers of calls are bounded by $\Tmax$.  
\end{proof}

\begin{lemma}
    \label{lem:probE_global}
    Let $\Ecal$ be the event that the deviation bounds in Definition~\ref{def:devbounds_global} hold for all data sets collected during Algorithm~\ref{alg:global_policy}.    Set $\ntrain$ and $\ntest$ such that
    \begin{align*}
        \ntrain & \geq \frac{16K}{\taupol^2}\log\left( \frac{16\Tmax M|\Pi||\Gcal|}{\delta}\right)\\
        \ntest & \geq \frac{1}{2\tauval^2}\log\left(\frac{16\Tmax M|\Pi||\Gcal|}{\delta}\right)
    \end{align*}
    Then $\prob(\bar \Ecal) \leq \delta$.
\end{lemma}
\begin{proof}
    Consider a single data set $D_q$ collected in $\dfslearn(p)$ at level $h$ where $p$ is learned for $q \in \learned(h)$. 
For the random variable $K\one\{\pi(x_h) = a_h\}\bar r$, since $a_h$ is chosen uniformly at random, it is not
hard to see that both the variance and the range are upper-bounded by $2K$ (see for example Lemma~14 by \citet{jiang2017contextual}). As
such, Bernstein's inequality and a union bound over all $\pi \in \Pi_h$ gives that with probability $1-\delta'$,
\begin{align*}
|\EEx_{D_q} [K\one\{a = \pi(x)\} \bar{r}] - \Ex_{q, \hat \pi_{h+1}:H}[K\one\{a_h = \pi(x_h)\}\bar r]| \leq \sqrt{\frac{4 K\log(2
   |\Pi|/\delta')}{\ntrain}} + \frac{4K}{3\ntrain}\log(2|\Pi| /\delta').
\end{align*}
Consider a single data set $D$ collected in $\testlearned(p,h)$. 
By Hoeffding's
inequality and a union bound, with probability $1-\delta'$, for all $g
\in \Gcal_{h}$ 
\begin{align*}
    |\EEx_{D} [g(x_h)] - \Ex_{p}[g(x_h)]| & \leq \sqrt{\frac{\log(2|\Gcal|/\delta')}{2\ntest}}
\end{align*}
    Analogously, for a data set $D'_q$ collected during $\testlearned(p,h)$ with $q \in \learned(q)$, we have
    with probability at least $1 - \delta'$
    that
    \begin{align*}
        |\EEx_{D'_q} [g(x_h)] - \Ex_{q}[g(x_h)]| &\leq \sqrt{\frac{\log(2|\Gcal|/\delta')}{2\ntest}}
    \end{align*}
Further, again by Hoeffding's inequality and a union bound we get that for a single data set $D$ collected in $\testlearned(p,h)$ and
a single data set $D'_q$ collected during $\testlearned(p,h)$ with $q \in \learned(q)$    with probability at least $1 - \delta'$ it holds

\begin{align*}
    |\EEx_{D'_q} [\bar r] - V^{\hat \pi_{h+1:H}}(q)| & \leq \sqrt{\frac{\log(4/\delta')}{2\ntest}} \quad \textrm{and}\\
    |\EEx_{D} [\bar r] - V^{\hat \pi_{h+1:H}}(p)| & \leq \sqrt{\frac{\log(4/\delta')}{2\ntest}}.
\end{align*}
    Combining all these bounds with a union bound and using $\delta' = \frac{\delta}{4M\Tmax }$, we get that the deviation bounds hold 
    for the first $M\Tmax$ data sets of the form $D'_q$ and $D_q$ and $D$ with probability at least $1 - \delta$.
    Using Lemma~\ref{lem:global_termination}, this is sufficient to show that $\prob(\bar \Ecal) \leq \delta$.
\end{proof}

\paragraph{Proof of Theorem~\ref{thm:global_policy}.}
We now have all parts to complete the proof of
Theorem~\ref{thm:local_value}. 
\begin{proof}For the calculation, we instantiate all the parameters as
\begin{align*}
    \taupol &= \frac{\epsilon}{6H}, 
    \quad \tauval = \frac{\epsilon}{6H}, 
    \quad \phi_h = (H-h+1)(8 \tauval + 3 \taupol),
    \quad \Tmax = 3M^2HK ,\\
  \ntest &= \frac{\log(16\Tmax M|\Pi||\Gcal|/\delta)}{2\tauval^2}, 
    \quad \ntrain= \frac{16K\log(16\Tmax M|\Pi||\Gcal|/\delta)}{\taupol^2}. 
\end{align*}
These settings suffice to apply all of the above lemmas for these algorithms and therefore
with these settings the algorithm outputs a policy that is at most
$\epsilon$-suboptimal, except with probability $\delta$. 
    For the
    sample complexity, since $\Tmax$ is an upper bound on the number of calls to \testlearned and at most $M$ states are learned per level $h \in [H]$, we collect a total of at most the following number of episodes:
\begin{align*}
    & (1+M)\Tmax \ntest + M^2H\ntrain\\
& = \otil\left(\frac{\Tmax MH^2}{\epsilon^2}\log(|\Pi||\Gcal|/\delta) + \frac{M^2H^3K}{\epsilon^2}\log(|\Pi||\Gcal|/\delta)\right)\\
& = \otil\left(\frac{M^3KH^3}{\epsilon^2}\log(|\Pi||\Gcal|/\delta)\right). \tag*{\qedhere}
\end{align*}
\end{proof}

\section{Oracle-Inefficiency of OLIVE}
\label{sec:nphardapp}

As explained in Section~\ref{sec:towardgeneral}
Theorem~\ref{thm:olive_notoracle} follows directly from
Theorem~\ref{thm:olive_nphard_g} 
and Proposition~\ref{prop:oracles_tabular_easy} by proof by contradiction with $P \neq NP$.
In the following two sections, we first prove Proposition~\ref{prop:oracles_tabular_easy} and then Theorem~\ref{thm:olive_nphard_g}.

\subsection{Proof for Polynomial Time of Oracles}
\begin{proof}[Proof of Proposition~\ref{prop:oracles_tabular_easy}]
   We prove the claim for each oracle separately 
    \begin{enumerate}
        \item \textbf{CSC-Oracle:} For tabular functions, the objective can be decomposed as
        \begin{align}
            n^{-1} \sum_{i=1}^n c\ii(\pi(x\ii))
            = \sum_{x \in \Xcal} n^{-1}
            \sum_{i=1}^n \one\{ x = x\ii\} c\ii(\pi(x)).
        \end{align}
        Each of the $|\Xcal|$ terms only depend on $\pi(x)$ but not on any
        action chosen for different observations. Hence, since $\Pi =
        (\Xcal \rightarrow \Acal) \triangleq \Acal^{|\Xcal|}$, the action
            chosen by $\hat \pi = n^{-1} \argmin_{\pi \in \Pi} \sum_{i=1}^n
        c\ii(\pi(x\ii))$ for $x \in \Xcal$ is $\argmin_{a \in \Acal}
        \sum_{i=1}^n \one\{ x = x\ii\} c\ii(\pi(x))$. To compute $\hat
        \pi$, we first compute for each $x$ the total cost vector
        $\sum_{i=1}^n \one\{ x = x\ii\} c\ii(\pi(x))$ and then pick the
        smallest entry as the action for $\hat \pi(x)$.
        Per $x$, this takes $O(Kn)$ operations and therefore, the total runtime
        for this oracle is $O(nK|\Xcal|)$. 
        \item \textbf{LS-Oracle:}
        Similarly to the CSC objective, the least-squares objective can be decomposed as
            \begin{align}
                \sum_{i=1}^n (v\ii - g(x\ii))^2
            = \sum_{x \in \Xcal}
            \sum_{i=1}^n \one\{ x = x\ii\} (v\ii - g(x))^2
        \end{align}
            and therefore $\hat g = \argmin{g \in \Gcal} \sum_{i=1}^n (v\ii - g(x\ii))^2$ can be computed for each observation separately. A minimizer per observation $x$ of $\sum_{i=1}^n \one\{ x = x\ii\} (v\ii - g(x))^2$ is 
            $\hat g(x) = \frac{\sum_{i=1}^n \one\{ x = x\ii\} v\ii}{\sum_{i=1}^n \one\{ x = x\ii\}}$, where we set $\hat g(x)$ arbitrarily if $\sum_{i=1}^n \one\{ x = x\ii\} = 0$. This can be computed with $O(n)$ operations and therefore the total runtime of the LS-oracle is $O(|\Xcal|n)$.
        \item \textbf{LP-Oracle:} 
            We parameterize $g \in \Gcal$ by vectors $\theta \in \RR^{|\Xcal|}$ where each the value of $g$ for each $x \in \Xcal$ is associated with a particular entry $\theta_x$ of $\theta$. Then the LP problem reduces to a standard linear program in $\RR^{|\Xcal|}$. \citet{khachiyan1980polynomial,grotschel1981ellipsoid} have shown using the ellipsoid method, these problems can be solved approximately in polynomial time. Note that the initial ellipsoid can be set to any ellipsoid containing $[0, 1]^{|\Xcal|}$ due to the normalization of rewards. 
            Further, the volume of the smallest ellipsoid can be upper bounded  by a polynomial in $\epsfeas$ using the fact that we only require a solution that is feasible up to $\epsfeas$ and applying the ellipsoid method to the extended polytope with all constraints relaxed by $\epsfeas$. \qedhere
    \end{enumerate}
\end{proof}

\subsection{\olive is NP-hard in tabular MDPs}
Instead of showing Theorem~\ref{thm:olive_nphard_g} directly, we first show the following simpler version: 
\begin{theorem}\label{thm:olive_nphard_g_simple}
  Let $\Pcal$ denote the family of problems of the
  form~\eqref{eqn:oliveprobg}, parameterized by
  $(\Xcal, \Acal, D_0, \Dcal)$ with implicit
  $\Gcal = (\Xcal \to [0,1])$ and $\Pi = (\Xcal \to \Acal)$ (i.e., the
  tabular value-function and policy classes) and with $\phi=0$. $\Pcal$
  is NP-hard.
\end{theorem}

Some remarks are in order about this statement
\begin{enumerate}
\item Our proof actually shows that it is NP-hard to find an
  $\epsilon$-approximate solution to these optimization problems, for polynomially small $\epsilon$. 
\item The two theorems differ in whether the data sets ($D_i \in \Dcal$) are chosen adversarially (Theorem~\ref{thm:olive_nphard_g_simple}), or induced naturally from an actual run of \olive (Theorem~\ref{thm:olive_nphard_g}). Therefore, Theorem~\ref{thm:olive_nphard_g} is strictly stronger. 
\item At a high level, these results imply that \olive in general must
  solve NP-hard optimization problems, presenting a barrier for
  computational tractability.
\item These results also hold with imperfect expectations and
  polynomially small $\phi$.
\item We use the $(\Gcal, \Pi)$ representation here but similar
  results hold with $\Fcal$ representation (i.e., approximating the $Q$-function; see Theorems~\ref{thm:olive_nphard_f_simple} and \ref{thm:olive_nphard_f}).
\end{enumerate}

For intuition we first sketch the proof of Theorem~\ref{thm:olive_nphard_g_simple}. The complete proof follows below.
\begin{figure}
    \begin{center}
\includegraphics[width=.6\textwidth]{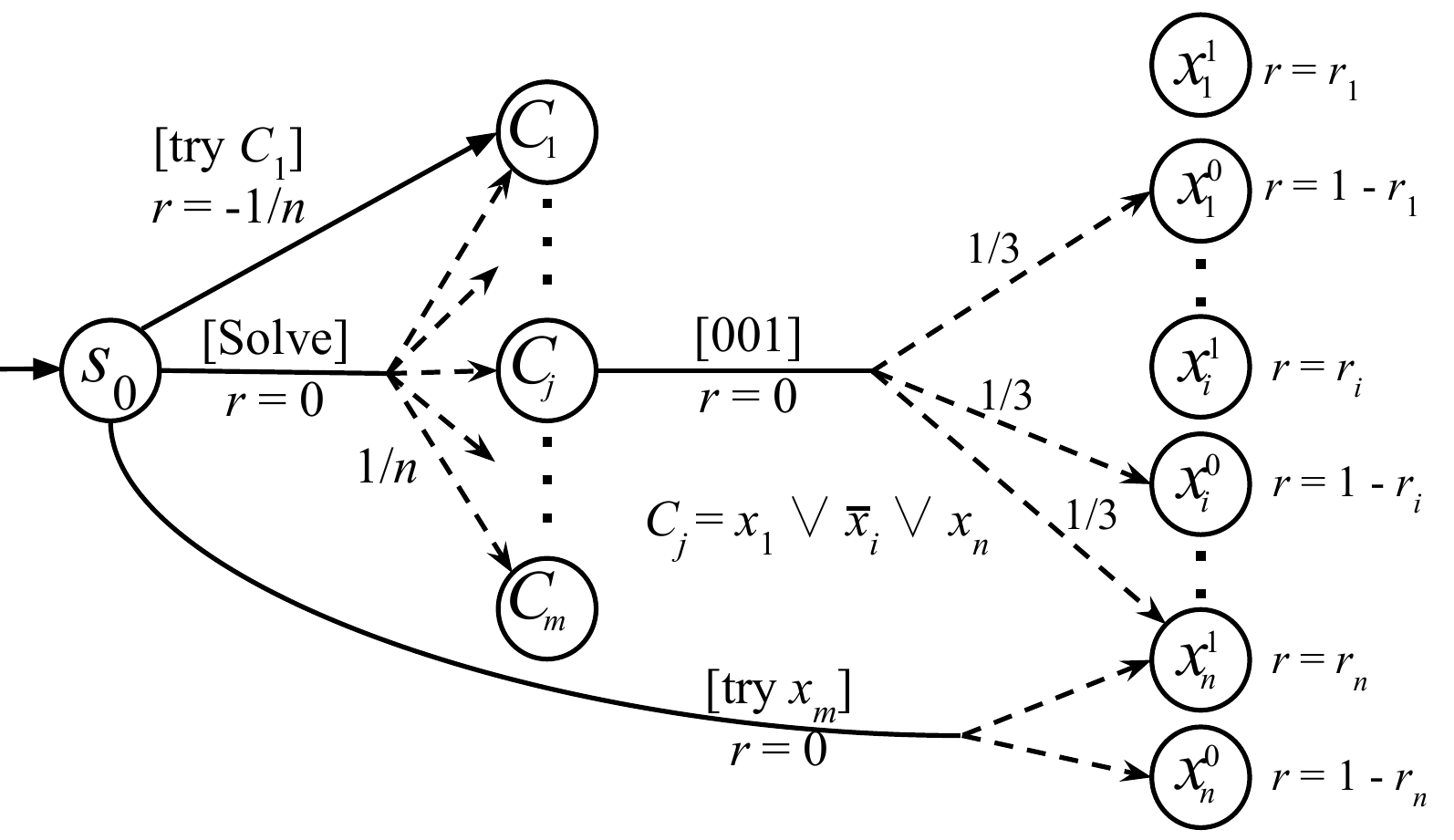}
    \end{center}
    \caption{Family of MDPs that are determined up to terminal rewards $r_1, \dots, r_n \in [0, 1]$. Finding the optimal value of the most optimistic MDP in this family solves the encoded 3-SAT instance.
Solid arrows represent actions and dashed arrows represent random transitions.}
    \label{fig:nphard1}
\end{figure}

\begin{proof}[Proof Sketch of Theorem~\ref{thm:olive_nphard_g_simple}]
  We reduce from 3-SAT. Let $\psi$ be a 3-SAT formula on $n$ variables
  $x_1,\ldots,x_n$ with $m$ clauses $c_1,\ldots,c_m$. We construct a
  family of MDPs as shown in Figure~\ref{fig:nphard1} that
  encodes the 3-SAT problem for this formula as follows: For each
  variable $x_i$ there are two terminal states $x_i^1$ and $x_i^0$
  corresponding to the Boolean assignment to the variable. For each
  variable, the reward in either $x_i^1$ or $x_i^0$ is $1$ and $0$ in
  the other. The family of MDPs contains all possible combinations of
  such terminal rewards.  There is also one state per clause $c_j$ and
  one start state $s_0$. From each clause, there are 7 actions, one
  for each binary string of length 3 except ``000." These actions all
  receive zero instantaneous reward. For clause
  $c_\ell = x_i \vee \bar{x}_j \vee \bar{x}_k$, the action
  "$b_1b_2b_3$" transitions to states $x_i^{b_1}, x_j^{1-b_2}$, or
  $x_k^{1-b_3}$, each with probability $1/3$. The intuition is that
  the action describes which literals evaluate to true for this
  clause.
  From the start state, there are $n+m+1$ actions. For each variable
  $x_i$, there is a [try $x_i$] action that transitions uniformly to
  $x_i^0, x_i^1$ and receives $0$ instantaneous reward. For each
  clause $c_j$ there is a [try $c_j$] action that transitions
  deterministically to the state for clause $c_j$, but receives reward
  $-1/n$. And finally there is a [solve] action that transitions to a
  clause state uniformly at random.

  For each $x_i$, we introduce a constraint into
  Problem \eqref{eqn:oliveprobg} corresponding to the [try $x_i$]
  action. These constraints impose that the optimal $\hat{g} \in \Gcal$
  satisfies
  $
      \forall i \in [m]: \hat g(x_i^0) + \hat g(x_i^1) = 1.
  $
  We also introduce constraints for the [try $c_j$] actions and from
  $s_0$.  Recall that values must be in $[0,1]$.

  With these constraints, if the 3-SAT formula has a satisfying
  assignment, then the optimal value from the start state is $1$, and
  it is not hard to see that there exists function $\hat{g} \in \Gcal$
  that achieves this optimal value, while satisfying all constraints
  with a $\hat \pi \in \Pi$. Conversely, if the value of the start date is $1$, we
  claim that the 3-SAT formula is satisfiable. In more detail, the
  policy must choose the [solve] action, and the value function must
  predict that each clause state has value $1$, then the literal
  constraints enforce that exactly one of $x_i^0,x_i^1$ has value $1$
  for each $i$. Thus the optimistic value function encodes a
  satisfying assignment, completing the reduction.
\end{proof}

\subsubsection{Proof of Theorem~\ref{thm:olive_nphard_g_simple}}

In this section, we prove that the optimization problem solved by \olive is
NP-hard.  The proofs rely on the fact that \olive only adds a constraint for a
single time step $h$ that has high average Bellman error. However, using an
extended construction, one can show similar statements for a version of \olive
that adds constraints for all time steps if there is high average Bellman
error in any time step.

For notational simplicity, we do not prove
Theorem~\ref{thm:olive_nphard_g_simple}
and Theorem~\ref{thm:olive_nphard_g} directly, but versions
of these statements below with a tabular $Q$-function representation $\Fcal$ instead of the $(\Gcal, \Pi)$ version presented in the paper.
For this formulation, \olive picks the policy for the next round as the greedy policy $\pi_{\hat f_k}$ of the $Q$-function that maximizes
\begin{align} \label{eqn:oliveprobf}
    &\hat f_k = \argmax_{f \in \Fcal} \EEx_{D_0}[f(x, \pi_f(x))]\\
    &\textrm{s.t. }
    ~\forall ~ D_i \in \Dcal:\nonumber\\
    &|\EEx_{D_i}[\one\{a = \pi_f(x) \}(f(x, a) - r - f(x', \pi_f(x')))]| \leq \phi.\nonumber
\end{align}

This proof naturally extends to the $(\Gcal, \Pi)$ representation: note that \olive runs in a completely equivalent way if it takes a set of $(g, \pi)$ pairs induced by $\Fcal$ as inputs, i.e., $\{(x \mapsto f(x, \pi_f(x)), x \mapsto \pi_f(x)): f \in \Fcal\}$ \citep[see Appendix A.2,]{jiang2017contextual}. When $\Fcal$ is the tabular $Q$-function class, it is easy to verify that the induced set is the same as $\Gcal \times \Pi$ where $\Gcal$ and $\Pi$ are the tabular value-function / policy classes respectively. Therefore, the
proof for Theorem~\ref{thm:olive_nphard_g_simple} just requires a
simple substitution where $f(x, \pi_f(x))$ is replaced by $g(x)$ and
$\pi_f(x)$ is replaced by $\pi$.

We first prove the simpler NP-hardness claim. 

\begin{theorem}[$\Fcal$-Version of Theorem~\ref{thm:olive_nphard_g_simple}]
    \label{thm:olive_nphard_f_simple}
    Let $\Pcal$ denote the family of problems of the
    form~\eqref{eqn:oliveprobf}, parameterized by $(\Xcal, \Acal, D_0, \Dcal)$
    with implicit $\Fcal = (\Xcal \times \Acal \to [0,1])$ (i.e., the tabular
    $Q$-function class) and with $\phi=0$. $\Pcal$ is NP-hard.
\end{theorem}
\begin{proof}
  For the ease of presentation, we show the statement for $\Fcal = (\Xcal
    \times \Acal \to [-1,1])$ and all values scaled to be in $[-1, 1]$. By
    linearly transforming all rewards accordingly, one obtains a proof for the
    statement with all values in $[0,1]$.

  We demonstrate a reduction from 3-SAT. Recall that an instance of
  3-SAT is a Boolean formula $\psi$ on $n$ variables can be described
  by a list of clauses $C_1, \dots C_m$ each containing at 3 literals
  (a variable $x_i$ or its negation $\bar x_i$), e.g.
    $C_1 = (\bar x_2 \vee x_3 \vee \bar x_5)$. As notation let $o_{j,i}^1$
  for $i \in \{1,2,3\}$ denote the $i^\textrm{th}$ literal in the $j^\textrm{th}$
    clause and $o_{j,i}^0$ its negation (e.g. $o_{1,3}^1 = \bar x_5$ and $o_{1,3}^0 = x_5$).
    Given a $3$-SAT instance with $m$
  clauses $C_{1:m}$ and $n$ variables $x_{1:n}$, we define a class of
  finite episodic MDPs $\Mcal$. This class contains (among others)
  $2^n$ MDPs that correspond each to an assignment of Boolean values
  to $x_{1:n}$.

  The proof proceeds as follows: First we describe the construction of
  this class of MDPs. Then we will demonstrate a set of constraints
  for the \olive program. Importantly, these constraints do not
  distinguish between the $2^n$ MDPs in the class $\Mcal$
  corresponding to the binary assignments to the variables
  $x_{1:n}$, so the optimistic planning step in \olive needs to reason about all possible assignments. Finally, we show that with the function class
  $\Fcal = (\Xcal \times \Acal) \to [-1,1]$, the solution to the
  optimization problem~\eqref{eqn:oliveprobf} determines whether
  $\psi$ is satisfiable or not.

  For simplicity, the MDPs in $\Mcal$ have different actions
  available in different states and rewards are in $[-1, 1]$ instead
  of the usual $[0, 1]$. We can however find equivalent MDPs that
  satisfy the formal requirements of \olive. 

\begin{figure}
\begin{center}
\includegraphics[width=.7\textwidth]{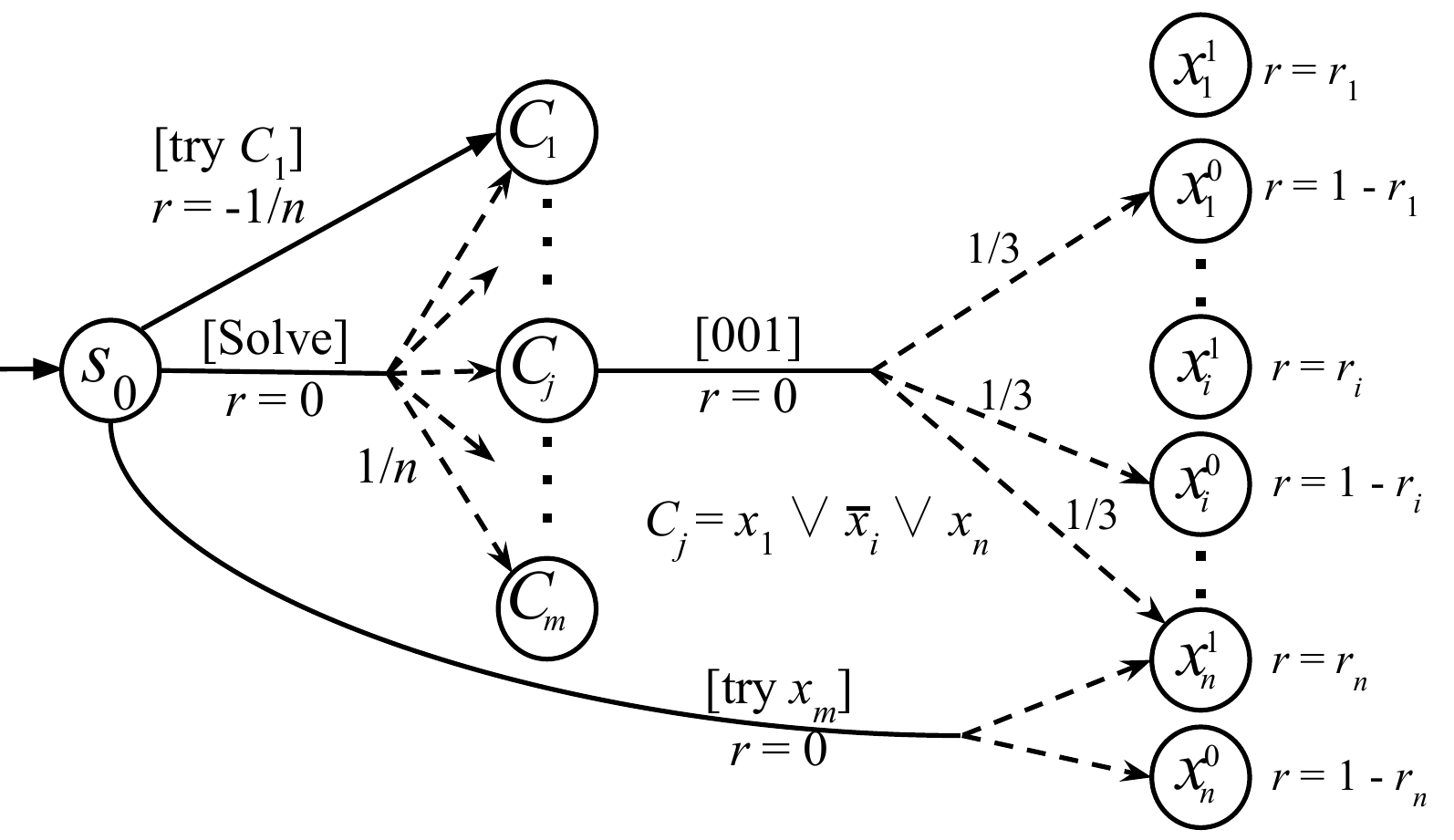}
\end{center}
\caption{Family of MDPs $\Mcal$ for a specific instance of a 3-SAT problem.}
\label{fig:nphard}
\end{figure}

\paragraph{MDP structure.} Let $\psi$ be the 3-SAT instance with
variables $x_{1:n}$ and clauses $C_{1:m}$. The state space for MDPs in $\Mcal$ consists of $m+2n+1$ states, two for each variable, one
for each clause, and one additional starting state. For each variable
$x_i$, there are two states $x_i^0, x_i^1$ corresponding to the
variable and its negation. Each clause $C_j$ has a state $C_j$, and
the starting state is denoted $s_0$.

The transitions are as follows: The states $x_i^0,  x_i^1$
corresponding to the literals are terminal, with just a single
action. The class $\Mcal$ differs only in how it assigns rewards to
these terminal states. Specifically let $y \in \{0,1\}^n$ be a binary
vector, then there is an MDP $M_y \in \Mcal$ where for all $i \in [n]$
    the reward for literal $x_i^{y_i} = 1$ and
    $x_i^{1 - y_i} = 0$. Specifically, all MDPs in $\Mcal$ have values 
that satisfy $V(x_i^1) + V(x_i^0) = 1$ for all $i \in [n]$.

Each clause state $C_j$ has 7 actions, indexed by
$b \in \{0,1\}^3 \setminus \{``000"\}$, each corresponding to an
assignment of the variables that would satisfy the clause.  Taking an
action $b$ transitions the agent to three literal states with equal
probability $1/3$ and the agent receives no immediate reward. Which
literals is determined by the clause. Assume the clause consists of
$C_t = (\bar x_i \vee x_j \vee \bar x_k)$. Then
\begin{align*}
\prob(x_i^1 | c_t, b) =& \frac 1 3 \one\{b_1 = 0\}, \quad \prob( x_i^0 | c_t, b) = \frac 1 3 \one\{b_1 = 1\}\\
\prob(x_j^1 | c_t, b) =& \frac 1 3 \one\{b_2 = 1\}, \quad \prob( x_j^0 | c_t, b) = \frac 1 3 \one\{b_2 = 0\}\\
\prob(x_k^1 | c_t, b) =& \frac 1 3 \one\{b_3 = 0\}, \quad \prob( x_k^0 | c_t, b) = \frac 1 3 \one\{b_3 = 1\}.
\end{align*}
For example, taking action $011$ in clause state
$C_1 = (\bar x_2 \vee x_3 \vee \bar x_5)$ transitions with equal
probability to $x_2^1$ (since the first component of the action is $0$),
$x_3^1$ (second component is $1$) and $x_5^0$ (last component is
$1$).

The initial state has $n+m+1$ actions. The first set of actions are
labeled \texttt{[try $x_i$]}, for $i \in [n]$. They receive zero
instantaneous reward and transition uniform to $x_i^1, x_i^0$. The
second set of actions are labeled \texttt{[try $C_j$]} (for
$j \in [m]$), which receives $1/m$ instantaneous reward and
transitions deterministically to $c_j$. Finally there is a
\texttt{[solve]} action that transitions uniformly to the
$\{C_j\}_{j=1}^m$ states and receives zero instantaneous reward.

\paragraph{\olive Constraints.}
We introduce constraints at the start state $s_0$, all of the
constraint states $C_j$, and the distributions induced when taking the
\texttt{[try $x_i$]} action. Since the literal states $x_i^1, x_i^0$
have no actions, we omit the second argument from the
$Q$-functions $f$. We list these constraints in the following writing out the constraints for each optimal action that are implied by the indicator of the original constraints in Problem~\eqref{eqn:oliveprobf}:
From initial state:
\begin{align}
    f(s_0, \texttt{[try $c_j$]}) &= \max_b f(c_1, b) - 1/m \label{eqn:npconstr1}
    & \textrm{if } \pi_{f}(s_0) =& \texttt{[try $c_j$]}\\
    f(s_0, \texttt{[solve]}) &= \frac{1}{m} \sum_{i=1}^m \max_{b} f(C_j,b)
    & \textrm{if } \pi_{f}(s_0) =& \texttt{[solve]}\\
    f(s_0, \texttt{[try $x_i$]}) &= \frac{f(x_i^0) + f(x_i^1)}{2}
    & \textrm{if } \pi_{f}(s_0) =& \texttt{[try $x_i$]}
\end{align}
From clause $j$ after \texttt{[try $C_j$]}:
\begin{align}
    f(C_j,b) &= \frac{f(o_{j,1}^{b(1)}) + f(o_{j,2}^{b(2)})
    + f(o_{j,3}^{b(3)})
    }{3}
    & \textrm{if } \pi_{f}(C_j) =& b
\end{align}
From variable $i$  after \texttt{[try $x_i$]}:
\begin{align}
  \frac{f(x_i^1) + f(x_i^0)}{2} & = \frac{1}{2}\label{eqn:npconstr2}
\end{align}
Note that all appearances of $f$ on the LHS could be replaced by $f(\cdot , \pi_f(\cdot))$.
There are other types of constraints involving literal states that could be imposed,
specifically constraints of the form
\begin{align}
\sum_{i=1}^m w_{2i-1}f(x_i^1) + w_{2i}f(x_i^0) = V\label{eqn:npnotconstr}
\end{align}
for some $V$ and
$w \in \Delta([2m])$, which appears by first applying \texttt{[solve]} or \texttt{[try $C_j$]}
and then various actions at the clause states to arrive at a
distribution over the literal states. It is important here that constraints of this type are \emph{not} included in the optimization
problem, since it distinguishes elements of the family $\Mcal$.

\paragraph{The Optimal Value.}
Consider the \olive optimization problem~\eqref{eqn:oliveprobf} on the
family of MDPs $\Mcal$ with constraints described above. Note that all MDPs in the
family generate identical constraints, so formulating the optimization
problem does not require determining whether $\psi$ has a satisfying
assignment or not.

Now, if $\psi$ has a satisfying assignment, say
$y^\star \in \{0,1\}^n$, then the MDP $M_{y^\star} \in \Mcal$ has
optimal value $1$. Moreover since the function class $\Fcal$ is
entirely unconstrained, this function class can achieve this value,
which is the solution to Problem~\eqref{eqn:oliveprobf}. To see why
$M_{y^\star}$ has optimal value $1$, consider the policy that chooses
the \texttt{[solve]} action and from each clause chooses the 3-bit
string that transitions to the literal states that have value
$1$. Importantly, since $\psi$ has a satisfying assignment, this must
be true for one of the 7 actions.

Conversely, suppose that Problem~\eqref{eqn:oliveprobf}, with all the
constraints described above, has value $1$. We argue that this implies
$\psi$ has a satisfying assignment. Let $\hat{f}, \hat{\pi}$
correspond to the $Q$-value and policy that achieve the optimal value
in the program. First, due to the constraints on the \texttt{[try
  $x_i$]} distributions and the immediate negative rewards for taking
\texttt{[try $C_j$]} actions, we must have
$\hat{\pi}(s_0) = \texttt{[solve]}$ and
$\hat{f}(s_0,\texttt{[solve]}) = 1$. The constraints on $\hat{f}$ now
imply that for each clause $C_j$ there exists a action
$b_j \in \{0,1\}^3 \setminus \{000\}$ such that
$\hat{f}(C_j,b_j) = 1$. Proceeding one level further, if $b_j$
satisfies $\hat{f}(C_j,b_j) = 1$ then we must have that $\hat f(o_{j,k}^{b_j(k)}) = 1$ for all $k \in \{1, 2, 3\}$.
And due to the boundedness conditions on $\hat{f}$ along with the
constraint that $\hat{f}(x_i^0) +\hat{f}(x_i^0) = 1$, one of these
values must be 1, while the other is zero. Therefore, for any variable
that appears in some clause the corresponding literal states must have
predicted value that is binary. Since the constraints corresponding to
the clauses are all satisfied (or else we could not have value $1$ at
$s_0$), the predicted values at the literal states encodes a
satisfying assignment to $\psi$.
\end{proof}

\subsubsection{Proof of Theorem~\ref{thm:olive_nphard_g}}
After showing that Problem~\eqref{eqn:oliveprobf} is NP-hard when constraints
are chosen adversarially, we extend this result to the class of problems
encountered by running \olive.  Again, we prove a version of the statement with
$\Fcal$ representation but the proof for Theorem~\ref{thm:olive_nphard_g} is
completely analogous.

\begin{theorem}[$\Fcal$ Version of  Theorem~\ref{thm:olive_nphard_g}]
    \label{thm:olive_nphard_f}
  Let $\Pcal_{\olive}$ denote the family of problems of the
  form~\eqref{eqn:oliveprobf}, parameterized by
  $(\Xcal,\Acal,\textrm{Env}, t)$, which describes the optimization
  problem induced by running \olive in the MDP environment
  $\textrm{Env}$ (with states $\Xcal$, actions $\Acal$ and perfect
  evaluation of expectations) for $t$ iterations with
  $\Fcal = (\Xcal \times \Acal \to [0,1])$ and with
  $\phi=0$. $\Pcal_{\olive}$ is NP-hard.
\end{theorem}
\begin{proof}
    The proof uses the same family of MDPs $\Mcal$ and set of constraints as
    the proof of Theorem~\ref{thm:olive_nphard_f_simple} above.  As mentioned there,
    it is crucial that constraints in
    Equations~\eqref{eqn:npconstr1}-\eqref{eqn:npconstr2} are added for all
    clauses and literals but none of the possible constraints of the form in
    Equation~\eqref{eqn:npnotconstr} that arise from distributions over literal
    states after taking actions \texttt{[try $C_j$]} or \texttt{[Solve]}.
    To prove that \olive can encounter NP-hard problems, it therefore remains to show that running \olive on any MDP in $\Mcal$
    can generate the exact set of constraints in Equations~\eqref{eqn:npconstr1}-\eqref{eqn:npconstr2}.

    The specification of \olive by \citet{jiang2017contextual} only
    prescribes that a constraint for one time step $h$ among all that
    have sufficiently large average Bellman error is added. It however
    leaves open how exactly $h$ is chosen and which $f \in \Fcal$ is
    chosen among all that maximize
    Problem~\eqref{eqn:oliveprobf}. Since this component of the
    algorithm is under-specified, we choose $h$ and $f \in \Fcal$ in an
    adversarial manner within the specification, which amounts to
    adversarial tie breaking in the optimization.

    We now provide a run of \olive on an arbitrary MDP in $\Mcal$ that generates exactly the set of constraints in Equations~\eqref{eqn:npconstr1}-\eqref{eqn:npconstr2}:
    \begin{itemize}
        \item For the first $t \in [m]$ iterations, \olive picks any Q-function $f_t \in \Fcal$ with
            $f_t(s_0, b) = \one\{ b = \texttt{[try $C_t$]}\}$ and 
            $f_t(C_t, b) = 1$ and $f_t(x_i^0, \pi_{f_t}(x_i^0)) =f_t(x_i^1, \pi_{f_t}(x_i^1)) = 0$ for all actions $b$ and $i \in [n]$ and chooses to add constraints for $h=2$.
            Since the context distributions is a different $C_t$ for every iteration $t$, this is a valid choice and 
generates constraints
\begin{align*}
    f(C_t,b) &= \frac{f(o_{t,1}^{b(1)}) + f(o_{t,2}^{b(2)})
    + f(o_{t,3}^{b(3)})
    }{3}
    & \textrm{if } \pi_{f}(C_t) =& b
\end{align*}
for all $b$.
\item For the next $n$ iterations $t = m+1, m+2, \dots m+n$, \olive picks any Q-function $f_t \in \Fcal$ with
        $f_t(s_0, b) = \one\{ b = \texttt{[try $x_{t-m}$]}\}$ and
            $f_t(x_{t-m}^0, \pi_{f_t}(x_{t-m}^0)) = f_t(x_{t-m}^1, \pi_{f_t}(x_{t-m}^1)) = 1$ for all $b$. The only positive average Bellman error occurs in the mixture over literal states at $h=2$ and therefore constraints 
\begin{align*}
    \frac{f(x_{t-m}^1, \pi_f(x_{t-m}^1)) + f(x_{t-m}^0, \pi_f(x_{t-m}^0))}{2} & = \frac{1}{2}
\end{align*}
are added.
\item Finally, in iteration $t=m+n+1$, \olive picks any $f_t \in \Fcal$ with
$f_t(s_0, b) = \one\{ b = \texttt{[try $x_{1}$]}\}$ and
$f_t(x_{1}^0, \pi_{f_t}(x_{1}^0)) = f_t(x_{1}^1, \pi_{f_t}(x_{1}^1)) = 1/2$ for all actions $b$. Now there is positive average Bellman error in the initial state $s_0$ and with $h_t = 1$ the following constraints are added
\begin{align*}
    f(s_0, \texttt{[try $c_j$]}) &= \max_b f(C_1, b) - 1/m
    & \textrm{if } \pi_{f}(s_0) =& \texttt{[try $c_j$]}\\
    f(s_0, \texttt{[solve]}) &= \frac{1}{m} \sum_{i=1}^m \max_{b} f(C_j,b)
    & \textrm{if } \pi_{f}(s_0) =& \texttt{[solve]}\\
    f(s_0, \texttt{[try $x_i$]}) &= \frac{f(x_i^0) + f(x_i^1)}{2}
    & \textrm{if } \pi_{f}(s_0) =& \texttt{[try $x_i$]}
\end{align*}
            for all $i \in [n]$ and $j \in [m]$.
    \end{itemize}
    Since at iteration $t=m+n+2$, the set of constraints matches exactly the one in the proof of Theorem~\ref{thm:olive_nphard_f_simple}, \olive solves exactly the problem instance described there which solves the given 3-SAT instance.
\end{proof}

\section{Additional Barriers}
\label{sec:examples}
In this section, we describe several further barriers that we must
resolve in order to obtain tractable algorithms in the stochastic
setting. 

\subsection{Challenges with Credit Assignment}
\label{sec:examples_learning}

We start with the learning step, ignoring the challenges with
exploration, and focus on a family of algorithms that we call
\emph{Bellman backup} algorithms.
\begin{definition}
  \label{def:bellman_backup}
  A \emph{Bellman backup} algorithm collects $n$ samples from every
  state and iterates the policy/value updates
  \begin{align*}
      \hat{\pi}_h &= \argmax_{\pi \in \Pi_h} \smashoperator[lr]{\sum_{s \in \Scal_h}} \Ex_s [r + \hat{g}_{h+1}(x') | a = \pi(x)]\\
    \hat{g}_h &= \argmin_{g \in \Gcal_h} \smashoperator[lr]{\sum_{s \in \Scal_h}} \Ex_s [(g(x) - r - \hat{g}_{h+1}(x'))^2 | a = \hat{\pi}_h(x)].
  \end{align*}
\end{definition}
This algorithm family differs only in the exploration component, which
we are ignoring for now, but otherwise is quite natural. In fact,
these algorithms can be viewed as a variants of Fitted Value Iteration
(FVI)\footnote{This family algorithms are prevalent in empirical RL
  research today: the popular Q-learning algorithm with function
  approximation can be viewed as a stochastic variant of its batch
  version known as Fitted Q-iteration \cite{ernst2005tree}, which fits
  into the FVI framework.} \cite{munos2008finite, farahmand2010error}
adapted to the $(g, \pi)$ representation. Unfortunately,
such algorithms
cannot avoid exponential sample complexity, even ignoring exploration
challenges.

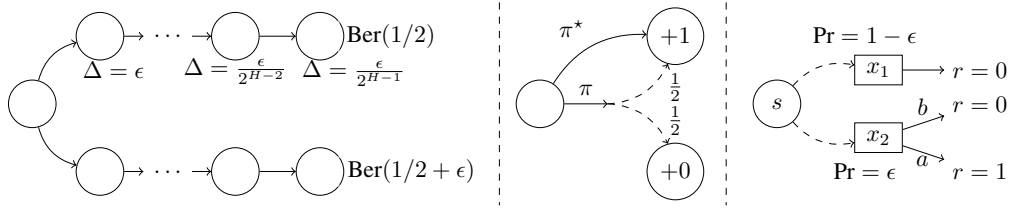
\begin{figure*}
\centering
    \scalebox{0.9}{
        \begin{tikzpicture}
\node[draw=black,circle,name=s0,minimum width=20pt] at (0,0) {};
\node[draw=black,circle,name=g1,minimum width=20pt] at (1,1) {};
\node[draw=black,circle,name=b1,minimum width=20pt] at (1,-1) {};
\node[draw=white,circle,name=g2,minimum width=20pt] at (2,1) {$\ldots$};
\node[draw=white,circle,name=b2,minimum width=20pt] at (2,-1) {$\ldots$};
\node[draw=black,circle,name=g3,minimum width=20pt] at (3,1) {};
\node[draw=black,circle,name=b3,minimum width=20pt] at (3,-1) {};
\node[draw=black,circle,name=g4,minimum width=20pt] at (4.25,1) {};
\node[draw=black,circle,name=b4,minimum width=20pt] at (4.25,-1) {};
\node[draw=white,name=g5,minimum width=20pt] at (5.3,1) {$\textrm{Ber}(1/2)$};
\node[draw=white,circle,name=b5,minimum width=20pt] at (5.6,-1) {$\textrm{Ber}(1/2+\epsilon)$};

\draw[black, ->, bend left]  (s0) to (g1);
\draw[black, ->, bend right] (s0) to (b1);
\draw[black, ->] (g1) to (g2);
\draw[black, ->] (g2) to (g3);
\draw[black, ->] (g3) to (g4);
\draw[black, ->] (b1) to (b2);
\draw[black, ->] (b2) to (b3);
\draw[black, ->] (b3) to (b4);

\draw[black] (4.75,0.45) node{$\Delta = \frac{\epsilon}{2^{H-1}}$};
\draw[black] (3,0.48) node{$\Delta = \frac{\epsilon}{2^{H-2}}$};
\draw[black] (1.2,0.5) node{$\Delta = \epsilon$};

\draw [dashed] (6.9,1.5) --(6.9,-1.5);
\node[draw=black,circle,name=s0,minimum width=20pt] at (7.5,0) {};
\node[draw=black,circle,name=g1,minimum width=20pt] at (9.5,1) {$+1$};
\node[draw=black,circle,name=b1,minimum width=20pt] at (9.5,-1) {$+0$};

\draw[black, ->] (s0) to ["$\pi$"] (8.5,0); 
\draw[black, ->, dashed, bend right] (8.5,0) to (g1); 
\draw[black, ->, dashed, bend left] (8.5,0) to (b1);
\draw[black, ->, bend left] (s0) to ["$\pi^\star$"] (g1); 

\node[] at (9.5,-0.25) {$\tfrac{1}{2}$};
\node[] at (9.5,0.25) {$\tfrac{1}{2}$};

\draw [dashed] (10.25,1.5) --(10.25,-1.5);
\node[draw=black,circle,name=b3,minimum width=20pt] at (11,0) {$s$};
\node[draw=black,rectangle,name=o1,minimum width=20pt] at (12.5,0.5) {$x_1$};
\node[draw=black,rectangle,name=o2,minimum width=20pt] at (12.5,-0.5) {$x_2$};
\node[draw=white,name=r1,minimum width=20pt] at (14,0.5) {$r=0$};
\node[draw=white,name=r2,minimum width=20pt] at (14,0) {$r=0$};
\node[draw=white,name=r3,minimum width=20pt] at (14,-1) {$r=1$};

\draw[black,->] (o1) to (r1);
\draw[black,->] (o2) to (r2);
\draw[black,->] (o2) to (r3);
\draw[black,dashed,->,bend left] (b3) to (o1);
\draw[black,dashed,->,bend right] (b3) to (o2);
\draw[black] (12.3,1) node{$\textrm{Pr} = 1-\epsilon$};
\draw[black] (12.3,-1) node{$\textrm{Pr} = \epsilon$};
\draw[black] (13.15, -0.05) node{$b$};
\draw[black] (13.15, -0.85) node{$a$};

\end{tikzpicture}}
\caption{Further barriers to tractable algorithms. Circles
  denote states, while rectangles denote observations. Solid lines
  denote deterministic transitions. Dashed lines denote stochastic
  transitions (middle) or context emissions (right).  Left: construction for Theorem~\ref{thm:backup}, where $\Delta := g_{\textrm{bad}} - g^\star$ reflects the amount that $g_{\textrm{bad}}$ over-predicts in each state. 
  On the upper chain, statistical fluctuations can favor
  $g_{\textrm{bad}}$ over $g^\star$, which leads to a policy choosing the wrong
  action at the start. Center: construction for
  Proposition~\ref{prop:sq_loss_good}, where most policies induce a
  uniform distribution over states at level two and an average
  constraint cannot drive the agent to the top state. Right:
  construction for Proposition~\ref{prop:sq_loss_bad} where an $\epsilon$ loss in roll-out policy converts into a $\sqrt{\epsilon}$ prediction error in value function.  \label{fig:constructions}}
\end{figure*}

\begin{theorem}
  \label{thm:backup}
For any $H\ge 1$, $\epsilon \in (0, 1)$, there exists a
layered tabular MDP with $H$ levels, 2 states per level, and
constant-sized $\Gcal$ and $\Pi$ satisfying
Assumptions~\ref{asm:pol_realize}
and~\ref{asm:val_realize}, such that when
$n < 4^H / (32\epsilon^2)$, with probability at least $1/4$, the
bellman backup algorithm outputs a policy $\hat{\pi}$ such that
$V^{\hat{\pi}} \le V^\star - \epsilon$.
\end{theorem}
A sketch of the construction is displayed in the left panel of
Figure~\ref{fig:constructions}. The intuition is that statistical
fluctuations at the final state can cause bad predictions, which can
exponentiate as we perform the backup. Ultimately this can lead to
choosing a exponentially bad action at the starting state.
The full proof follows:

\begin{proof}
    Consider an MDP with $H+1$ levels with deterministic transitions and with
one start state $x_0$ and two states per level
$\{x_{h,a}, x_{h,b}\}_{1 \le h \le H}$. From the start state there are
two actions $a,b$ where $a$ transitions to $x_{1,a}$ and $b$
transitions to $x_{1,b}$. From then on, there is just one action which
transitions from $x_{h,z}$ to $x_{h+1,z}$ $z \in \{a,b\}$. 
The reward
from $x_{H,a}$ is $\textrm{Ber}(1/2+\epsilon)$ and the reward
from $x_{H,b}$ is $\textrm{Ber}(1/2)$. 
Both value functions in the class have $g(x_{h,a}) =
1/2+\epsilon$. $g^\star$ is in the class and it has
$g^\star(x_0) = 1/2+\epsilon$, $g^\star(x_{h,b}) = 1/2$. There
is also a bad function 
$g_{bad}(x_0) = 1/2+\epsilon, g_{bad}(x_{h,b}) = 1/2 +
\epsilon/2^{h-1}$.

Since in our construction there are only two policies, and they only
differ at $x_0$, for the majority of the proof we can focus on
policy evaluation. The first step is to show that with
non-trivial probability, we will select $g_{bad}$ in the first square
loss problem. Since all functions make the same predictions on
$x_{H,a}$ we focus on $x_{H,b}$. Our goal is to show that $g_{bad}$ will be chosen by the algorithm with substantial probability.

\paragraph{A lower bound on the binomial tail.} The rewards from
$x_{H,b}$ are drawn from $\textrm{Ber}(1/2)$. Call this values
$r_1,\ldots, r_n$ with average $\bar{r}$. We select $g_{bad}$ if
$\bar{r} \ge 1/2 + \epsilon/2^{H}$. By Slud's lemma, the
probability is
\begin{align*}
  \PP[ \bar{r} \ge 1/2 +  \epsilon/2^{H}] \ge 1- \Phi\left( \frac{n \epsilon/2^{H} }{\sqrt{n/4}}\right)
\end{align*}
where $\Phi$ is the standard Gaussian CDF, which can be upper bounded by
\begin{align*}
\Phi(x) = \frac{1}{\sqrt{2\pi}} \int_{-\infty}^x \exp(-u^2/2)du  
= \frac{1}{2} + \frac{1}{\sqrt{2\pi}} \int_{0}^x \exp(-u^2/2)du \le \frac{1}{2} + \frac{x}{\sqrt{2\pi}}.
\end{align*}
Thus the probability is at least
\begin{align*}
  \ge \frac{1}{2} - \frac{1}{\sqrt{2\pi}} \frac{n \epsilon/2^H}{\sqrt{n/4}} \ge \frac{1}{2} - \sqrt{\frac{8}{2\pi}} \epsilon/2^H \sqrt{n} \ge 1/4
\end{align*}
where the final inequality holds since
$n \le \frac{4^H}{32\epsilon^2} <  \frac{\pi 4^H}{64\epsilon^2}$.

Thus with probability at least $1/4$ the average reward from state
$x_{H,b}$ is at least $1/2 + \epsilon/2^H$, in which case $g_{bad}$ is the
square loss minimizer. From this point on, at every subsequent level
$1 < h < H$, both $g$ and $g_{bad}$ have the same square loss and
tie-breaking can cause $g_{bad}$ to always be chosen. Thus the final
policy optimization step uses $g_{bad}$ to approximate the future but
$g_{bad}(x_{1,a}) = g_{bad}(x_{1,b}) = 1/2 + \epsilon$. Thus the
final policy can select action $b$, which leads to a loss of $\epsilon$.
\end{proof}

Actually, the policy optimization step is inconsequential in the
construction. As such, the theorem shows that FVI style learning rules
cannot avoid bias that propagates exponentially without further
assumptions, leading to an exponential sample complexity
requirement. We emphasize that the result focuses exclusively on the
learning rule and applies \emph{even} with small observation spaces
and \emph{regardless} of exploration component of the algorithm, with
similar conclusions holding for variations including
$Q$-representations and different loss
functions. Theorem~\ref{thm:backup} provides concrete motivation for
stronger realizability conditions such as
Assumption~\ref{asm:polval_complete}, variants of which are also used
in prior analysis of FVI-type methods~\citep{munos2008finite}.

\subsection{Challenges with Exploration}
We now turn to challenges with exploration that arise when factoring
the $Q$-function class into the $(g,\pi)$ pairs, which works well in
the deterministic setting, as in Section~\ref{sec:alg}. However, the
stochastic setting presents further challenges.
Our first construction shows that a decoupled approach using \olive's
average Bellman error in the learning rule can completely fail to
learn in the stochastic setting.

Consider an algorithm that uses an optimistic estimate for
$\hat{g}_{h+1}$ to find a policy $\hat{\pi}_h$ that drives further
exploration. Specifically, suppose that we find an estimate
$\hat{g}_{h+1}$ such that for all previously visited distributions
$D \in \Dcal_{h+1}$ at level $h+1$
\begin{align}  \label{eq:expectation_constraint}
 \EEx_{D} [\hat{g}_{h+1}(x)] =  \EEx_{D} [g^\star(x)],
\end{align}
where we assume that all expectations are exact.  We may further
encourage $\hat{g}_{h+1}$ to be optimistic over some distribution that
provides good coverage over the states at the next level. Then, we use
$\hat{\pi}_h = \argmax_{\pi} \EEx_{D_h} [r + \hat{g}_{h+1}(x') | a =
  \pi(x)]$ as the next exploration policy. Intuitively, optimism in
$\hat{g}_{h+1}$ should encourage $\hat{\pi}_h$ to visit a new
distribution, which will drive the learning process. Unfortunately,
the next proposition shows that this policy $\hat \pi$ may be highly
suboptimal and also fail to visit a new distribution.

\begin{proposition}
  \label{prop:sq_loss_good}
  There exists a problem, in which the algorithm above stops exploring new distributions when the best policy it finds is worse than the optimal policy by constant value.
\end{proposition}
We sketch the construction in the center panel of
Figure~\ref{fig:constructions}.  We create a two level problem where
most policies lead to a uniform mixture over two subsequent states,
one good and one bad. Constraint~\eqref{eq:expectation_constraint} on
this distribution favors a value function that predicts $1/2$ on both
states, and with this function, the optimistic policy leads us back to
the uniform distribution. Thus no further exploration occurs!

\begin{proof}[Proof of Proposition~\ref{prop:sq_loss_good}]
    Consider a two-layer process with one initial state $s_0$ with
  optimal value $V(s_0) = 1$ and two future states $s_1,s_2$ where
  $V(s_1) = 1$, $V(s_2) = 0$. From $s_0$ action $a$ deterministically
  transitions to $s_1$ and action $b$ deterministically transitions to
  $s_2$, from $s_1$ and $s_2$ all actions deterministically receive
  the corresponding reward. No rewards are received upon making the
  first transition. There are $m$ contexts that are equally likely
  from $s_0$ and the policy class consists of one policy, $\pi^\star$
  that always chooses action $a$, and $\Omega(2^m)$ bad policies that
  choose action $a,b$ with equal probability. These policies each have
  value $1/2$.

  If we perform a roll-in with a bad policy, we generate a constraint
  of the form $|g(s_1)/2 + g(s_2)/2 - 1/2 | \le \epsilon$ at level
  two. Hence, with this constraint we might pick $\hat{g}_{h+1}$ such
  that $\hat{g}_{h+1}(s_1) = \hat{g}_{h+1}(s_2) = 1/2$, since it
  satisfies all the constraints and also has maximal average value on
  any existing roll-in. However, using this future-value function in
  the optimization
  \begin{align} \label{eq:needle_value}
    \EE_{x_h \sim s_0} [r_h + \hat{g}_{h+1}(s') | a_h = \pi(x_h)],
  \end{align}
  we see that all policies, including $\pi^\star$ have the same
  objective value. When we choose any one of them but $\pi^\star$, the ``optimistic'' value computed by maximizing Eq.\eqref{eq:needle_value} will be achieved by the chosen policy and the algorithm stops exploration with a suboptimal policy.
\end{proof}

The main point is that by using the average value
constraints~\eqref{eq:expectation_constraint}, we lose information
about the ``shape" of $g^\star$, which can be useful for exploration.
In fact, the proposition does not rule out approaches that learn the
shape of the state-value function, for example with square loss
constraints that capture higher order information. However square loss
constraints are less natural for value based reinforcement learning,
as we show in the next proposition. We specifically focus on measuring
a value function by its square loss to a near optimal roll-out.

\begin{proposition}
  \label{prop:sq_loss_bad}
  In the environment in the right panel of
  Figure~\ref{fig:constructions}, an $\epsilon$-suboptimal policy
  $\hat{\pi}$ achieves reward $0$, and the square loss of $g^\star$ w.r.t.~the roll-out reward is
$\Ex_{x \sim s}[ (g^\star(x) - r)^2 \mid a \sim \hat{\pi}] = \epsilon$. 
  This square loss is also achieved by a bad value function
  $g_{\textrm{bad}}$ such that $\Ex_{x \sim s}[g_{\textrm{bad}}(x) - g^\star(x)] = O(\sqrt{\epsilon})$. 
\end{proposition}

The claim here is weaker than the previous two barriers, but it does
demonstrate some difficulty with using square loss in an approach that
decouples value function and policy optimization. The essence is that
a roll-out policy $\hat{\pi}$ that is slightly suboptimal on average
may have significantly lower variance than $\pi^\star$. Since the
square loss captures variance information, this means that $g^\star$
may have significantly larger square loss to $\hat{\pi}$'s rewards,
which either forces elimination of $g^\star$ or prevents us from
eliminating other bad functions, like $g_{\textrm{bad}}$ in the
example.

\begin{proof}[Proof of Proposition~\ref{prop:sq_loss_bad}]
  Consider a process with $H=1$ and just one state with two
  observations: $x_1$ and $x_2$, both with two actions. For $x_1$,
  both actions receive $0$ reward, while for $x_2$ action $a$ receives
  reward $1$ while action $b$ receives reward $0$. However,
  observation $x_2$ appears only with probability $\epsilon$. As such,
  an $\epsilon$-optimal policy from this state may choose action $b$
  on both contexts, receiving zero reward. Let $\hat{\pi}$ denote this
  near-optimal policy.

  The value function class $\Gcal$ provided to the algorithm has many
  functions, but three important ones are (1) $g_0$ which always
  predicts zero and is the correct value function for $\hat{\pi}$
  above, (2) $g^\star$ which is the optimal value function, and (3)
  $g_{\textrm{bad}}$, which we now define. These latter two function
  have
  \begin{align*}
    g^\star(x_1) \triangleq 0, & \quad g^\star(x_2) \triangleq 1\\
    g_{\textrm{bad}}(x_1) \triangleq \sqrt{\epsilon}, & \quad g_{\textrm{bad}}(x_2) \triangleq \sqrt{\epsilon}
  \end{align*}

  Now, let us calculate the square loss of these three value functions
  to the roll-out achieved by $\hat{\pi}$.
  \begin{align*}
    \textrm{sqloss}(g_0, \bd_H, \hat{\pi}) &= (1-\epsilon)(0-0)^2 + \epsilon (0 -0)^2 = 0\\
    \textrm{sqloss}(g^\star, \bd_H, \hat{\pi}) &= (1-\epsilon)(0-0)^2 + \epsilon (1 - 0)^2 = \epsilon\\
    \textrm{sqloss}(g_{\textrm{bad}}, \bd_H, \hat{\pi}) &= (1-\epsilon)(\sqrt{\epsilon}-0)^2 + \epsilon (\sqrt{\epsilon} - 0)^2 = \epsilon
  \end{align*}
  We see that $g_{\textrm{bad}}$ and $g^\star$ have identical square
  loss on this single state, which proves the claim.

  Intuitively, this is bad because if we use constraints defined in
  terms of square loss, we risk eliminating $g^\star$ from the
  feasible set, or we need the constraint threshold to be so high that
  bad functions like $g_{\textrm{bad}}$ remain. These bad function can
  cause exploration to stagnate or introduce substantial bias
  depending on the learning rule.
\end{proof}

To summarize, in this section we argue for the necessity of
completeness type conditions for FVI-type learning procedures, and
demonstrate barriers for exploration with decoupled optimization
approaches, both with expectation and square loss constraints. We
believe overcoming these barriers is crucial to the development of a
computationally efficient algorithm.

\end{document}